%% file: arxiv_main.tex
\documentclass[11pt]{article}

\input{preamble/packages}
\input{preamble/macros}

\usepackage{upgreek}
\DeclareMathOperator*{\Expop}{\mathbb{E}}

\begin{document}
\title{
%Reinforcement meta-learning via contextual bandit base-learners
Bayesian decision-making under misspecified priors with applications to meta-learning
 }

 \author{ 
 Max Simchowitz\thanks{UC Berkeley, \texttt{msimchow@berkeley.edu}. Research initiated during author's internships at MSR and generously supported by an Open Philanthropy PhD Fellowship grant.}
 \and 
 Christopher Tosh\thanks{Columbia University, \texttt{c.tosh@columbia.edu}}
\and 
Akshay Krishnamurthy\thanks{Microsoft Research NYC, \texttt{akshaykr@microsoft.com}} 
 \and 
 Daniel Hsu\thanks{Columbia University, \texttt{djhsu@cs.columbia.edu}. Research supported by NSF grants CCF-1740833, DMREF-1534910, IIS-1563785; a Bloomberg
Data Science Research Grant, a JP Morgan Faculty Award, and a Sloan Research Fellowship.} 
\and
 Thodoris Lykouris\thanks{Massachusetts Institute of Technology, \texttt{lykouris@mit.edu}. Research initiated during author's postdoc at MSR.} 
\and 
Miroslav Dud\'{i}k\thanks{Microsoft Research NYC, \texttt{mdudik@microsoft.com}} 
\and
Robert E. Schapire\thanks{Microsoft Research NYC, \texttt{schapire@microsoft.com}}
 }
\date{\today}
%\date{November 2019}
\maketitle

\begin{abstract}
\input{body/abstract}
\end{abstract}

\input{body/intro}

\input{body/related}
\input{body/setting}
\input{body/smoothness_new}

\input{body/meta_learning_section}

\input{body/experiments}
\input{body/discussion}

\bibliographystyle{alpha}
\bibliography{bibliog}

\newpage
\tableofcontents
\appendix
\newpage

\input{appendix/app_experiments}
\input{appendix/app_smoothness}
\input{appendix/app_formal_algorithms}

\input{body/lower_bound_new}

\input{appendix/app_bayes_cdps}

\input{appendix/app_estimation}

\end{document}

%% file: preamble/packages.tex
\usepackage{amsthm,etoolbox}
  \usepackage{inconsolata}

\newtoggle{neurips}
\togglefalse{neurips}

\usepackage{mathtools}
\usepackage{amsmath}

\usepackage{amsfonts}
\usepackage{bbm}

\usepackage{amssymb}

\usepackage{algorithm}
\usepackage[noend]{algpseudocode}

\usepackage[letterpaper, left=1in, right=1in, top=1in,bottom=1in]{geometry}

\usepackage{verbatim}

\usepackage{booktabs} % For formal tables

\usepackage[utf8]{inputenc}

\usepackage{geometry}

\usepackage{graphpap,amscd,mathrsfs,graphicx,lscape,enumitem,dsfont,bm,url,subfigure}
\usepackage{epsfig,amstext,xspace}

\usepackage{thmtools,thm-restate}

\usepackage[utf8]{inputenc} % allow utf-8 input
\usepackage[T1]{fontenc}    % use 8-bit T1 fonts
\usepackage{hyperref}       % hyperlinks
\usepackage{url}            % simple URL typesetting
\usepackage{booktabs}       % professional-quality tables
\usepackage{amsfonts}       % blackboard math symbols
\usepackage{nicefrac}       % compact symbols for 1/2, etc.
\usepackage{microtype}      % microtypography
\usepackage{bbm}
\usepackage{cleveref}

\usepackage{microtype} 

\usepackage{tikz}
\usetikzlibrary{intersections}
\usetikzlibrary{positioning}
\usetikzlibrary{decorations.pathreplacing}
\usetikzlibrary{matrix}
\usetikzlibrary{calc}

\newif\ifpx
\ifpx

  \usepackage[varg]{pxfonts}
\fi

\usepackage{xcolor}              % Need the color package
\definecolor{orange}{RGB}{244,123,63}
\definecolor{cerulean}{rgb}{0.10, 0.58, 0.75}
\usepackage[suppress]{color-edits}
\addauthor{tl}{cyan}
\addauthor{ms}{purple}
\addauthor{as}{brown}
\addauthor{ak}{orange}
\addauthor{dh}{green}
\addauthor{md}{cerulean}
\addauthor{ct}{gray}

%% file: preamble/macros.tex
\newcommand{\algfont}[1]{\texttt{\textbf{#1}}}
\newcommand{\kfpost}{(k,f_{1:H})\text{-}\mathrm{PosteriorSample}}
\newcommand{\tilP}{\tilde{P}}

\numberwithin{equation}{section}
\newcommand{\calE}{\mathcal{E}}

\DeclarePairedDelimiter{\nrm}{\|}{\|}

\newcommand{\algcomment}[1]{\textcolor{blue!70!black}{\footnotesize{\texttt{\textbf{//
          #1}}}}}

\DeclarePairedDelimiter{\ceil}{\lceil}{\rceil}

\let\Pr\undefined
\let\P\undefined

\DeclareMathOperator{\P}{\mathbb{P}}
\DeclareMathOperator{\Pr}{\mathbb{P}}
\DeclareMathOperator{\E}{\mathbb{E}}

\DeclareMathOperator{\diag}{diag}

\newcommand{\nipsvminpt}{\iftoggle{neurips}{\vspace{-.1in}}}

\DeclareMathOperator*{\argmax}{arg\,max}

 \newcommand{\R}{\mathbb{R}}
 \newcommand{\N}{\mathbb{N}}

\newcommand{\ind}{\mathbbm{1}}    %Indicator

  \newtheorem{claim}{Claim}
  \newtheorem{theorem}{Theorem}[section]
  
  \newtheorem{lemma}{Lemma}[section]

  \newtheorem{corollary}[theorem]{Corollary}

 \newtheorem{proposition}[theorem]{Proposition}
\theoremstyle{definition}
\newtheorem{remark}{Remark}

\newtheorem{definition}{Definition}[section]

\usepackage{tikz}
\usetikzlibrary{arrows}

\newcommand{\one}{\mathbf{1}}

\newcommand{\calA}{\mathcal{A}}
\newcommand{\calX}{\mathcal{X}}

%% file: body/abstract.tex
%!TEX root = ../neurips_main.tex

Thompson sampling and other Bayesian sequential decision-making algorithms are among the most popular approaches to tackle explore/exploit trade-offs in (contextual) bandits.
The choice of prior in these algorithms offers flexibility to encode domain knowledge but can also lead to poor performance when misspecified.
In this paper, we demonstrate that performance degrades gracefully with misspecification.
We prove that the expected reward accrued by Thompson sampling (TS) with a misspecified prior differs by at most $\tilde{O}(H^2 \epsilon)$ from TS with a well specified prior, where $\epsilon$ is the total-variation distance between priors and $H$ is the learning horizon.

Our bound does not require the prior to have any parametric form.  For priors with bounded support, our bound is independent of the cardinality or structure of the action space, and we show that it is tight up to universal constants in the worst case.

Building on our sensitivity analysis, we establish generic PAC guarantees for algorithms in the recently studied Bayesian meta-learning setting and derive corollaries for various families of priors.
Our results generalize along two axes: (1) they apply to a broader family of Bayesian decision-making algorithms, including a Monte-Carlo implementation of the knowledge gradient algorithm (KG), and (2) they apply to Bayesian POMDPs, the most general Bayesian decision-making setting, encompassing contextual bandits as a special case. Through numerical simulations, we illustrate how prior misspecification and the deployment of one-step look-ahead (as in KG) can impact the convergence of meta-learning in multi-armed and contextual bandits with structured and correlated priors.

%% file: body/intro.tex
%!TEX root = ../neurips_main.tex

\newcommand{\tvarg}[2]{\mathrm{TV}(#1 \parallel #2)}

\section{Introduction}
\label{sec:intro}
Bayesian decision-making algorithms are widely popular, due to both strong empirical performance and the flexibility afforded by incorporating inductive biases and domain knowledge through priors. However, in practical applications, any chosen prior is at best an
approximation of the true environment in which the algorithm is
deployed. This raises a critical question:
\begin{quote}
\emph{How sensitive are Bayesian decision-making algorithms to prior misspecification?}
\end{quote}

For decision-making problems with a very large horizon, it suffices that the misspecified prior places a vanishingly small probability mass on the ground truth environment; this condition is referred to informally as a ``grain of truth.'' This is because, in the large-horizon limit, Bayesian algorithms (like many non-Bayesian methods) should converge to the optimal policy.

But in many practical settings, decision-making takes place on shorter time scales. Consider a news recommendation website that, when presented with a new user, sequentially offers a selection of currently trending articles. Such a system may only have a few opportunities to make recommendations before the user decides to navigate away, leaving little time to correct for misspecified or underspecified prior knowledge. Such examples are described more broadly by the meta-learning paradigm, where a single learning agent
must complete multiple disparate-though-related tasks.\looseness=-1

In meta-learning problems, and in short horizon problems more broadly, the ``grain of truth'' argument paints a rather uninformative picture. Consequently, recent work has begun to explore sensitivity bounds in shorter horizon applications \cite{liu2016prior, kveton2021meta}. However, these recent works focus on particular classes of priors and/or reward models, as well as on the Thompson sampling algorithm specifically. Notably, this leaves open questions about the extent to which prior sensitivity is determined by properties of the Bayesian decision-making algorithm, the reward model, and the prior itself.\looseness=-1

\newcommand{\Priorst}{P}
\newcommand{\Priorhat}{P'}
\subsection{Our Contributions}
Motivated by meta-learning problems with short task horizons, we establish general, distribution-independent, and worst-case optimal bounds on the sensitivity of Bayesian algorithms to prior misspecification. We focus on the Bayesian bandit setting, where a mean-vector ``environment'' $\boldsymbol{\mu}$ is drawn from a distribution $\Priorst$, and rewards for each action are drawn in accordance with $\boldsymbol{\mu}$. We study the performance of Bayesian algorithms which operate according to a \emph{misspecified} prior $\Priorhat$.

\paragraph{Sensitivity of Thompson Sampling and Related Bayesian Bandit Algorithms.} As a concrete example, we consider the expected reward obtained by Thompson sampling with misspecified prior $\Priorhat$ under environments drawn from true prior $\Priorst$.

When the mean rewards lie in the range $[0,1]$, as in the Bernoulli reward setting, we show that that the difference in expected reward between Thompson sampling with $\Priorhat$ and with $\Priorst$ is at most twice the total variation distance between $\Priorst$ and $\Priorhat$ multiplied by the \emph{square} of the horizon length. We prove a lower bound demonstrating that, for worst-case priors, this result is tight up to constants.
Moreover, our upper bound holds for any two priors $\Priorst$ and $\Priorhat$ and suffers no dependence on the complexity of the decision space.

%We show that this average reward is at most twice the total variation $2\|\hat{\mathbb{P}} - \mathbb{P}_{\star}\|_{\mathrm{TV}}$ between the priors. This bound holds for any two priors $\hat{\mathbb{P}}$ and $\mathbb{P}_{\star}$, and suffers no dependence on the complexity of the decision space.

We extend this result in two directions. First, we remove the boundedness requirement on the mean reward range, showing that so long as certain tail probability conditions on the prior means are satisfied, a similar result holds. Second, we generalize beyond Thompson sampling, bounding the prior sensitivity of a broad class of Bayesian bandit algorithms, which we term \emph{$n$-Monte Carlo algorithms}. Our lower bounds extend to this class, verifying sharp dependence on the parameter $n$. %\msdelete{We also provide lower bounds, illustrating that for general priors and reward models, our results are tight up to constant factors.}
%\mscomment{we have some experiments + linear bandits + stuff on sharpness + Chris's lower bounds}
\nipsvminpt

\paragraph{Sample Complexity of Bayesian Meta-Learning.} We apply our prior sensitivity results to the Bayesian bandit meta-learning setting, in which a meta-learner iteratively interacts on bandit instances that are sampled from an unknown prior distribution.
Motivated by our sensitivity analysis we describe a generic algorithmic recipe for Bayesian meta-learning, in which the meta-learner explores for several episodes to estimate the prior and then exploits by instantiating a Bayesian decision-maker with the learned prior.
We formally consider two instantiations of this setup: (1) the Beta-Bernoulli setting where the rewards are Bernoulli and the prior is a product of Beta distributions and (2) the Gaussian-Gaussian setting where the rewards are Gaussian and the prior is a Gaussian (with arbitrary covariance structure) over the means. We note that the Gaussian-Gaussian setting was recently studied in \cite{kveton2021meta} but only for the diagonal covariance setting.

\paragraph{Bayesian Decision-Making Beyond the Bandit Setting.}

A striking feature of our proof is that it makes no explicit reference
to the structure of bandit decision-making. As a consequence, our
results extend seamlessly to both contextual bandits and the most
general Bayesian decision-making problem: Bayesian POMDPs. While our
sensitivity bounds hold almost verbatim in these settings, we note
that estimating the prior may be statistically much more challenging
in these scenarios, so there is no free lunch. To facilitate
readability of the paper, we defer all further discussion and formal
results to~\Cref{app:gen_bayes_dec}.\looseness=-1

\paragraph{Experimental results.}
We complement our meta-learning theory with synthetic experiments in multi-armed bandit and contextual bandit settings.
% investigating the benefits of meta-learning in bandit and contextual bandit settings.
Our experiments show the benefits of (a) meta-learning broadly, (b) estimating higher-order moments of the prior distribution, and (c) using less myopic algorithms like the Knowledge Gradient~\cite{ryzhov2012knowledge} over Thompson sampling when faced with structured environments.\looseness=-1

%% file: body/related.tex
\subsection{Related Work}
\paragraph{Bayesian Decision-Making.}  
Bayesian decision-making broadly refers to a class of algorithms that use Bayesian methods to estimate various problem parameters, and then derive decision/allocation rules from these estimates.
The study of Bayesian decision-making began with the seminal work of Thompson~\cite{thompson1933likelihood}, who introduced the Thompson sampling algorithm for adaptive experiment design in clinical trials. 
Thompson sampling later gained popularity in the reinforcement learning community as a means to solve multi-armed bandit and tabular reinforcement learning problems~\cite{strens2000bayesian,osband2017posterior}, and has been extended in many directions~\cite{abeille2017linear,abeille2018improved,gopalan2014thompson}. 
Recent years have seen the proliferation of other Bayesian decision-making and learning algorithms, including Information Directed Sampling~\cite{russo2016information}, Top-Two Thompson Sampling~\cite{russo2016simple}, and Knowledge-Gradient~\cite{ryzhov2012knowledge}. 

\paragraph{Sensitivity Analysis and Frequentist Regret.} 
The field of robust Bayesian analysis examines the sensitivity of Bayesian inference to prior and model misspecification (c.f.,~\cite{berger1994overview}).
These approaches typically do not consider decision-making, so they do not account for multi-step adaptive sampling inherent in our setting.
More recent works study frequentist regret for Thompson sampling \cite{agrawal2012analysis, kaufmann2012thompson}. 
These guarantees can be interpreted as controlling the sensitivity to arbitrary degrees of prior misspecification, but consequently, they do not provide a precise picture of how misspecification affects performance. 
Moreover, frequentist guarantees for Thompson sampling focus on relatively long learning horizons, so they are less relevant in the context of meta-learning with many short-horizon tasks.\looseness=-1

\paragraph{Short-Horizon Sensitivity.} 
Most closely related to our paper are two previous works on sensitivity of Thompson sampling to small amounts of misspecification in short horizon settings. \cite{liu2016prior} study the sensitivity of Thompson sampling for two-armed bandits when the prior has finite support. 
More recently, \cite{kveton2021meta} study meta-learning with Thompson sampling and derive sensitivity bounds for Thompson sampling in multi-armed bandits with Gaussian rewards and independent-across-arm Gaussian priors.
In contrast to both of these works, the bounds presented in this work apply to arbitrary families of priors, more general decision-making problems, and to more general families of decision-making algorithms. 
Further, as illustrated in \Cref{rem:comparison}, our bounds are also tighter than those achieved by \cite{kveton2021meta} when specialized to their precise setting. 
Finally, our lower bounds demonstrate that the square-horizon factor incurred in \cite{kveton2021meta} is unavoidable for worst-case priors (though perhaps not for their special case). 

\paragraph{Meta-learning and Meta-RL.}  Meta-learning is a classical learning
paradigm in which a learner faces many distinct-but-related
tasks~\cite{thrun1996explanation,thrun1998lifelong,baxter1998theoretical,baxter2000model,hochreiter2001learning}. While
the classical work primarily considered supervised learning tasks,
recent, predominantly empirical, work has focused on
meta-reinforcement learning (Meta-RL), where each task is itself a
decision-making
problem (c.f.,~\cite{wang2016learning,duan2016rl}). This includes
some Bayesian approaches~\cite{humplik2019meta}. While there have been some theoretical results on Meta-RL in various settings~\cite{azar2013sequential,cella2020meta,yang2021impact,hu2021near}, apart
from~\cite{kveton2021meta} we are not aware of other theoretical
treatments with a Bayesian flavor.

%% file: body/setting.tex
%!TEX root = ../neurips_main.tex

\newcommand{\trace}{\operatorname{tr}}
\newcommand{\F}{\operatorname{F}}
\newcommand{\traj}{\uptau}
\newcommand{\signal}{\upsigma}
\newcommand{\Plaw}{P}
\newcommand{\TS}{\mathsf{TS}}
\newcommand{\Hcal}{\mathcal{H}}
\newcommand{\alg}{\mathsf{alg}}
\newcommand{\TVof}[1]{\TV(#1)}
\newcommand{\KLof}[1]{\KL(#1)}
\newcommand{\history}{\uptau}
\newcommand{\Xtil}{\tilde{X}}
\newcommand{\thetast}{\theta^{\star}}
\newcommand{\thetahat}{\hat{\theta}}
\newcommand{\boldtheta}{{\bm{\theta}}}
\newcommand{\boldthetast}{{\bm{\theta}^{\star}}}
\newcommand{\muhat}{\hat{\mu}}
\newcommand{\zhat}{\hat{z}}
\newcommand{\Law}{\mathrm{Law}}
\newcommand{\KL}{\mathrm{KL}}
\newcommand{\TV}{\mathrm{TV}}
\newcommand{\Breg}{D_A}
\newcommand{\zrew}{z^r}
\newcommand{\zts}{z^{\mathrm{TS}}}
\newcommand{\ats}{a^{\mathrm{TS}}}
\newcommand{\tinyskip}{\mkern 1mu}
\newcommand{\matx}{\bm{x}}
\newcommand{\matxbar}{\bar{\matx}}

\newcommand{\matz}{\bm{z}}
\newcommand{\I}{\mathbb{I}}
\newcommand{\sfP}{\mathsf{P}}
\newcommand{\sfE}{\mathsf{E}}
\newcommand{\calB}{\mathcal{B}}
\newcommand{\Qlaw}{\mathbb{Q}}

\newcommand{\trajt}[1]{\traj^{(#1)}}
\newcommand{\boldmut}[1]{\boldmu^{(#1)}}
\newcommand{\expalg}{\mathsf{explore}}
\newcommand{\algt}[1]{\expalg^{(#1)}}

\newcommand{\iidsim}{\overset{\mathrm{i.i.d.}}{\sim}}
\newcommand{\mustar}{\mu_{\star}}
\newcommand{\prior}{\pi}
\newcommand{\Plawof}[1]{\Plaw^{\tinyskip #1}}
\newcommand{\boldmu}{\bm{\mu}}
\newcommand{\calD}{\mathcal{D}}
\newcommand{\boldA}{\mathbf{A}}
\newcommand{\boldr}{\mathbf{r}}
\newcommand{\bolds}{\mathbf{s}}
\newcommand{\bolde}{\mathbf{e}}
\newcommand{\boldu}{\mathbf{u}}
\newcommand{\boldv}{\mathbf{v}}
\newcommand{\boldw}{\mathbf{w}}
\newcommand{\boldx}{\mathbf{x}}
\newcommand{\boldy}{\mathbf{y}}
\newcommand{\boldz}{\mathbf{z}}
\newcommand{\boldI}{\mathbf{I}}

\newcommand\boldalpha{\bm{\alpha}}
\newcommand\boldbeta{\bm{\beta}}

\newcommand\boldnu{\bm{\nu}}
\newcommand\boldPsi{\bm{\Psi}}
\newcommand\normmean{\boldnu_{\star}}
\newcommand\normmeanest{\widehat{\boldnu}}
\newcommand\normcov{\boldPsi_{\star}}
\newcommand\normcovest{\widehat{\boldPsi}}
\newcommand\Normal{\mathcal{N}}
\newcommand\Uniform{\operatorname{Uniform}}
\newcommand\Beta{\operatorname{Beta}}
\newcommand\Bernoulli{\operatorname{Bern}}
\newcommand\Binomial{\operatorname{Bin}}
\newcommand\Categorical{\operatorname{Cat}}
\newcommand\alphast{\bm{\alpha}^{\star}}
\newcommand\alphasti[1]{\alpha^{\star}_{#1}}
\newcommand\betast{\bm{\beta}^{\star}}
\newcommand\betasti[1]{\beta^{\star}_{#1}}
\newcommand\boldq{\bm{q}}

\newcommand\T{{\scriptscriptstyle{\mathsf{T}}}} % transpose
\newcommand\tr{\operatorname{tr}}

\newcommand{\regtheT}{\mathrm{Reg}_{T,\thetast}}
\section{Setting and Notation}
Throughout, we use bold $\mathbf{v}$ to denote vectors and non-bold $v_a$ to denote scalar indices. When the vector $\mathbf{v}_h$ has a subscript, $v_{h,a}$ denotes its coordinates.

\paragraph{Bayesian Bandit Learning under Misspecification.} A Bayesian bandit learning instance is specified by (a) an abstract action space $\calA$, (b) a parametric family of priors $P_{\theta}$ indexed by parameters $\theta \in \Theta$ over mean vectors $\boldmu \in \R^{\calA}$ with coordinates $\mu_a$, and (c) a function $\calD(\cdot):\R^{\calA} \to \Delta(\R^{\calA})$ mapping mean vectors $\boldmu$ to reward vectors $\boldr \in \R^{\calA}$ such that the mean reward under $\calD(\boldmu)$ is $\boldmu$:  $\E_{\boldr \sim \calD(\boldmu)}[\boldr] = \boldmu$.\footnote{In fact, our analysis extends to more general cases where the reward distribution is parameterized by more than just the mean vectors, but we restrict ourselves to the current setting for ease of exposition.} Note that this general setup allows the prior $P_{\theta}$ to encode complex dependencies between the mean rewards $\mu_a$ of actions $a \in \calA$. 

We consider an episodic bandit protocol with horizon $H$. First, $\boldmu \sim P_{\theta}$ is drawn from the prior. Then, at each time step $h =1,2,\dots,H$, the learner's policy, specified by an algorithm $\alg$, selects an action $a_h \in \calA$. Simultaneously, a reward vector $\boldr_h$ is drawn independently from $\calD(\boldmu)$, and the learner observes reward $r_{h} = r_{h,a_h}$.  The choice of action $a_h$ may depend on the partial trajectory $\traj_{h-1} = (a_1,r_1,\dots,a_{h-1},r_{h-1})$.
We let $P_{\theta,\alg}$ denote the joint law over $\boldmu$, and the full trajectory $\traj_H$, while expectations are denoted $E_{\theta,\alg}$.  We abbreviate the full trajectory $\traj = \traj_H$. We denote the cumulative reward 
\begin{align*}
R(\theta,\alg) := E_{\theta,\alg}\left[\sum_{h=1}^H r_h\right] = E_{\theta,\alg}\left[\sum_{h=1}^H \mu_{a_h}\right].
\end{align*}

\paragraph{Bayesian Learning Algorithms.} We study a class of algorithms $\alg(\theta)$ also parameterized by ${\theta \in \Theta}$. For concreteness, the reader may think of $\alg(\theta)$ as corresponding to Thompson sampling, where the learner {internally} computes posteriors {using $P_\theta$ as its prior.}
More general classes of Bayesian algorithms are defined in \Cref{ssec:n_monte_carlo}. We are interested in the consequences of misspecification; that is, interacting with $\boldmu \sim \Plaw_\theta$, but executing $\alg(\theta')$ for some other $\theta' \ne \theta$. Note that our notation for the induced law on the trajectory is $\Plaw_{\theta, \alg(\theta')}$.  

\paragraph{Episodic Bayesian Meta-Learning.} We apply the above framework to the problem of Bayesian meta-learning. Let $\thetast \in \Theta$ be a ground-truth parameter. At each episode $t = 1,2,\dots,T$, a mean parameter $\boldmut{t}$ is drawn i.i.d.~from $P_{\thetast}$. Simultaneously, the learner commits to a (potentially non-Bayesian) exploration strategy $\algt{t}$ and collects the induced trajectory $\trajt{t}$. At the end of $T$ episodes, the learner selects a parameter $\thetahat \in \Theta$ as a function of $\trajt{1},\dots,\trajt{T}$. The learner's performance is evaluated on the expected reward of the plug-in algorithm on $\thetahat$: $R(\thetast,\alg(\thetahat))$.

\newcommand{\klarg}[2]{\mathrm{KL}(#1 \parallel #2)}
\newcommand{\Bregarg}[2]{\mathrm{D}_A(#1 \parallel #2)}
\newcommand{\range}{\mathrm{diam}}
\newcommand{\tail}{\Psi}

\paragraph{Further notation.} Given two probability distributions $P$ and $Q$ over the same probability space $(\Omega,\mathcal{F})$, we denote their total variation $\tvarg{P}{Q} :=\sup_{\mathcal{E}\in \mathcal{F}}|P[\calE] - Q[\calE]|$ and Kullback-Leibler divergence $\klarg{P}{Q}$. If $P$ is a joint distribution of random variables $(X,Y,Z,\dots)$, $P(X)$ denotes the marginal of $X$ under $P$, and $P(Y | X)$ the conditional distribution (as a function of random variable $X$). We define the diameter of a mean vector as $\range(\boldmu) := \sup_{a \in \calA} \mu_a - \inf_{a \in \calA} \mu_a$, which is a random variable when $\boldmu$ is drawn from $P_{\theta}$. Throughout, $\log(\cdot)$ denotes the natural logarithm. Given a space $\calX$, we let $\Delta(\calX)$ be the set of probability distributions on $\calX$; see \Cref{sec:key_prop_tv} for measure-theoretic considerations.

%% file: body/smoothness_new.tex
%!TEX root = ../arxiv_main.tex
\newcommand{\boldmuhat}{\hat{\boldmu}}
\newcommand{\Pst}{P^\star}
\newcommand{\Phat}{\hat{P}}
\newcommand{\normal}{\mathcal{N}}
\newcommand{\cvar}{\Psi}
\newcommand{\cvarbar}{\bar{\Psi}}
\newcommand{\veps}{\varepsilon}
\newcommand{\boldphi}{\boldsymbol{\phi}}
\newcommand{\topk}{\mathrm{Top}\text{-}k}

\newcommand{\kshot}{k\text{-}\mathrm{TS}}
\newcommand{\nshot}{n\text{-}\mathrm{TS}}

\newcommand{\boldmubar}{\bar{\boldmu}}
\newcommand{\mubar}{\bar{\mu}}

\newcommand{\twompc}{2\text{-}\mathrm{RHC}}
\newcommand{\boldmutil}{\tilde{\boldmu}}
\newcommand{\boldrtil}{\tilde{\boldr}}
\newcommand{\rtil}{\tilde{r}}
\newcommand{\mutil}{\tilde{\mu}}

\section{Prior Sensitivity in Bayesian Learning}
\label{sec:prior_sensitivity}
This section states sensitivity bounds for various Bayesian bandit algorithms and families of priors, starting with the concrete instance of Thompson sampling under priors with bounded-range means.  Our results extend almost verbatim to more general decision-making tasks such as contextual bandits; see \Cref{app:gen_bayes_dec} for further details. Throughout, we use the fact that the posterior distribution of the mean $\boldmu$ given trajectories $\traj_h$ does not depend on the choice of learning algorithm $\alg$; hence, we denote these posteriors $P_{\theta}[\cdot \mid \traj_{h-1}]$.\footnote{Note that whenever $\traj_h$ lies in the support of $P_\theta$, the posterior $P_\theta[\boldmu \mid \traj_{h-1}]$ is well-defined and unique, even if $\traj_{h-1}$ was generated by interacting with mean $\mu \sim P_{\theta'}$ for some $\theta'\neq \theta$. When $\traj_{h-1}$ does not lie in the support of $P_\theta$, we allow $P_{\theta}[\boldmu \mid \traj_{h-1}]$ to be any distribution over $\boldmu$ (for concreteness, one may default to $P_{\theta}[\boldmu].$) Note, however, that although $P_{\theta}[\boldmu \mid \traj_{h-1}]$ may not be uniquely defined, $P_{\theta,\alg}[\traj_{h-1} \mid \boldmu]$ is always uniquely defined and  independent of $\theta$.}

Recall the classical Thompson sampling algorithm: at each step $h$, $\TS(\theta)$ draws a mean $\tilde{\boldmu}_h \sim P_{\theta}[\cdot \mid \traj_{h-1}]$ and selects the reward-maximizing action $a_h \in \argmax_{a} \tilde{\mu}_{h,a}$.  We say that the prior $P_{\theta}$ is \emph{$B$-bounded} if $P_{\theta}[\range(\boldmu) \le B] = 1$. For Thompson sampling under $B$-bounded priors, we have the following result:

\begin{corollary}\label{cor:TS} Let $P_{\theta}$ be $B$-bounded. Then,  the suboptimality of misspecified Thompson sampling $\TS(\theta')$ on instance $\theta$ is at most
\begin{align*}|R(\theta,\TS(\theta)) - R(\theta,\TS(\theta'))| \le  2H^2 \cdot \tvarg{P_\theta}{P_{\theta'}} \cdot B.
\end{align*}
\end{corollary}
\Cref{cor:TS} follows directly from \Cref{thm:n_montecarlo_body}, which we state in \Cref{sec:main_general_sensitivity}, and which generalizes the statement of the corollary
along two axes: to a more general family of Bayesian algorithms that we call ``$n$-Monte Carlo''
and to less restrictive conditions on the behavior of $\range(\boldmu)$, such as sub-Gaussian tails.
Due to lack of space, we focus on the first such generalization; the second direction is more technical in nature, and we leave its exposition to  \Cref{sec:general_sensitivity}.

\subsection{$n$-Monte Carlo algorithms}\label{ssec:n_monte_carlo}
Unfortunately, for arbitrary Bayesian bandit algorithms, the behavior under two different priors cannot always be controlled in terms of the total variation distance of their priors. Indeed, consider an algorithm that always pulls a particular arm $a^\star$ if the prior places any probability mass on a mean for which $a^\star$ is best; clearly, this algorithm's behavior is not robust to small changes in its prior distribution. However, many important Bayesian bandit algorithms, such as Thompson sampling, are not arbitrary functions of their priors; rather, they select actions based on their internal posterior distribution in a relatively stable manner.
We call such algorithms \emph{$n$-Monte Carlo algorithms}.
\begin{definition}[$n$-Monte Carlo algorithm]\label{defn:n_monte_carlo} Given $n > 0$, we say that a family of algorithms $\alg(\cdot)$ parameterized by $\theta \in \Theta$ is $n$-Monte Carlo if, for any $\theta,\theta'$, step $h \ge 1$, and partial trajectory $\traj_{h-1}$,
\begin{align*}
\tvarg{P_{\alg(\theta)}(a_{h} \mid \traj_{h-1})}{P_{\alg(\theta')}(a_{h} \mid \traj_{h-1})} \le n\cdot\tvarg{P_{\theta}(\boldmu \mid \traj_{h-1})}{P_{\theta'}(\boldmu \mid \traj_{h-1})}.
\end{align*}
\end{definition}

In words, $n$-Monte Carlo algorithms are those Bayesian algorithms for which small changes in the posterior distribution result in small changes (up to a multiplicative factor of $n$) in the distribution over actions. Note that on the left-hand side, we do not need to specify the true $\theta^\star$, because each algorithm's choice of an action can only depend on $\traj_{h-1}$.
The nomenclature arises because any algorithm that selects actions based exclusively on $n$ \emph{samples} from its posterior $P_\theta(\boldmu \mid \traj_{h-1})$ is $n$-Monte Carlo. 
\iftoggle{neurips}
{
	\input{body/neurips_mc_algs_body}
}
{

	We now elaborate upon various examples of algorithms satisfying the $n$-Monte Carlo property.
	\input{body/body_mc_algorithms}

}

\begin{comment}
\begin{definition}[$k$-shot Thompson Sampling] The \emph{Optimistic $k$-shot Thompson sampling algorithm},
denoted $\kshot(\theta)$, selects action $a_h$ at time $h$ by sampling $k$ means $\boldmu^{(1)}, \ldots, \boldmu^{(k)}$ independently from the posterior $P_{\theta}[\cdot \mid \traj_{h-1}]$, and selecting $
a_h \in \argmax_{a} \max \{\boldmu_a^{(1)},\ldots, \boldmu_a^{(k)}\}.$
Note that the classical Thompson sampling algorithm, which we denote $\TS(\theta)$, is identical to $\kshot(\theta)$ for $k=1$.
\end{definition}
\end{comment}

\subsection{General Sensitivity Upper and Lower Bounds}
\label{sec:main_general_sensitivity}

We are now ready to state a general prior sensitivity bound for $n$-Monte Carlo algorithms. For simplicity, we state our bounds for $B$-bounded priors, that is, $P_\theta[\range(\boldmu) \le B] = 1$, and under a natural sub-Gaussian tail condition stated formally in \Cref{sec:general_sensitivity} (\Cref{thm:n_montecarlo_tails}).

\begin{theorem}\label{thm:n_montecarlo_body} Let $\alg(\cdot)$ be an $n$-Monte Carlo family of algorithms on horizon $H \in \N$, and let $\theta,\theta' \in \Theta$. Setting $\veps = \tvarg{P_\theta}{P_{\theta'}}$,
we have the following guarantees.
\begin{itemize}
	\item[(a)] If $P_\theta$ is $B$-bounded, then $|R(\theta,\alg(\theta)) - R(\theta,\alg(\theta'))| \le 2n H^2 \veps \cdot B$.
	\item[(b)] If $P_\theta$ is coordinate-wise $\sigma^2$-sub-Gaussian, then
	\begin{align*} |R(\theta,\alg(\theta)) - R(\theta,\alg(\theta'))| \le 2nH^2 \veps \left( \range(E_\theta[\boldmu]) + \sigma \left(8 + 5 \sqrt{\log\left( \tfrac{|\calA|^2}{\min\{1,2nH\epsilon\}}\right)} \right)  \right).
	\end{align*}
	\begin{comment}\item[(c)] If $P_\theta$ is coordinate-wise $(\sigma^2, \nu)$-sub-Gamma, then \[|R(\theta,\alg(\theta)) - R(\theta,\alg(\theta'))| \le 2nH^2
	 \veps \left( \range(E_\theta[\boldmu]) + \sigma \left(4 + 3 \sqrt{\log\left( \frac{|\calA|^2}{2nH\epsilon}\right)} \right) + \nu \left(6 + 4 \log \left( \frac{|\calA|^2}{2nH\epsilon} \right) \right)  \right).\]
	 \end{comment}
\end{itemize}
\end{theorem}

Next, we complement our upper bound with a lower bound that matches \Cref{thm:n_montecarlo_body}(a) for $n$-shot Thompson sampling (an $n$-Monte Carlo algorithm) up to a multiplicative constant:
\begin{theorem}[Lower Bound, Informal]\label{thm:lb_informal} For any parameter $n \in \N$, horizon $H \gg 1$, number of arms $N =|\calA| \gg H$, and separation $\epsilon \ll 1/nH$, there exist two priors $P_\theta$ and $P_{\theta'}$ over bounded means $\boldmu \in [0,1]^{N}$ such that $\tvarg{P_\theta}{P_{\theta'}} = \epsilon$ and
\begin{align*}
R(\theta,\nshot(\theta)) \ge R(\theta,\nshot(\theta')) + (1 - o(1))\cdot\smash{\frac{nH^2 \epsilon}{2}},
\end{align*}
where the $o(1)$ decays to zero as $1/H$, $H/N$, $\epsilon n H \to 0$.
\end{theorem}
See \Cref{thm:an_lb} for a precise, quantitative statement and \Cref{app:LB} for a full proof.

\iftoggle{neurips}
{
	\input{body/neurips_proof_sketch}

}
{
	\input{body/arxiv_proof_sketch}

}

%% file: body/neurips_mc_algs_body.tex
%!TEX root = ../neurips_main.tex
However, the definition is more general and
in \Cref{app:mc_algs} we describe various algorithms that satisfy the $n$-Monte Carlo property,
summarizing key insights here:
\begin{itemize}
    \item We show that $\TS(\theta)$ is $1$-Monte Carlo.
    \item We introduce a generalization of Thompson sampling, which we call $k$-shot Thompson sampling ($\kshot(\theta)$), that samples $k$ means $\boldmutil_1, \ldots, \boldmutil_k$ i.i.d. from the posterior $P_{\theta}[\cdot \mid \traj_{h-1}]$, and selects the action $a_h \in \argmax_{a} \max \{\tilde{\mu}_{1,a},\ldots, \tilde{\mu}_{k,a}\}$. We show that $\kshot(\theta)$ is $k$-Monte Carlo.
    \item We introduce a Monte Carlo approximation of the knowledge gradient algorithm~\cite{ryzhov2012knowledge}, which we call two-step Receding Horizon Control ($\twompc(\theta)$). This algorithm is non-myopic in that it chooses an action that maximizes the expected value at the subsequent time (according to its own posterior updates). We show that when $\calA$ is finite, $\twompc(\theta)$ is $n$-Monte Carlo for some $n$ that is polynomial in $|\calA|$ and the number of Monte Carlo samples it draws from its posterior.
\end{itemize}

%% file: body/body_mc_algorithms.tex
%!TEX root = ../arxiv_main.tex

%%%%%% NOTE! This file is used in both the arxiv and the neurips version

\paragraph{$k$-Shot Thompson Sampling.}  The first is a natural generalization of Thompson Sampling, where one draws not one but $k \in \N$ mean vectors $\boldmutil^{(i)}$ from the posterior at each step $h$, and selects the action for which one of the $k$ draws attains the highest observed realization: $a_h \in \argmax_{a} \max_i \mutil^{(i)}_a$. See \Cref{alg:kshot_TS}.

\begin{algorithm}
\caption{$k$-Shot Thompson Sampling ($\kshot(\theta)$)}\label{alg:kshot_TS}
\begin{algorithmic}[1]
\State{}\textbf{Input:} Prior $\theta$, sample size $k \in \N$
    \For{$h=1, \ldots, H$}
    \Statex{}\algcomment{action selection at step $h$}
    \State{}{Sample} $\boldmutil^{(1)}, \ldots, \boldmutil^{(k)}$ independently from the posterior $P_{\theta}[\cdot \mid \traj_{h-1}]$
    \State{}Select action $ a_h \in \argmax_{a} \max \{\mutil_{a}^{(1)},\ldots, \mutil_{a}^{(k)}\}.$
    \EndFor
\end{algorithmic}
\end{algorithm}
We show that $\kshot(\theta)$ is $k$-Monte Carlo.
\begin{restatable}{lemma}{lemkshot}\label{lem:kshot} For every $k \ge 1$, $\kshot(\cdot)$ is $k$-Monte Carlo. In particular, $\TS(\cdot)$ is $1$-Monte Carlo.
\end{restatable}
\paragraph{Generalized Posterior Sampling.} \Cref{lem:kshot} follows from an analysis of a more general recipe for $n$-Monte Carlo algorithms. \Cref{alg:general_post_sample} describes describe a family of posterior sampling algorithms $\kfpost(\theta)$, parameterized by prior $\theta$ and determined by a sample size $k \in \N$ and functions $f_1, \ldots, f_H: \R^{\calA \times n} $ from $\R^{\calA \times k}$ to probability distribution  $\Delta^\calA$ over actions. At each step $h$, $k$ means $\boldmutil^{(1)}, \ldots, \boldmutil^{(k)}$ are sampled from the posterior, and an action $a_h$ is drawn from the probability distribution $f_h(\cdot \mid \boldmutil^{(i)}, \ldots, \boldmutil^{(k)})$ induced by evaluating $f_h$ on the sampled means. 

\begin{algorithm}
\caption{$(k,f_{1:H})$-Posterior Sampling ($\kfpost(\theta)$)}\label{alg:general_post_sample}
\begin{algorithmic}[1]
\State{}\textbf{Input:} Prior $\theta$, sample size $k \in \N$, functions $f_1, \ldots, f_H: \R^{\calA \times k} \rightarrow \Delta^{\calA}$. 
    \For{$h=1, \ldots, H$}
    \Statex{}\algcomment{action selection at step $h$}
    \State{}{Sample} $\boldmutil^{(1)}, \ldots, \boldmutil^{(k)}$ independently from the posterior $P_{\theta}[\cdot \mid \traj_{h-1}]$
    \State{}Select action $ a_h \sim f_h(\cdot \mid \boldmutil^{(1)}, \ldots, \boldmutil^{(k)})$.
    \EndFor
\end{algorithmic}
\end{algorithm}

Note that $\kshot(\theta)$ corresponds to the special case where $f_h = f$ is constant across $h$, and places a dirac mass on the action for which $\max_i \mutil^{(i)}_a$ is largest (with a suitable tie-breaking rule). In particular, $\TS(\theta)$ is a special case of $\kfpost(\theta)$ with $k = 1$. Other algorithms in the family include the rule which selects the arm with largest sample average/sum  $a_h \in \sum_{i} \mutil^{(i)}_a$, or a policy which selects $a_h$ according to a softmax distribution on the sums $\sum_{i} \mutil^{(i)}_a$. 

The following lemma shows that, regardless of the functions $f_1,\dots,f_H$, $\kfpost(\theta)$ is $k$-Monte Carlo. Note that \Cref{lem:kshot} follows as a special case.
\begin{lemma}
\label{lem:samples_to_n_mc}
For any $k \in \N$ and $f_{1},\dots,f_H: \R^{\calA \times k} \to \Delta^{\calA}$, the family of Bayesian algorithms given by $\kfpost(\cdot)$ is $k$-Monte Carlo.
\end{lemma}
The proof of \Cref{lem:samples_to_n_mc} is quite intuitive, and is given in \Cref{sec:lem:samples_to_n_mc}.

\paragraph{Receding Horizon Control.} Some Bayesian bandit algorithms do not exactly satisfy the conditions of \Cref{lem:samples_to_n_mc} but are nonetheless $n$-Monte Carlo. As an example, we consider a sampling-based implementation of two-stage receding horizon control, $\twompc(\theta)$,  detailed \Cref{alg:two_mpc}. $\twompc(\theta)$ selects an action $a$ which maximizes $V_{a}$, which can be thought of as a discounted two-step value function, balancing (a) selection of actions with large {posterior} means  and (b) selection of actions that are sufficiently informative such that the best action for ``look-ahead'' means sampled from the next stage yield large reward. This balance is controlled by a discount parameter $\alpha \in [0,1]$.

\begin{algorithm}
\caption{Two-Step Receding Horizon Control ($\twompc(\theta)$)}\label{alg:two_mpc}
\begin{algorithmic}[1]
\State{}\textbf{Input:} Prior $\theta$, discount parameter $\alpha \in [0,1]$, sample sizes $k_1,k_2 \in \N$.
    \For{$h=1, \ldots, H$}
    \Statex{}\algcomment{action selection at step $h$}
    \For{actions $a \in \calA$ and $i = 1,\ldots, k_1$}
    \State{}{Sample} mean $\boldmutil^{(a,i)}\sim P_{\theta}[\cdot \mid \traj_{h-1}]$ 
    \State{}Sample reward vector $\boldrtil^{(a,i)} \sim \calD(\boldmutil^{(a,i)})$
    \For{$j = 1,2,\dots, k_2$}
    \State{}{Sample} ``look-ahead'' means $\boldmuhat^{(a,i,j)} \sim P_{\theta}[\in \cdot \mid \traj_{h-1} \text{ and } \{r_{h+1,a} = \rtil_{a}^{(a,i)}\}]$ 
    \EndFor
    \EndFor
    \State{}\textbf{Select} action $a_h \in \argmax_{a} V_{a,h}$, where 
    \begin{align}
    V_{a} :=  \sum_{i=1}^{k_1}\left((1-\alpha)  \mutil_{a}^{(a,i)} + \alpha\left(\max_{a'}\frac{1}{k_2}\sum_{j=1}^{k_2}\muhat_{a'}^{(a,i,j)}\right)\right). \label{eq:Va}
    \end{align}
    \EndFor
\end{algorithmic}
\end{algorithm}

At the one extreme $\alpha = 0$, $\twompc(\theta)$ has no look-ahead, and is a special case of $(k,f_{1:H})$-Posterior Sampling with $k = |\calA|k_1$. At the other extreme $\alpha = 1$, $\twompc(\theta)$ disregards means sampled from the posterior, and only evaluates actions $a$ by how informative they are about the look-ahead means. This latter case, $\alpha = 1$,  in fact gives a Monte Carlo approximation of the classical knowledge gradient algorithm~\cite{ryzhov2012knowledge}. The following lemma verifies the Monte Carlo property for all choices of $\alpha$.
\begin{restatable}{lemma}{lemtwompc}
\label{lem:twompc_n_mc}
$\twompc(\cdot)$ with budgets $k_1,k_2$ and finite action set is $n$-Monte Carlo for $n = |\calA|\cdot k_1\cdot (2k_2+3)$, regardless of discount $\alpha \in [0,1]$.
\end{restatable}
The proof of the above lemma is {provided} in \Cref{sec:MC_prop_two_mpc}.

%% file: body/neurips_proof_sketch.tex
%!TEX root = ../neurips_main.tex
\paragraph{Proof ideas.}
One of the key ingredients in the proof of \Cref{thm:n_montecarlo_body}, and a result which may be of independent interest, is the following bound on the total variation of the trajectory of an algorithm run with the true prior and the same algorithm run with an incorrect prior.
\begin{proposition}
\label{prop:prior_sensitivity_tv_bound}Let $\alg(\cdot)$ be an $n$-Monte Carlo family of algorithms on horizon $H \in \N$. Then,
\begin{align*}
\tvarg{P_H}{P_H'} \le 2nH\cdot \tvarg{P_\theta}{P_{\theta'}},\end{align*}
where $P_H = P_{\theta,\alg(\theta)}(\boldmu,\traj_H)$ and $P_H' = P_{\theta,\alg(\theta')}(\boldmu,\traj_H)$.
\end{proposition}
A full proof is given in \Cref{app:prop_tv_sens}.  The factor of $H$ arises from a telescoping argument (\Cref{lem:perf_diff_bandits}) based on the performance-difference lemma \cite{kakade2003sample}.

For $B$-bounded priors, \Cref{prop:prior_sensitivity_tv_bound} directly translates into the sensitivity bound in \Cref{thm:n_montecarlo_body}(a), where the difference in rewards can be bounded as $B H$ times the probability that the trajectory of $\alg(\theta)$ differs from the trajectory of $\alg(\theta')$. Addressing more general tail conditions like sub-Gaussianity requires more care; see \Cref{app:sensitivity} for details.

%% file: body/arxiv_proof_sketch.tex
%!TEX root = ../arxiv_main.tex
\subsection{Proof sketch of \Cref{thm:n_montecarlo_body}.}
One of the key ingredients in the proof of \Cref{thm:n_montecarlo_body}, and a result which may be of independent interest, is the following bound on the total variation of the trajectory of an algorithm run with the true prior and the same algorithm run with an incorrect prior.
\begin{proposition}
\label{prop:prior_sensitivity_tv_bound}Let $\alg(\cdot)$ be an $n$-Monte Carlo family of algorithms on horizon $H \in \N$. Then,
\begin{align*}
\tvarg{P_H}{P_H'} \le 2nH\cdot \tvarg{P_\theta}{P_{\theta'}},\end{align*}
where $P_H = P_{\theta,\alg(\theta)}(\boldmu,\traj_H)$ and $P_H' = P_{\theta,\alg(\theta')}(\boldmu,\traj_H)$.
\end{proposition}
We provide a full proof of the proposition below, but let us first explain how the proposition implies \Cref{thm:n_montecarlo_body}. For $B$-bounded priors, \Cref{prop:prior_sensitivity_tv_bound} directly translates into the sensitivity bound in \Cref{thm:n_montecarlo_body}(a), where the difference in rewards can be bounded as $B H$ times the probability that the trajectory of $\alg(\theta)$ differs from the trajectory of $\alg(\theta')$. Addressing more general tail conditions like sub-Gaussianity requires more care; see \Cref{app:sensitivity} for details.

\paragraph{Proof of \Cref{prop:prior_sensitivity_tv_bound}.}
The proof consists of two steps. First, we introduce a telescoping decomposition based on the performance-difference lemma \cite{kakade2003sample}.
\input{body/app_prior_sensitivity_tv_bound}

%The factor of $H$ arises from a telescoping argument (\Cref{lem:perf_diff_bandits}) based on the performance-difference lemma \cite{kakade2003sample}.

%% file: body/app_prior_sensitivity_tv_bound.tex
%!TEX root = ../arxiv_main.tex
\begin{lemma}\label{lem:perf_diff_bandits} For any two algorithms $\alg,\alg'$, it holds that
\begin{align*}
\tvarg{P_{\theta,\alg}(\boldmu,\traj_H)}{P_{\theta,\alg'}(\boldmu,\traj_H)} \le \sum_{h=1}^H \Expop_{\traj_{h-1}\sim P_{\theta,\alg}}\left[\tvarg{P_{\alg}(a_h \mid \traj_{h-1})}{P_{\alg'}(a_h \mid \traj_{h-1})}\right],
\end{align*}
where on the right hand side, we consider the total variation distance between the conditional distribution of $a_h$ under $\alg,\alg'$ given $\traj_{h-1}$ and take expectation over $\traj_{h-1}$ under $P_{\theta,\alg}$.
%% with expectation under $\traj_{h-1}$ under $E_{\theta,\alg}$.
\end{lemma}
\begin{proof}[Proof Sketch] Using the variational characterization of total variation, we represent the total variation between the two measures as the supremum of differences in rewards under two Markov reward process induced by $P_{\theta,\alg}(\boldmu,\traj_H)$ and $P_{\theta,\alg'}(\boldmu,\traj_H)$. The decomposition then follows from a careful application of the performance difference lemma. See \Cref{sec:proof:lem:perf_diff_bandits} for the full proof. 
\end{proof}

% \subsection{Proof of \Cref{thm:n_montecarlo}}
\newcommand{\textb}{\text{(b)}}
\newcommand{\texta}{\text{(a)}}
\newcommand{\trajbar}{\bar{\traj}}

\begin{comment}Specializing \Cref{lem:perf_diff_bandits} to parameterized algorithms $\alg(\theta),\alg(\theta')$ satisfying the $n$-Monte Carlo condition (\Cref{defn:n_monte_carlo}), and using our short hand $P_H = P_{\theta,\alg(\theta)}(\boldmu,\traj_H)$ and $P_H' = P_{\theta,\alg(\theta')}(\boldmu,\traj_H)$, we find
\begin{align}
\tvarg{P_H}{P_H'} \le n\sum_{h=1}^H \Expop_{\traj_{h-1}\sim P_{\theta,\alg(\theta)}}\left[\tvarg{P_{\theta}(\boldmu \mid \traj_{h-1})}{P_{\theta'}(\boldmu \mid \traj_{h-1})}\right]. \label{eq:intermediate_PH_bound}
\end{align}
\end{comment}
In the second stage of the proof, we apply the following rather general lemma.
\begin{lemma}[Fundamental De-conditioning Lemma]\label{lem:fundamental_deconditioning} Let $Q$ and $Q'$ be two measures on a pair of random variables $(X,Y)$ such that the conditionals of $X$ given $Y$ coincide: $Q(X \mid Y) = Q'(X \mid Y)$ almost surely. Then, 
\begin{align*}
\E_{X \sim Q} \tvarg{Q(Y\mid X)}{Q'(Y\mid X)} \le 2\tvarg{Q(Y)}{Q'(Y)}. 
\end{align*}
\end{lemma}
\begin{proof}[Proof of \Cref{lem:fundamental_deconditioning}] We first review the essential properties of total variation used in the proof; we then turn to applying said properties to establish the lemma.

\paragraph{Properties of Total Variation.} Let $P$ and $P'$ over jointly distributed random variables $(X,Y)$. First, if the marginals under $X$ coincide, then their total variation can be expressed as the expected total variation between the conditions $Y \mid X$; that is,
\begin{align}
\text{if } P(X) = P'(X), \quad \text{then}\quad \tvarg{P(X,Y)}{P'(X,Y)} &= \E_{X \sim P} \tvarg{P(Y \mid X)}{P'(Y \mid X)}. \label{eq:tv_same_marg_in_proof}
\end{align}
On the other hand, if their conditionals of $Y \mid X$ coincide, then we have the following simplification:
\begin{align}
\text{if } P(Y \mid X) = P'(Y \mid X), \quad  \text{then}\quad \tvarg{P(X,Y)}{P'(X,Y)} &= \tvarg{P(X)}{P'(X)}. \label{eq:tv_same_conditional_in_proof}
\end{align}
\Cref{eq:tv_same_marg_in_proof} is established in 
\Cref{lem:tv_same_marginal}, and \Cref{eq:tv_same_conditional_in_proof} in \Cref{lem:tv_same_conditional}\iftoggle{neurips}{}{; both lemmas are formally stated and proven in the appendix}. We shall also use that the total variation satisfies  the triangle inequality $(\tvarg{P}{P'} \le \tvarg{P}{P''} +\tvarg{P'}{P''})$ and the data-processing inequality ($\tvarg{P(X)}{P'(X)} \le \tvarg{P(X,Y)}{P(X,Y)}$), stated formally and proven in \Cref{lem:key_tv_props}.

\paragraph{Main proof.} We introduce an interpolating law $Q_{\to}$ such that $Q_{\to}(X) = Q(X)$ and $Q_{\to}(Y \mid X) = Q'(Y \mid X)$. Then
\begin{align*}
\E_{X \sim Q} \tvarg{Q(Y\mid X)}{Q'(Y\mid X)}
&= \E_{X \sim Q} \tvarg{Q(Y\mid X)}{Q_{\to}(Y\mid X)} \\
&\overset{(i)}{=} \tvarg{Q(X,Y)}{Q_{\to}(X,Y)}\\
&\overset{(ii)}{\le}  \underbrace{\tvarg{Q_{\to}(X,Y)}{Q'(X,Y)}}_{(a)} + \underbrace{\tvarg{Q(X,Y)}{Q'(X,Y)}}_{(b)}, 
\end{align*}
where equality $(i)$ uses \eqref{eq:tv_same_marg_in_proof} given the fact that $Q(X) = Q_\to(X)$, and where $(ii)$ applies the triangle inequality. This leaves us with two terms, $(a)$ and $(b)$. First we upper bound term $(a)$ by term $(b)$:
\begin{align*}
\tvarg{Q_{\to}(X,Y)}{Q'(X,Y)} &= \tvarg{Q_{\to}(X)}{Q'(X)}\\
&= \tvarg{Q(X)}{Q'(X)} \le \tvarg{Q(X,Y)}{Q'(X,Y)} =: (b).
\end{align*}
Above, the first equality uses $Q_{\to}(Y \mid X) = Q'(Y \mid X)$ to invoke \eqref{eq:tv_same_conditional_in_proof}, the second the fact that $Q_{\to}(X) = Q(X)$, and the final equality applies the data-processing inequality described above. Hence, 
\begin{align*}
\E_{X \sim Q} \tvarg{Q(Y\mid X)}{Q'(Y\mid X)} &\le 2\tvarg{Q(X,Y)}{Q'(X,Y)}.
\end{align*}
Finally, since $Q(X \mid Y) = Q'(X \mid Y)$, \Cref{eq:tv_same_conditional_in_proof} entails that $\tvarg{Q(X,Y)}{Q'(X,Y)} = \tvarg{Q(Y)}{Q'(Y)}$.
%\akcomment{The references to lemmas in the appendix makes this proof kind of hard to read. However, if I understand it, most of these lemmas are pretty obvious (e.g., Lemma~\ref{lem:tv_same_marginal} is just pushing something inside an absolute value. Basically referring to the lemmas makes the proof look more complicated than it actually is and since it's in the main body, maybe we can try to simplify by just using names instead (e.g., triangle inequality, data-processing) without the pointers.}
\end{proof}
Using our shorthand $P_H = P_{\theta,\alg(\theta)}(\boldmu,\traj_H)$ and $P_H' = P_{\theta,\alg(\theta')}(\boldmu,\traj_H)$, we can finish the proof of \Cref{prop:prior_sensitivity_tv_bound} as follows:
\begin{align*}
    \tvarg{P_{H}}{P'_{H}}
    &\stackrel{(i)}{\leq} \sum_{h=1}^H \Expop_{\traj_{h-1}\sim P_{\theta,\alg(\theta)}}\left[\tvarg{P_{\alg(\theta)}(a_h \mid \traj_{h-1})}{P_{\alg(\theta')}(a_h \mid \traj_{h-1})}\right] \\
    &\stackrel{(ii)}{\leq} n \sum_{h=1}^H \Expop_{\traj_{h-1}\sim P_{\theta,\alg(\theta)}}\left[\tvarg{P_{\theta}(\boldmu \mid \traj_{h-1})}{P_{\theta'}(\boldmu \mid \traj_{h-1})}\right] \\
    &\stackrel{(iii)}{=} n \sum_{h=1}^H \Expop_{\traj_{h-1}\sim P_{\theta,\alg(\theta)}}\left[\tvarg{P_{\theta,\alg(\theta)}(\boldmu \mid \traj_{h-1})}{P_{\theta',\alg(\theta)}(\boldmu \mid \traj_{h-1})}\right] \\
    &\stackrel{(iv)}{\leq} 2n \sum_{h=1}^H\tvarg{P_{\theta}(\boldmu)}{P_{\theta'}(\boldmu)} 
    = 2nH \cdot \tvarg{P_{\theta}(\boldmu)}{P_{\theta'}(\boldmu)}.
\end{align*}
Here, $(i)$ follows from \Cref{lem:perf_diff_bandits}, $(ii)$ is the definition of $n$-Monte Carlo, $(iii)$ follows from the observation that the conditional distribution of $\boldmu$ given $\tau_{h-1}$ does not depend on the algorithm that helped generate $\tau_{h-1}$, and $(iv)$ follows from \Cref{lem:fundamental_deconditioning}, where we have used the fact that $P_{\theta,\alg(\theta)}(\traj_{h-1} \mid \boldmu) = P_{\theta',\alg(\theta)}(\traj_{h-1} \mid \boldmu)$, i.e. the conditional distribution of the trajectories does not depend on the prior when conditioning on the true mean $\boldmu$.

%% file: body/meta_learning_section.tex
%!TEX root = ../arxiv_main.tex

%\newpage
\section{Meta-learning}

In this section, we apply the {above} prior sensitivity guarantees to episodic Bayesian meta-learning and obtain sample-efficiency guarantees for canonical Bayesian bandit setups.

Suppose an episodic Bayesian meta-learner uses an exploration strategy $\algt{t}$ in $T$ episodes and computes an estimate $\thetahat = \thetahat(\trajt{1},\dotsc,\trajt{T})$ of the ground-truth parameter $\thetast$.
Suppose further that, for any $\veps,\delta \in (0,1)$, with probability at least $1-\delta$ over the realizations of the episodes and internal randomization of the meta-learner, the estimate $\thetahat$ satisfies $\tvarg{P_{\thetast}}{P_{\thetahat}} \leq \veps$.
Then, Theorem~\ref{thm:n_montecarlo_body} implies that, for any $n$-Monte Carlo algorithm $\alg(\cdot)$, the relative performance of $\smash{\alg(\thetahat)}$ compared to $\alg(\thetast)$ is (essentially) bounded as $\tilde{O}(nH^2\veps)$ over horizon $H$.

Our task of designing meta-learners is thus reduced to that of designing estimators (and exploration strategies) for $\thetast$ { that enjoy convergence guarantees in TV distance}.
{This is quite a general recipe that can produce concrete meta-learning algorithms in many Bayesian bandit settings.}
We explain how to do so in two setups:
(1) $P_{\boldthetast}$ is a product of Beta distributions, and the rewards are Bernoulli;
(2) $P_{\boldthetast}$ is a multivariate Gaussian and the rewards are Gaussian.
%We now explain how such an estimator is derived for the two most classical exponential-family distributions: Beta distributions with Bernoulli rewards and Gaussian distributions with Gaussian rewards.

%Our general approach lies in sampling two arms for the exploration rounds and subsequently using those samples to create a consistent estimator. The prior sensitivity bound then automatically provides sample-efficiency guarantees for any $n$-Monte-Carlo algorithm.

%\subsection{Bounding the prior sensitivity by parameter sensitivity for exponential families}
%
%\tlcomment{HERE}

\subsection{Beta Priors and Bernoulli Rewards}

We first consider the situation where the prior distribution is a product of Beta distributions $P_{\boldthetast} = \bigotimes_{a \in \calA} \Beta(\alphasti{a},\betasti{a})$ and the reward distribution is a product of Bernoulli distributions $\calD(\boldmu) = \bigotimes_{a \in \calA} \Bernoulli(\mu_a)$.
Recall that $\Beta(\alpha,\beta)$ for $\alpha > 0$ and $\beta > 0$ is a continuous probability distribution supported on $(0,1)$, and hence our parameter space $\Theta$ is the (strictly) positive orthant in $\smash{\R^{2|\calA|}}$.

Our approach is to directly estimate the parameters $\boldthetast = (\alphast,\betast)$ from the observed rewards in the $T$ episodes.
Since the family of Beta distributions is an exponential family~\cite{brown1986fundamentals} (with $(\alpha,\beta)$ being the natural parameters), we can appeal to general statistical theory to bound the total variation distance between two such distributions in terms of their parameter distance.

Suppose we adopt the exploration strategy where arm $1$ is selected in the first $n$ rounds in each of the first $T/|\calA|$ episodes, arm $2$ in the next $T/|\calA|$ episodes, and so on.
(We assume the horizon $H$ and $n$ satisfy $H \geq n \geq 2$.)
We focus on the estimation of $(\alphasti{1},\betasti{1})$, as the exact same approach works for all of the arms.
Let $X_t$ denote the {cumulative} reward collected in {the first $n$ rounds of} episode $t$.
Then, the random variables $X_1,\dotsc,X_{T/|\calA|}$ are i.i.d.~draws from a Beta-Binomial distribution with parameters $(\alphasti{1},\betasti{1},n)$, where $n$ denotes the number of trials of the binomial component.
The first and second moments of $X_t$ are
\begin{equation*}
  m_1^\star = \E[X_t] = \frac{n\alphasti1}{\alphasti1+\betasti1} \quad\text{and}\quad
  m_2^\star = \E[X_t^2] = \frac{n\alphasti1(n(1+\alphasti1)+\betasti1)}{(\alphasti1+\betasti1)(1+\alphasti1+\betasti1)} .
\end{equation*}
These moments uniquely determine $\alphasti1$ and $\betasti1$ as long as $n\geq2$.
Therefore, we can estimate $(\alphasti1,\betasti1)$ using plug-in estimates of the first two moments $(m_1^\star,m_2^\star)$ via the method of moments~\cite{tripathi1994estimation}.
{Using this approach, we obtain the following sample complexity guarantee for estimating the prior distribution:}

\begin{theorem}
  There is an exploration strategy $\algt{t}$ and an estimator $\hat{\bm\theta}(\cdot)$ with the following property.
  If $P_{\boldthetast} = \bigotimes_{a \in \calA} \Beta(\alphasti{a},\betasti{a})$ and $\calD(\boldmu) = \bigotimes_{a \in \calA} \Bernoulli(\mu_a)$, then there is a constant $C$ depending only on $(\alphast,\betast)$ such that, for any $\veps, \delta \in (0,1)$, if $H\geq2$ and
  \begin{equation*}
    T \geq \frac{C \cdot |\calA|^2 \log(|\calA|/\delta)}{\veps^2} ,
  \end{equation*}
  then $\Pr[\tvarg{P_{\boldthetast}}{P_{\hat{\bm\theta}}} \leq \veps] \geq 1-\delta$.
\end{theorem}
The proof of the theorem is given in Appendix~\ref{app:beta-binomial}.

%In Appendix~\ref{app:beta-binomial}, we prove that the above exploration strategy and method-of-moments estimator delivers an $\veps$-close estimate of $P_{\boldthetast}$ (in TV distance) with probability at least $1-\delta$ as long as $H\geq2$ and $T \geq C |\calA|^2 \log(|\calA|/\delta) / \veps^2$ for some $C>0$ that depends only on $\boldthetast$.

%The assumption of bounded Hessian is satisfied by classical distributions from the exponential family as shown below. \tlcomment{check}

%As a result, combining Lemma~\ref{lem:exponential_families} and Theorem~\ref{thm:n_montecarlo_body}, bounding the regret of any $n$-Monte Carlo algorithm by $\veps$ boils down to finding an estimator $\theta'$ whose Euclidean distance to the ground truth $\theta$ is $\nrm{\theta'-\theta}\leq \min\{c,\veps/C\}$. The number of samples required for that give the desired PAC bound. We now explain how such an estimator is derived for the two most classical exponential-family distributions: Beta distributions with Bernoulli rewards and Gaussian distributions with Gaussian rewards.

\subsection{Gaussian Priors and Gaussian Rewards}

We now consider the situation where the prior distribution is a multivariate Gaussian $P_{\boldthetast} = \Normal(\normmean, \normcov)$ in $\R^{\calA}$, and the reward distribution is a spherical Gaussian distribution $\calD(\boldmu) = \Normal(\boldmu, \sigma^2 \boldI)$.
Note that such a prior distribution is able to capture correlations between the arms' mean rewards in an episode, which cannot be captured by the product-form priors in the previous subsection (nor in previous work~\cite{kveton2021meta}).

We again directly estimate the parameters $\boldthetast = (\normmean,\normcov)$ using a simple exploration strategy and the method of moments.
In each episode (which we assume have horizon $H \geq 2$), we select independent and uniformly random actions in the first two rounds.
Let $a_t$ and $b_t$ denote the actions taken in episode $t$, and let $r_t$ and $s_t$ denote the corresponding observed rewards.
Our estimates for $\normmean$ and $\normcov$ based on the information collected in $T$ episodes are\footnote{This estimator can be generalized to explore for more of the episode and use more of the observed rewards.}
\begin{equation*}
  \normmeanest := \frac{|\calA|}{T} \sum_{t=1}^T r_t \bolde_{a_t}
  \quad\text{and}\quad
  \normcovest := \frac{|\calA|^2}{T} \sum_{t=1}^T r_t s_t \left( \bolde_{a_t} \bolde_{b_t}^\T + \bolde_{b_t} \bolde_{a_t}^\T \right) - \normmeanest \normmeanest^\T .
\end{equation*}
{For these estimators, we have the following theorem.}
\begin{theorem}
  There is an exploration strategy $\algt{t}$ and an estimator $\hat{\bm\theta}(\cdot)$ with the following property.
  If $P_{\boldthetast} = \Normal(\normmean, \normcov)$ and $\calD(\boldmu) = \Normal(\boldmu, \sigma^2 \boldI)$, then there is a constant $C$ depending only on $(\normmean,\normcov)$ and $\sigma^2$ such that, for any $\veps, \delta \in (0,1)$, if $H\geq2$ and
  \begin{equation*}
    T \geq \frac{C \cdot (|\calA|^4 + |\calA|^3 \log(1/\delta))}{\veps^2} ,
  \end{equation*}
  then $\Pr[\tvarg{P_{\boldthetast}}{P_{\hat{\bm\theta}}} \leq \veps] \geq 1-\delta$.
\end{theorem}
The proof of the theorem and the precise dependence on $\normmean$, $\normcov$, and $\sigma^2$ are given in Appendix~\ref{app:gaussian}.
%\akcomment{Add a remark about if the variance is known then we can get a better rate, to improve upon~\cite{kveton2021meta}.}
The quartic dependence on $|\calA|$ is due to estimating $\normcov$; it improves to $|\calA|^2$ if $\normcov$ is known.

\begin{remark}[Comparison to \cite{kveton2021meta}]\label{rem:comparison}%\akcomment{Move this to the Gaussian meta-learning section.}
\cite{kveton2021meta} study the case where $P_{\boldthetast} = \Normal(\normmean, \sigma_0^2 \boldI)$, which is a product-form prior over means $\boldmu$ with known $\sigma_0^2$. For  $\tilde{\epsilon} = |\calA|\cdot\|\normmean - \normmeanest\|_{\infty}/\sigma_0$, they show that\footnote{The following optimizes Lemma 5 of \cite{kveton2021meta} over its free parameter $\delta > 0$ for $\tilde{\epsilon}$ small.}
\begin{align*}
|R(\boldtheta,\TS(\boldtheta)) - R(\boldtheta,\TS(\hat{\boldtheta}))| \le \mathcal{O}\left(\|\normmean\|_{\infty} + \sigma_0\sqrt{\log (H / \tilde{\epsilon}) }\right)\cdot H^2\tilde{\epsilon}.
\end{align*}
On the other hand, \Cref{thm:n_montecarlo_body} applied to the $1$-Monte Carlo Thompson Sampling algorithm (and bounding $\range(\normmean) \le \|\normmean\|_{\infty}$) yields the same inequality, but with $\tilde{\epsilon}$ replaced by $\epsilon = \tvarg{\Normal(\normmean, \sigma_0^2 I)}{\Normal(\normmeanest, \sigma_0^2 I)} \leq \|\normmean - \normmeanest\|_2 / \sigma_0$.
% denoting the total variation between the distributions.
%By Pinsker's inequality (Lemma 2.5 in \cite{tsybakov2008introduction}) and the standard computation of the KL-divergence between Gaussians, we may bound $\epsilon \le \|\normmean - \normmeanest\|_2/\sigma_0$ by the Euclidean distance between the means.
Note that $\tilde{\epsilon}$ is always larger than $\epsilon$ by a factor of at least $\sqrt{|\calA|}$; thus, our result is strictly sharper.
%Moreover, our bound can be be as much as $|\calA|$-sharper when the  $\boldmu_0$ are
\end{remark}

%% file: body/experiments.tex
\section{Experiments}
We demonstrate the generality of our results in three distinct
meta-learning experimental settings. First, we study a simple
multi-armed bandit scenario with Gaussian prior and Gaussian rewards,
where we demonstrate how meta-learning higher-order moments of the
prior can significantly improve performance. Next, we consider a
Gaussian linear contextual bandits scenario, to demonstrate the
generality of Bayesian meta-learning. Finally, we study a more
interesting multi-arm\tledit{ed} bandit problem with discrete priors, where, in
addition to the value of meta-learning, we see that look-ahead
algorithms can substantially outperform Thompson sampling. Additional experimental details are presented in~\Cref{app:experiments}.

\begin{figure}
\includegraphics[width=0.5\textwidth]{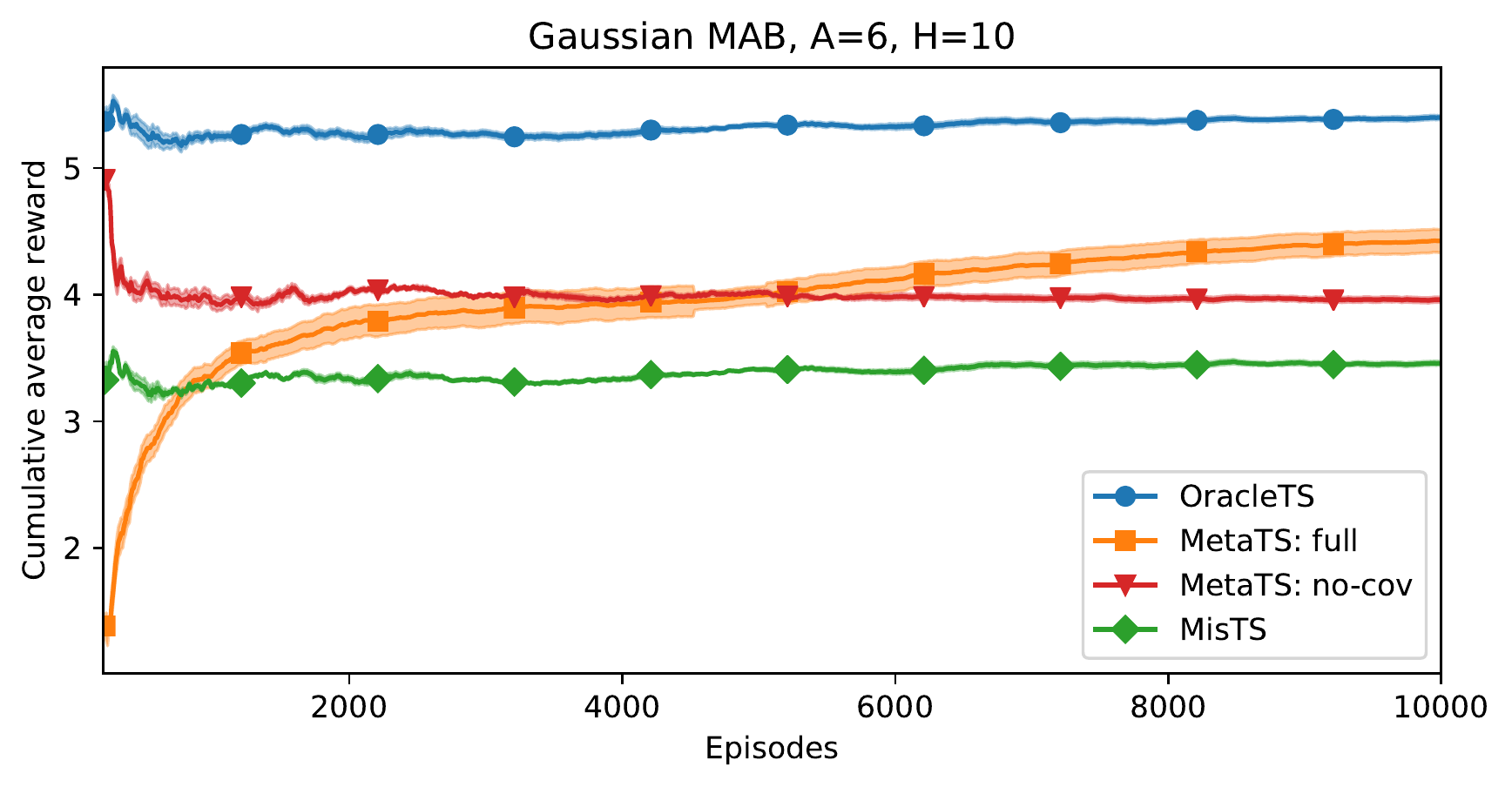}
\includegraphics[width=0.5\textwidth]{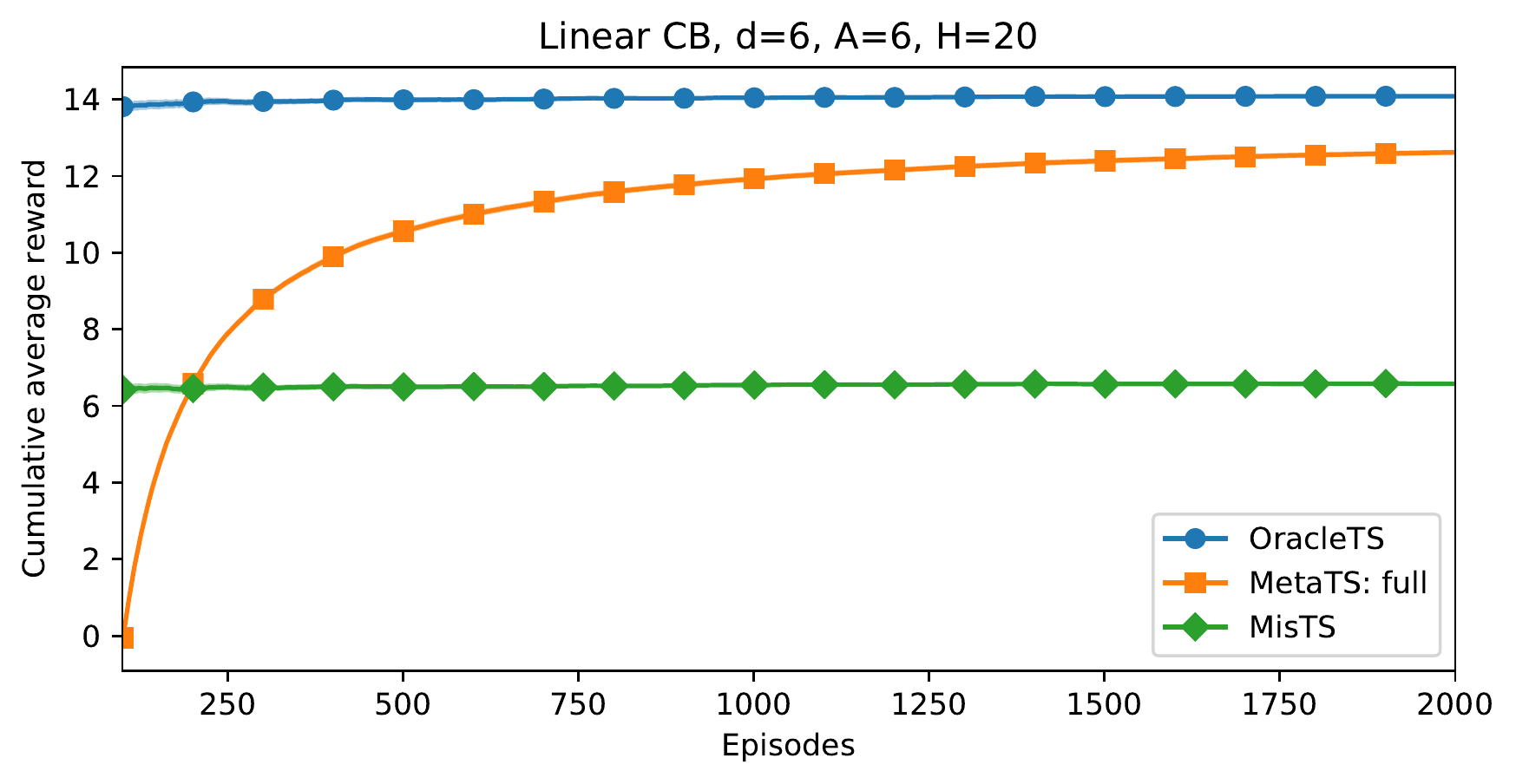}
\vspace{-0.5cm}
\caption{Learning curves for Gaussian MAB and linear CB
  experiments. We run 100 replicates per algorithm and visualize two
  standard errors with error bands. For meta-learners we tune the
  number of exploration rounds and display the performance of the best
  configuration at each point, which we call the upper envelope.}
  \vspace{-0.5cm}
\label{fig:gaussian_expts}
\end{figure}

\paragraph{Gaussian MAB.}
Our first scenario is a multi-arm\tledit{ed} bandit problem with Gaussian prior
and Gaussian reward. The instance has $|\calA|=6$ arms and each episode
has horizon $H=10$. The prior is $\mathcal{N}(\normmean, \normcov)$ where $\normmean
= [0.5,0,0,0.1,0,0]$ and $\normcov$ has block structure so that arms
$1,2,3$ are highly correlated, and analogously for arms $4,5,6$. The
rewards are Gaussian with variance $1$, which is known to all
learners.\looseness=-1

We run four algorithms. Two are non-meta-learning Thompson sampling algorithms: \algfont{OracleTS}, which uses the correct prior, and \algfont{MisTS}, which uses the misspecified prior $\mathcal{N}(\boldsymbol{0}, \mathbf{I})$.
%The two non-meta-learning algorithms are
%\algfont{OracleTS} and \algfont{MisTS} which both run Thompson
%sampling and use the correct prior and incorrect/misspecified prior $\mathcal{N}(\vec{0}, I)$ respectively. 
We also run \algfont{MetaTS:no-cov} which only attempts to
meta-learn the prior mean $\mu_0$ and assumes that the prior
covariance matrix is the identity (this algorithm is essentially the
one studied in~\cite{kveton2021meta}). Finally, our algorithm is \algfont{MetaTS:full}
which meta-learns both the prior mean and covariance. Both
meta-learners are run in an explore-then-commit fashion where the
first $T_0$ episodes are used for exploration.\footnote{For
  \algfont{MetaTS:no-cov}, we follow~\cite{kveton2021meta} and only use the first step
  of each exploration episode for exploration, switching to TS with
  the current prior estimate for the rest of the episode. On the other
  hand, \algfont{MetaTS:full} explores for all time steps in the first $T_0$
  episodes.}\looseness=-1

In Figure~\ref{fig:gaussian_expts}, we plot the cumulative average
per-episode reward for each algorithm, where for the meta-learners we
sweep over many choices of $T_0$ and display the pointwise best (i.e.,
the upper envelope). The experiment clearly shows the value of
meta-learning as both \algfont{MetaTS:no-cov} and \algfont{MetaTS:full} quickly
outperform misspecified TS. Additionally, we also see the importance
of learning the covariance matrix, even though it can require many
samples. Indeed, the final performance of \algfont{MetaTS:full} with $T_0=5$K,
ignoring the regret incurred due to exploration, is competitive with
\algfont{OracleTS}, while \algfont{MetaTS:no-cov} asymptotes to a much lower
performance (see~\Cref{fig:test_errors} in~\Cref{app:experiments}).

\paragraph{Gaussian linear contextual bandits.}
Our second experiment concerns Gaussian linear contextual
bandits. Here we run \algfont{OracleTS}, \algfont{MisTS}, and
\algfont{MetaTS:full}, on a synthetic linear contextual bandit problem where
there are $|\calA|=6$ actions each with a $d=6$ dimensional action feature
(generated stochastically at each time step), and with horizon
$H=20$. The prior is over the linear parameter $\boldmu$ that
determines the reward for action-feature $x_a \in \mathbb{R}^d$ as
$r(a) \sim \mathcal{N}(\langle \boldmu, \boldx_a\rangle, 1)$. We set the
prior as $\mathcal{N}(\mathbf{1}, \normcov)$ where $\normcov$ is a
scaled-down version of the block diagonal matrix used in the previous
experiment. In the right panel of~\Cref{fig:gaussian_expts} we
again see that by meta-learning the prior, we quickly outperform the
misspecified approach and asymptotically achieve the oracle
performance. This demonstrates that Bayesian meta-learning is quite
broadly applicable and highlights the importance of our general
theoretical development.

\begin{figure}
\begin{minipage}{0.49\textwidth}
\includegraphics[width=\textwidth]{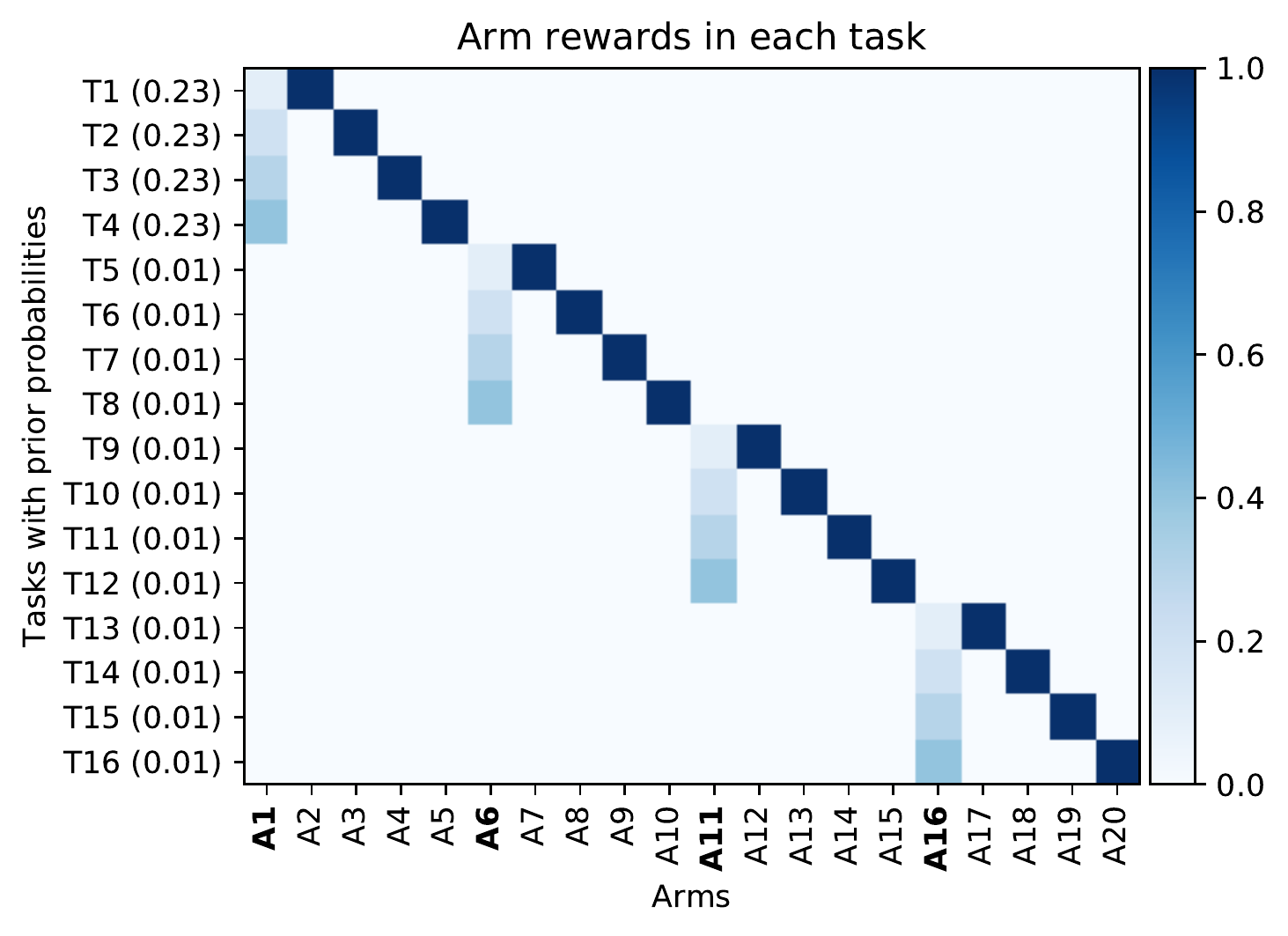}
\end{minipage}
\begin{minipage}{0.49\textwidth}
\vspace{-0.5cm}
\includegraphics[width=\textwidth]{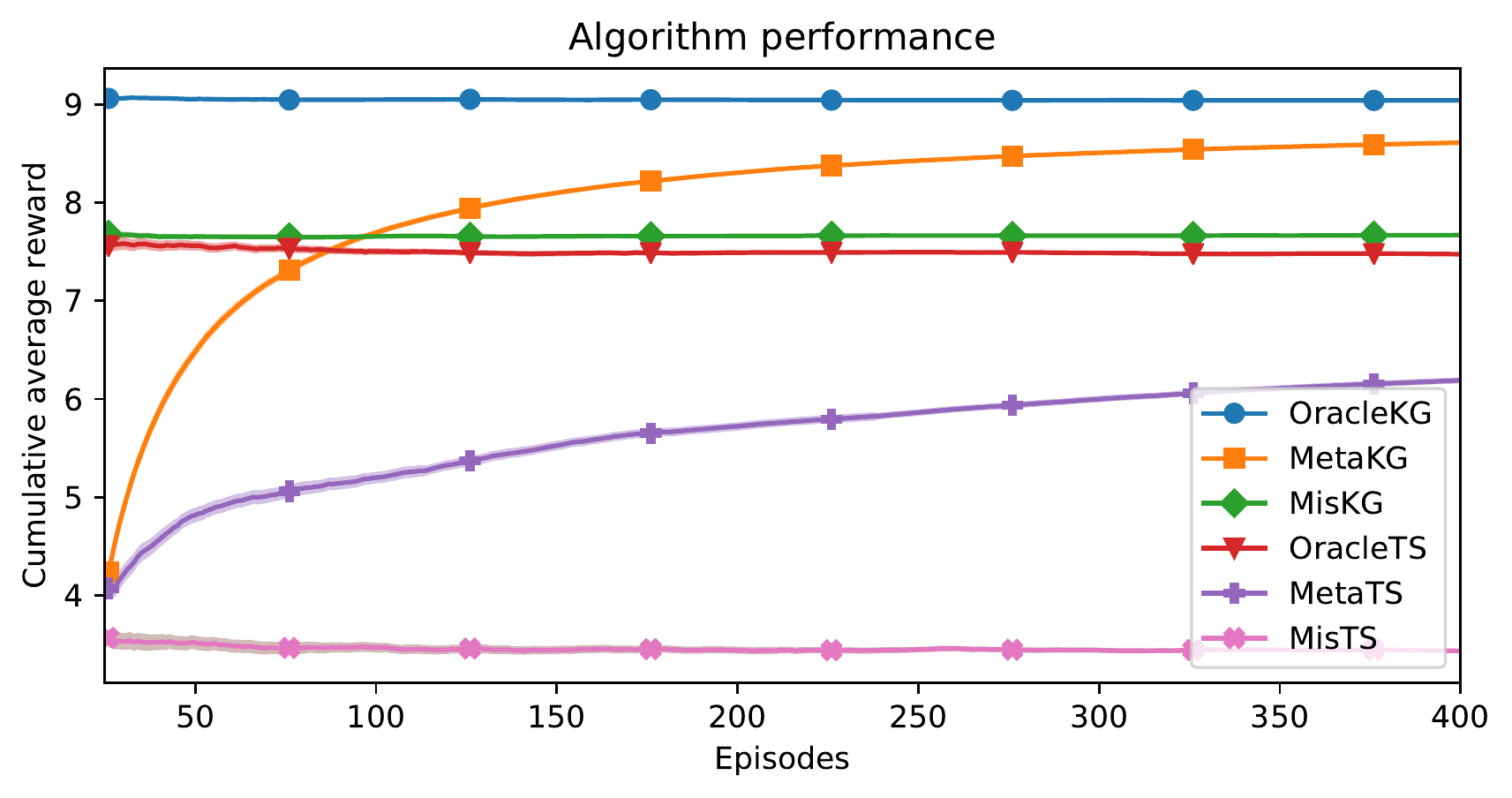}
\includegraphics[width=\textwidth]{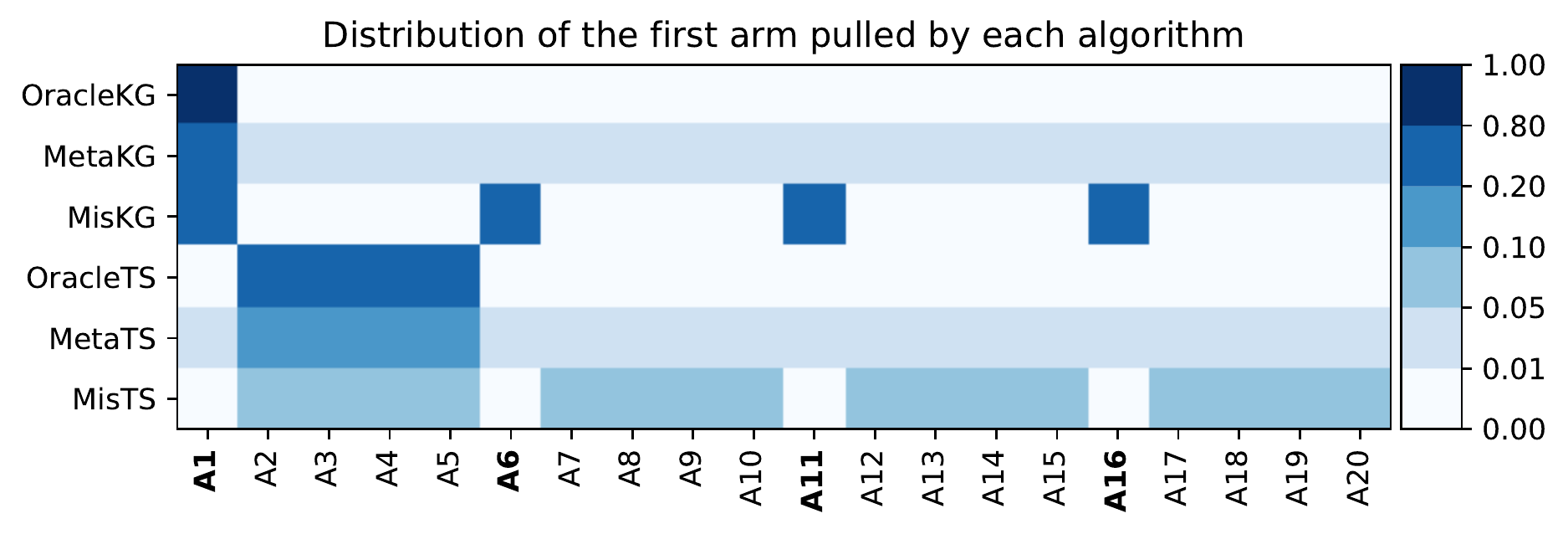}
\end{minipage}
\caption{Synthetic experiments with discrete MAB for $|\calA|=20$ and $H=10$. \emph{Left}: visualization
  of the instance showing the reward for each of the arms in
  each of the $16$ possible tasks along with the prior distribution
  over tasks (probabilities rounded, actual values are $9/40$ and $1/120$). \mbox{\emph{Top right}:} learning curves for 6 algorithms (100
  replicates, error bands at 2 standard errors, we tune the number of
  exploration rounds and plot upper envelopes for
  meta-learners). \mbox{\emph{Bottom right}:} empirical distribution of the first
  arm pulled in each episode by each algorithm. Note that the color scale is non-linear.}
\label{fig:kg_expts}
\end{figure}

\paragraph{Discrete bandits.}
Finally, we study a synthetic MAB setting with $|\calA|=20$ arms and a prior
supported on a finite set of $16$ reward distributions (tasks), under each of which
rewards are deterministic. The
instance is visualized in the left panel of
Figure~\ref{fig:kg_expts}. It is constructed so that each task has a
unique optimal arm and there are four arms that can quickly identify
which task the agent is in (arms A1, A6, A11, A16), so that it can infer the optimal
arm. Additionally, the prior is concentrated on the first four tasks,
so that pulling the first identifying arm almost always reveals the
current task. 

We evaluate 6 algorithms: Oracle, Misspecified, and Meta-learning each
with TS and Monte-Carlo Knowledge Gradient (an instantiation of the $\twompc(\theta)$ algorithm detailed in \iftoggle{neurips}{\Cref{app:mc_algs}}{\Cref{alg:two_mpc}}) as the base learners, and
we visualize the results in the top right panel of
Figure~\ref{fig:kg_expts}. Perhaps more revealing is the bottom right
panel of Figure~\ref{fig:kg_expts}, where we visualize the empirical
distribution over the first arm pull in each episode for each
algorithm. We see that \algfont{OracleTS} typically plays uniformly
over arms A2--A5 in the first round as these are highly likely to be
the optimal arm under the prior, while \algfont{MisTS} plays
uniformly over the 16 plausibly optimal arms. \algfont{MetaTS} quickly
learns to play uniformly over arms A2--A5 and is asymptotically
competitive with \algfont{OracleTS}. 

{The interesting property of this instance is that} playing the identifying arms is crucial for optimal behavior. However, since TS is myopic and these arms never produce large rewards, TS
will never play them. Thus, to achieve optimal behavior, we must use a less myopic base learner like Knowledge Gradient.
% However, playing the identifying arms is crucial for optimal behavior,
% but since TS is myopic and these arms never produce large rewards, TS
% will never play them. To achieve optimal behavior, we must use a less
% myopic base learner like Knowledge Gradient. 
As can be seen, both
\algfont{OracleKG} and \algfont{MisKG} first play the
identifying arms, where the oracle almost always pulls the first one
while \algfont{MisKG} plays them uniformly. The performance of
\algfont{OracleKG} is much better than all TS configurations. Finally,
the meta-learning configuration of Knowledge Gradient quickly learns
to pull the first identifying arm and competes with \algfont{OracleKG}.

%% file: body/discussion.tex
\section{Discussion}

In our simulations, we demonstrated the superiority of more expressive prior families (e.g., modeling means and covariances) and non-myopic base algorithms (e.g., Knowledge Gradient) over less expressive priors (e.g., product measures) and greedy base learners (e.g., Thompson sampling). 
Notably, the generality and flexibility of our theoretical contributions ensure robustness to prior misspecification even for these richer priors and sophisticated base learners. 

Still, theory and experiments alike point to a tradeoff: despite the potential for improved performance, richer prior families are harder to learn, and some base learners (e.g., $n$-Monte Carlo algorithms for large $n$) %$n$-large 
can be more sensitive to incorrect priors.  
It is an exciting direction for future work to investigate the joint problems of \emph{model selection} (over priors) and \emph{algorithm selection} (over base learners) in order to optimally navigate these tradeoffs. Perhaps model and algorithm selection can be coupled so that certain base learners exhibit improved performance, or greater robustness, over certain classes of priors. 
We would like to further understand how these tradeoffs interface with computational burdens of using certain priors and base learners, and whether our sensitivity analysis extends to computationally efficient approximations of sampling-based decision-making algorithms (e.g., via Laplace approximations, MCMC, Gibbs Sampling, and Variational Methods; the long-horizon performance of Thompson sampling under approximate inference has already been studied \cite{phan2019thompson}).
% More broadly, it is of interest to characterize what is the broadest family of Bayesian base learners which is robust to prior misspecification. 
Finally, we hold hope that a more instance-dependent analysis may improve our sensitivity bounds for certain families of priors, which may in turn inform more clever exploration strategies that circumvent worst-case tradeoffs. 

\subsection*{Acknowledgements}
The authors thank Wen Sun for many discussions that helped shape the current paper.

%Both theory and experiments point a number of salen

%\mscomment{regarding not all things stable for Bayes algorithms. Non Monte-Carlo Algorithms}
%\mscomment{Extensions to Laplace + Gibbs + Variational Methods + Approximate}

%\dhcomment{play up differences relative to csaba paper: non-myopic base learner, more flexible priors (esp as revealed in experiments); potential trade-offs between sensitivity and complicated non-myopic base learners, maybe slower learning due to curse of dimension, computational complexity of planning}

%% file: appendix/app_experiments.tex
\section{Additional Experimental Details}
\label{app:experiments}
\begin{figure}
\includegraphics[width=0.5\textwidth]{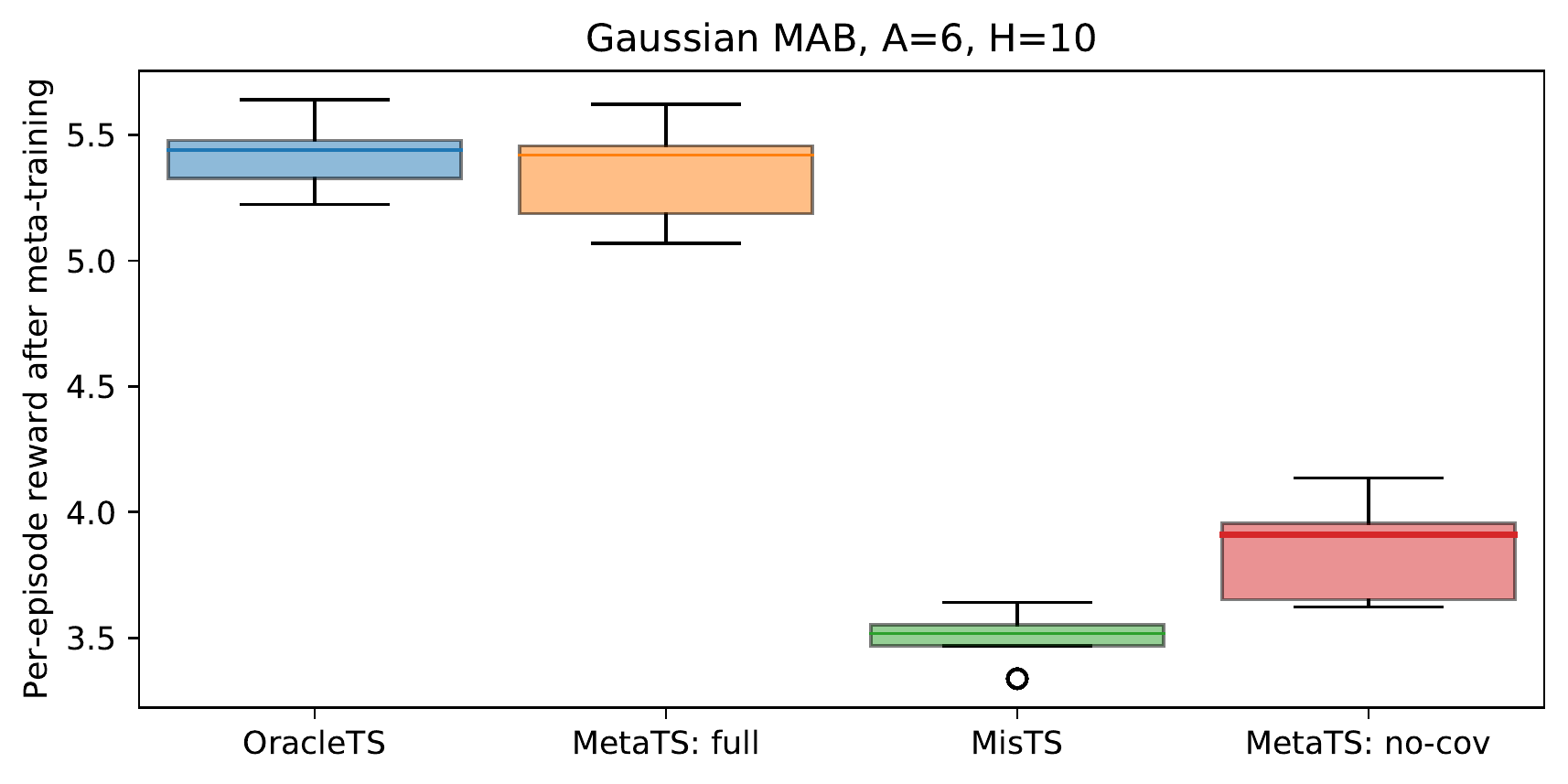}
\includegraphics[width=0.5\textwidth]{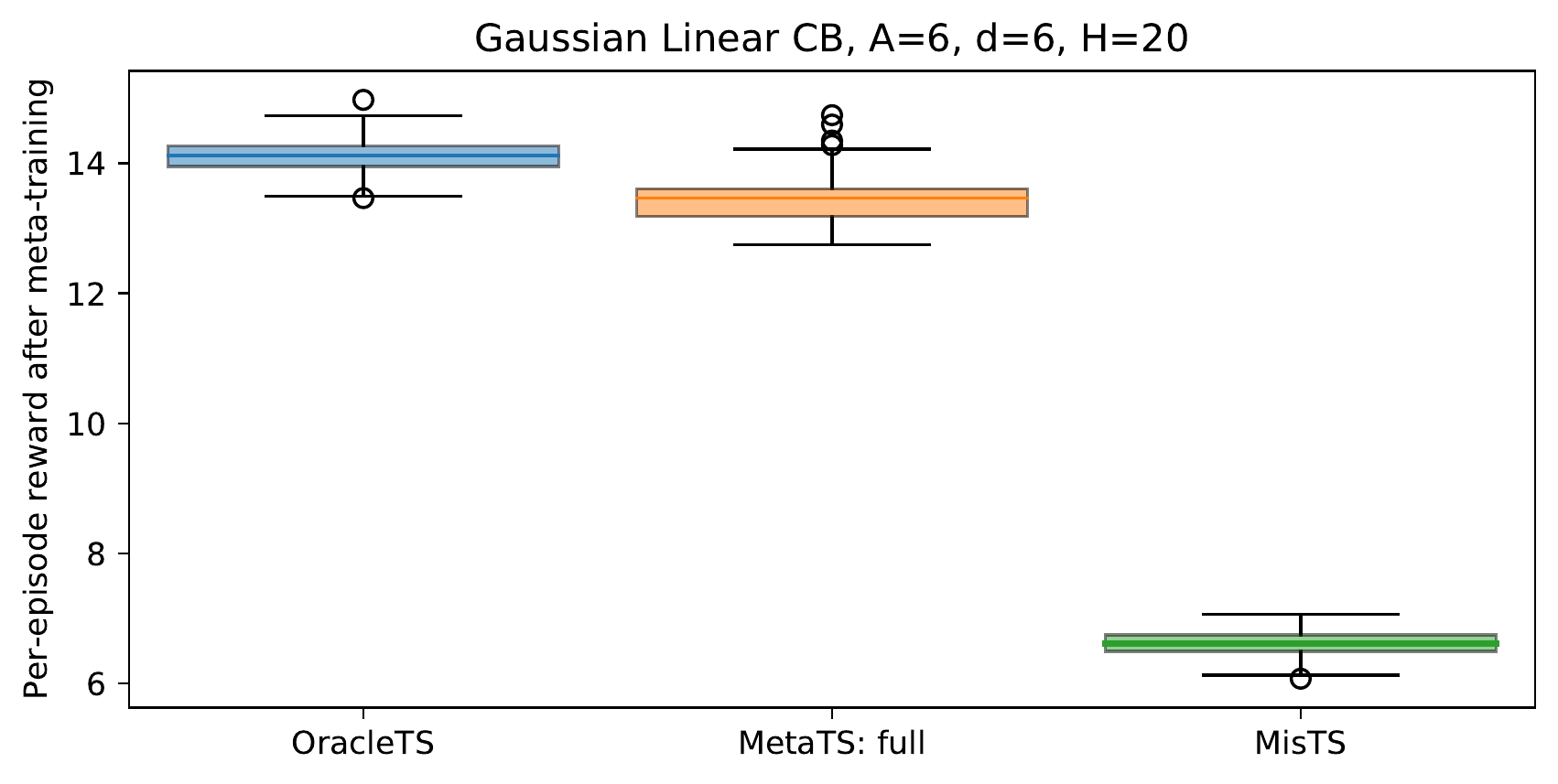}
\caption{Test performance in Gaussian MAB and Gaussian linear CB
  experiments.}
\label{fig:test_errors}
\end{figure}

In this section, we provide additional experimental details for each
setting. As a prelude, the total amount of compute is very minimal and
primarily inflated by the large number of replicates used in each
experiment. On a standard CPU cluster the experiments can easily be
completed in 2-4 hours, even with running 100 replicates for each
algorithm/configuration.

\subsection{Multi-armed bandit experiments}
As described, the left panel of~\Cref{fig:gaussian_expts} is based on
a $|\mathcal{A}|=6$ arm bandit problem with horizon $H=10$ and prior
$\mathcal{N}(\normmean,\normcov)$, where
\begin{align*}
\normmean = [0.5,0,0,0.1,0,0], \quad \textrm{and} \quad 
\normcov = 
\left(\begin{matrix}
1 & 0.9 & 0.9 & 0 & 0 & 0\\
0.9 & 1 & 0.9 & 0 & 0 & 0\\
0.9 & 0.9 & 1 & 0 & 0 & 0\\
0 & 0 & 0 & 1 & 0.9 & 0.9\\
0 & 0 & 0 & 0.9 & 1 & 0.9\\
0 & 0 & 0 & 0.9 & 0.9 & 1
\end{matrix}\right).
\end{align*}
The rewards are Gaussian, with variance $1.0$. 

The four algorithms we run are:
\begin{itemize}
\item \algfont{OracleTS}: The standard implementation of Gaussian Thompson sampling, with the correct prior $(\normmean,\normcov)$.
\item \algfont{MisTS}: The standard implementation of Gaussian Thompson sampling, with the incorrect prior $(\boldsymbol{0},\mathbf{I})$.
\item \algfont{MetaTS:full}: A meta-learning implementation of
  Gaussian Thompson sampling with an ``explore-then-commit''
  strategy. This algorithm has a hyperparameter $T_0$ which determines
  the number of exploration rounds. In the first $T_0$ rounds, the
  algorithm simply selects all actions uniformly at random. Then at
  the end of the $T_0$ exploration rounds, it forms an estimate
  $(\normmeanest,\normcovest)$ as follows:
\begin{align*}
\normmeanest &= \frac{|\mathcal{A}|}{T_0}\sum_{i=1}^{T_0} \widehat{\bm{\mu}}_i = \frac{|\mathcal{A}|}{T_0}\sum_{i=1}^{T_0}\left(\frac{1}{H}\sum_{h=1}^H \sum_{a \in \mathcal{A}} \one\{a_{i,h} = a\}e_a r_{i,h}\right),\\
\normcovest &= \frac{1}{T_0}\sum_{i=1}^{T_0} \left( \widehat{\bm{\mu}}_i \widehat{\bm{\mu}}_i^{\top} -\textrm{diag}( \widehat{\bm{\mu}}_i \widehat{\bm{\mu}}_i^{\top}) + \textrm{diag}\left( \frac{|\mathcal{A}|}{H} \sum_{h=1}^H \sum_a \one\{a_{i,h} = a\}e_a r^2_{i,h}\right) - \mathbf{I}\right) - \normmeanest\normmeanest^\top.
\end{align*}
Here $a_{i,h}$ is the action played at the $h^{\textrm{th}}$ time step
of the $i^{\textrm{th}}$ episode and $r_{i,h}$ is the corresponding
reward. It is not difficult to verify that both of these are unbiased
estimators for $\normmean$ and $\normcov$ respectively. We
additionally project $\normcovest$ onto the positive semidefinite
cone.  After the $T_0$ exploration rounds, \algfont{MetaTS:full} forms
the above estimators and runs standard Gaussian Thompson sampling with
the estimates $(\normmeanest,\normcovest)$.
\item \algfont{MetaTS:no-cov}: A meta-learning implementation of
  Gaussian Thompson sampling with an ``explore-then-commit'' strategy,
  which does not estimate the prior covariance. As above, it has a
  hyperparameter $T_0$ determining the number of exploration
  rounds. In the first $T_0$ rounds, the algorithm chooses just the
  first action uniformly at random and then chooses the remaining
  actions by instantiating Gaussian Thompson sampling with the current
  estimate of the prior mean and the incorrect prior covariance
  $\mathbf{I}$. The prior mean at round $t$ is estimated as
\begin{align*}
\normmeanest_t &= \frac{|\mathcal{A}|}{t-1}\sum_{i=1}^{t-1}\sum_{a \in \mathcal{A}} \one\{a_{i,1} = a\}e_a r_{i,1},
\end{align*}
which is analogous to the estimate above. After $T_0$ rounds, we set
$\normmeanest = \normmeanest_{T_0}$ and we run standard Gaussian
Thompson sampling with prior $(\normmeanest, \mathbf{I})$ for the
remaining rounds.
\end{itemize}

\paragraph{Experimental Protocol and Results.}
In the left panel of~\Cref{fig:gaussian_expts} we run each algorithm
(with each hyperparameter configuration) for 100 replicates with
different random seeds. For both \algfont{MetaTS} variants, we choose
$T_0$ from the set $\{200,400,600,\ldots,5000\}$. In the figure, we
record the average (across replicates) performance at each episode
number with error bands corresponding to $\pm 2$ standard errors.

For the algorithms with a hyperparameter, we plot the performance of
the pointwise best hyperparameter configuration. That is, we optimize
hyperparameters (based on average-across-replicates performance) for
each episode number $n$ individually.

In the left panel of~\Cref{fig:test_errors} we visualize the ``test
performance'' of the various algorithms, which corresponds to the
average per-episode performance for the last $5,000$ episodes. Here
the box plots visualize the 100 different replicates. For both
\algfont{MetaTS} variants, we use $T_0=5,000$ as the
hyperparameter. Note that since the total number of episodes is
$10,000$, both algorithms do not update their prior estimate for the episodes during
which we record performance.

\subsection{Linear contextual bandit experiments}
The experimental protocol is similar to the one above. Here we
consider a Gaussian linear contextual bandit setup with $|\calA|=6$
actions and $d=6$ dimensional action features and horizon $H=20$. The
prior is $\mathcal{N}(\normmean,\normcov)$ where
\begin{align*}
\normmean = \mathbf{1}, \quad \textrm{and} \quad \normcov = 0.1\times \left(\begin{matrix}
1 & 0.9 & 0.9 & 0 & 0 & 0\\
0.9 & 1 & 0.9 & 0 & 0 & 0\\
0.9 & 0.9 & 1 & 0 & 0 & 0\\
0 & 0 & 0 & 1 & 0.9 & 0.9\\
0 & 0 & 0 & 0.9 & 1 & 0.9\\
0 & 0 & 0 & 0.9 & 0.9 & 1
\end{matrix}\right).
\end{align*}
In each round the action features are generated by sampling each entry
from a standard normal distribution and then normalizing so that the
feature vector has $\ell_2$ norm equal to $1$. For action feature
$\boldx_a$ the reward is given by $r(a) \sim \mathcal{N}(\langle
\boldmu, \boldx_a\rangle, 1)$.

We run three algorithms here. The first two \algfont{OracleTS} and
\algfont{MisTS} are standard implementations of Gaussian linear
Thompson sampling with well-specified and mis-specified priors
respectively. Here \algfont{MisTS} is initialized with prior
$\mathcal{N}(\mathbf{0}, \mathbf{I})$. The final algorithm,
\algfont{MetaTS:full} is implemented in the explore-then-commit
fashion described above. The only difference is the estimator for the
prior. Here in each episode of the exploration stage, we choose
actions uniformly at random and use ordinary least squares to estimate
the parameter $\boldmu$ of the episode. The prior mean is simply
estimated using the average of these OLS solutions. The prior covariance is estimated as
\begin{align*}
\normcovest = \left(\frac{1}{T_0} \sum_{i=1}^{T_0} \widehat{\bm{\mu}}_i\widehat{\bm{\mu}}_i^\top - \Sigma_i^{-1}\right) - \normmeanest\normmeanest^\top,
\end{align*}
where $\Sigma_i = \sum_{h=1}^H \boldx_{h,a_h}\boldx_{h,a_h}^\top$ is the second
moment matrix of the action features chosen in the episode. As above,
this is an unbiased estimator of the prior covariance.

\paragraph{Experimental Protocol and Results.}
We follow the same protocol as above, running each algorithm for 100
replicates and, for \algfont{MetaTS:full}, we plot the pointwise best
performance across hyperparameter configurations. Here we tune $T_0
\in \{100,200,\ldots,1000\}$. In the right panel
of~\Cref{fig:test_errors} we plot the test performance of each
algorithm, measured as the average performance in the final $1,000$
episodes. We use $T_0 = 1000$ for \algfont{MetaTS:full}.

\subsection{Discrete bandits}
The final experiment is with the discrete MAB instance visualized
in~\Cref{fig:kg_expts}. As the instance is visualized in the left
panel, we only describe the algorithms and the experimental protocol.
As the reward distributions are singular, posteriors collapse
frequently in this experiment. Once this happens, all algorithms
simply play the best arm from then on.

Thompson sampling as a base learner is standard. We maintain a
posterior distribution over tasks, sample an instance/task from this
distribution, and play the best arm for that task. Posterior updates
are straightforward due to the singular nature of the reward
distributions.

For Knowledge Gradient, we implement a one-step look-ahead variant,
which is exactly as described in~\Cref{alg:two_mpc}, with $k_1=k_2=10$
and $\alpha = 1$. We also implement a random tie breaking scheme where
we choose randomly among actions with the maximum $V_a$.

We implement the meta-learners in a straightforward
explore-then-commit manner. In each exploration round, we choose
actions uniformly at random. If the posterior collapses, then we
increment a counter associated with the current task. If the posterior
does not collapse during the episode then we do not increment any
counter. After $T_0$ exploration rounds we estimate the posterior by
the empirical fraction of times we observed each task.

As above, we run 100 replicates of each algorithm. Misspecified
variants are initialized with the uniform prior over tasks. For the
meta-learners we tune $T_0 \in \{25,50,\ldots,200\}$. We plot the
pointwise best (across hyperparameters) mean performance across
replicates, with bands corresponding to $\pm 2$ standard errors. Note
that there is very little variance here since we run many replicates
and the problem has little noise. 

In the bottom panel of~\Cref{fig:kg_expts} we plot the empirical
distribution of the first action chose by each algorithm, where we
compute this distribution using all $400$ episodes and all $100$
replicates of each algorithm. For both meta-learners we use $T_0=100$
here.

%% file: appendix/app_smoothness.tex
%!TEX root = ../arxiv_main.tex
\section{Proof of Sensitivity Bounds}\label{app:sensitivity}
\newcommand{\calF}{\mathcal{F}}
\newcommand{\rmd}{\mathrm{d}}
\newcommand{\scrF}{\mathscr{F}}
\newcommand{\boldxi}{\boldsymbol{\xi}}
\newcommand{\Dseed}{\mathcal{D}_{\mathrm{seed}}}
In this appendix, we give the proofs of \Cref{thm:n_montecarlo_body} and \Cref{prop:prior_sensitivity_tv_bound}. The results in this appendix are much more general than those stated in \Cref{sec:prior_sensitivity} and require us to introduce some new concepts. The following roadmap may be useful in navigating the rest of this appendix.
\begin{itemize}
    \item \Cref{sec:key_prop_tv} provides some key properties of total variation distance that are used in the rest of \Cref{app:sensitivity}, as well as in \Cref{app:mc_algs}.
    \item \Cref{sec:general_sensitivity} provides the statements of the main results of this section. In particular, we define our notion of upper tail expectation, we introduce our tail conditions, and we provide the statements of our generalizations of \Cref{thm:n_montecarlo_body} (\Cref{thm:n_montecarlo_general_tail} and \Cref{thm:n_montecarlo_tails}).
    \item \Cref{sec:tail_exp} and \Cref{sec:estimates_upper_bound} provide key properties of our upper tail expectation and bound the upper tail expectation under our tail conditions.  
\iftoggle{neurips}{
    \item \Cref{app:prop_tv_sens} and \Cref{sec:proof:lem:perf_diff_bandits} together give the proof of \Cref{prop:prior_sensitivity_tv_bound}.
}
{
    \item \Cref{sec:proof:lem:perf_diff_bandits} gives the proof of \Cref{lem:perf_diff_bandits}, which was used to prove~\Cref{prop:prior_sensitivity_tv_bound}. 
}
    \item \Cref{app:proof_n_montecarlo_general_tail} finishes the proof of \Cref{thm:n_montecarlo_general_tail}.
\end{itemize}
\iftoggle{neurips}{
    \Cref{fig:dependencies} illustrates relationships between the subsections in this appendix.
    \begin{figure}
    \begin{center}
    \includegraphics[width=0.9\textwidth]{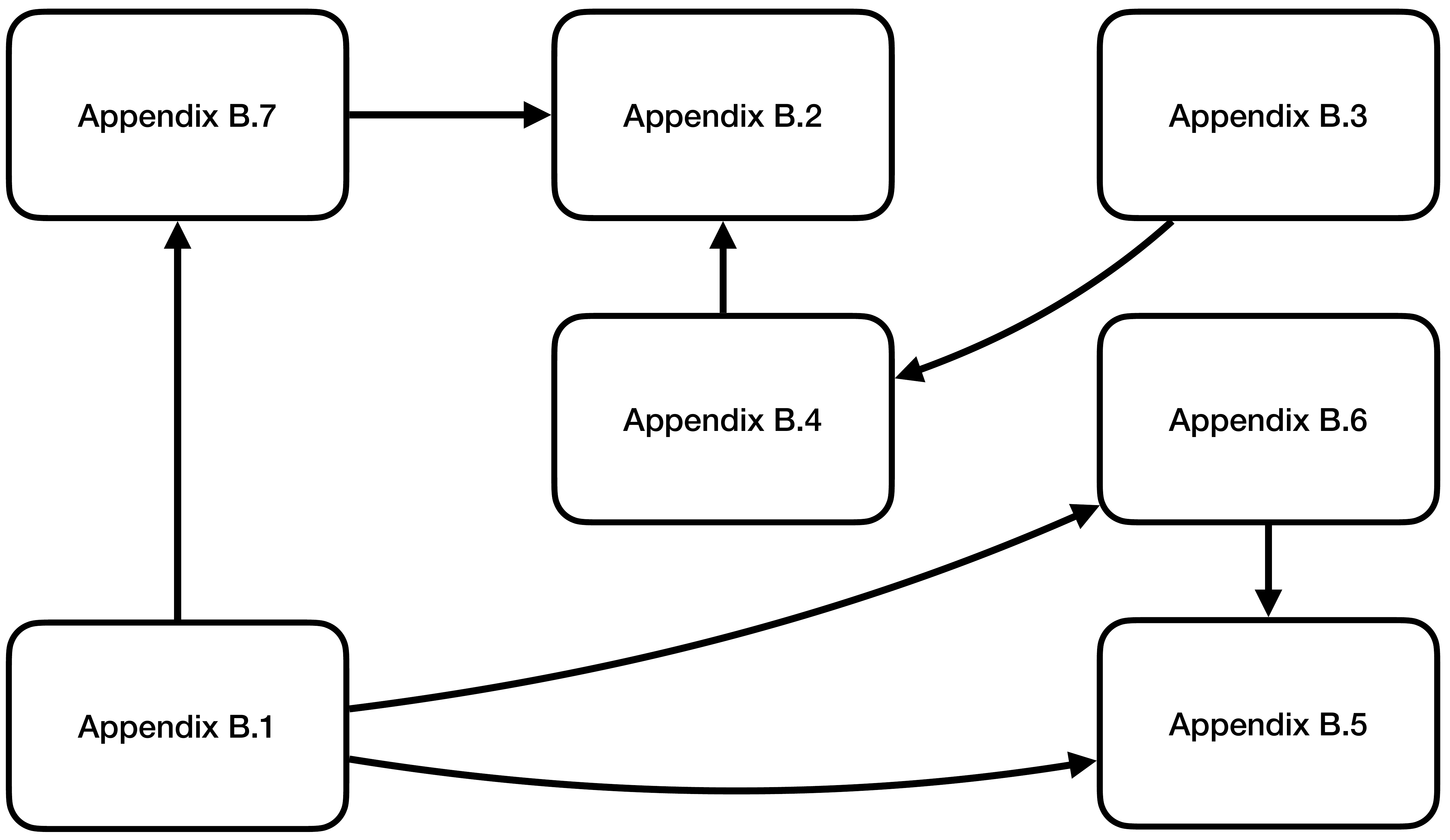}
    \end{center}
    \caption{Dependency graph of \Cref{app:sensitivity}. An arrow indicates that a result or proof is provided in the section that the arrow is coming out of and is used in the section where the arrow is entering.}
    \label{fig:dependencies}
    \end{figure}
}
{
    % \begin{figure}
    % \begin{center}
    % \includegraphics[width=0.9\textwidth]{./appendix/figures/dependencies-arxiv.pdf}
    % \end{center}
    % \caption{Dependency graph of \Cref{app:sensitivity}. An arrow indicates that a result or proof is provided in the section that the arrow is coming out of and is used in the section where the arrow is entering.}
    % \label{fig:dependencies}
    % \end{figure}
}

\subsection{Key Properties of the Total Variation Distance \label{sec:key_prop_tv}}
\paragraph{Technical disclaimer.} In what follows, we will need that our probability space $(\Omega, \scrF)$ allows for the equivalence between total variation distance and couplings. One way that this can be guaranteed is if (a) our space $\Omega$ is Polish, i.e., that $\Omega$ is metrizable by a metric that makes it complete and separable and (b) our $\sigma$-algebra $\scrF$ is the Borel algebra \msedit{$\mathscr{B}(\Omega)$}, i.e., the $\sigma$-algebra generated by open sets in $\Omega$~\cite{lindvall2002lectures}. Furthermore, we assume all random variables $X:(\Omega,\scrF) \to \calX$ take values in a Polish space $\calX$. We endow $\calX$ with the Borel $\sigma$-algebra $\mathscr{B}(\calX)$, and assume that $X$ is measurable from $(\Omega,\scrF) \to (\calX,\mathscr{B}(\calX))$; that is, $X^{-1}(\mathcal{E}) \in \scrF$ for all $\mathcal{E} \in \mathscr{B}(\calX)$. We let $\Delta(\calX)$ denote the set of all Borel-measurable distributions on $\calX$. 

\paragraph{Randomized Algorithms.} Throughout, we often refer to \emph{randomized algorithms}. Formally, a family of randomized bandit algorithms $\alg(\theta)$ is a specified by a distribution $\mathcal{D}_{\mathrm{seed}}$ (independent of $\theta$), a domain $\Xi$ over random seeds $\boldxi$, and step-wise mappings $f_{1},\dots,f_H$ from trajectories, the random seed, and parameters $\theta$ to distributions over actions:
\begin{align*}
f_h(\tau_{h-1},\boldxi \mid \theta): \{h\text{-trajectories}\} \times \Xi \times \Theta \to \Delta(\calA).
\end{align*}
Each $\alg(\theta)$ operates as follows:
\begin{itemize}
    \item $\boldxi$ is drawn from $\Dseed$ at the start of the episode before interaction.
    \item At each step $h$, $a_h$ is chosen as $a_h \sim f_h(\tau_h,\boldxi \mid \theta)$, independently of the past
\end{itemize}
\begin{remark}[Sources of Randomness]\label{rem:randomness}
Note that we allow for \emph{two} sources of randomness: the draw of $a_h$ from the distribution $f_h(\tau_{h-1},\boldxi \mid \theta)$, and the initial random seed $\boldxi$ at the start of the episode. For many natural algorithms - such as those \Cref{app:mc_algs} - we do not need $\boldxi$, and can just represent the randomness via actions selected independently for trajectory-dependent distributions.  However, in some case, it may be desirable for there to be a random seed $\boldxi$ encoding randomness shared across stages. Moreover, the assumption that $\Dseed$ does not depend on $\theta$ is very mild, and can be satisfied by all families $\alg(\cdot)$ which can be run on a single random number generator independent of $\theta$.
\end{remark}

\paragraph{Total Variation and its Key Properties.} Recall the definition of the total variation distance.
\begin{definition}\label{defn:total_variation} Let $P,P'$ be two probability measures on a space $(\Omega,\scrF)$. Then $\tvarg{P}{P'} = \sup_{\calE}|P(\calE) - P'(\calE)|$ is the maximal difference in probabilities of measurable events $\calE \in \scrF$.
\end{definition}
In our proofs, we make use of the following elementary properties of the total variation distance.
\begin{lemma}[Total Variation Properties]\label{lem:key_tv_props} Let $P,P',P''$ be any three probability measures of the same probability space $(\Omega,\scrF)$.
\begin{itemize}
\item[\emph{(a)}] \emph{Coupling Form:} Let $Q$ be a coupling of $P$ and $P'$, i.e. a joint distribution over $(X, Y)$ such that its marginal distribution over $X$ is $P$ and its marginal distribution over $Y$ is $P'$. Then for any such coupling $Q$, we have
\[ \tvarg{P}{P'} \leq Q(X \neq Y). \]
Moreover, there exists a maximal coupling $Q$ such that \[ \tvarg{P}{P'} = Q(X \neq Y). \]
\item[\emph{(b)}] \emph{Variational Forms:} $\tvarg{P}{P'} = \sup_{\calE} P(\calE) - P'(\calE)$ (that is, without the absolute value). Moreover, if $E,E'$ denote the associated expectations,  and letting $V$ quantify $[0,1]$-bounded random variables on $(\Omega,\scrF)$,
\begin{align*}
\tvarg{P}{P'} = \sup\{E[V] - E'[V] \text{ s.t. } V:(\Omega,\scrF) \to [0,1]\}
\end{align*}
\item[\emph{(c)}] \emph{Symmetry:} $\tvarg{P}{P'} = \tvarg{P'}{P}$.
\item[\emph{(d)}] \emph{Triangle Inequality:} 
\begin{align*}
\tvarg{P}{P''} \le \tvarg{P}{P'} + \tvarg{P'}{P''}.
\end{align*}
\item[\emph{(e)}] \emph{Data Processing:} Let $(X,Y)$ be random variables on $(\Omega,\scrF)$. Then 
\begin{align*}
\tvarg{P(X)}{P'(X)} \leq \tvarg{P(X,Y)}{P'(X,Y)}.
\end{align*}
\item[\emph{(f)}] \emph{Tensorization:} Let $(X_1,\dots,X_n)$ be $n$ random variables on $(\Omega,\scrF)$ which are independent under both $P$ and $P'$. Then, 
\begin{align*}
\tvarg{P(X_1,\dots,X_n)}{P'(X_1,\dots,X_n)}  \le \sum_{i=1}^n \tvarg{P(X_i)}{P'(X_i)}.
\end{align*}
\end{itemize}
\end{lemma}
\begin{proof}
The coupling and variational forms can be found in \cite[Chapter 1]{lindvall2002lectures}. Symmetry follows immediately from the definition. % of total variation distance.

To see the triangle inequality, note that for any measurable $\calE \subset \Omega$, we have
\[ |P(\calE) - P''(\calE)| \le |P(\calE) - P'(\calE)| + |P'(\calE) - P''(\calE)| \le \tvarg{P}{P'} + \tvarg{P'}{P''}. \]
As the above holds for any such $\calE$, we can conclude
\[ \tvarg{P}{P''} \le \tvarg{P}{P'} + \tvarg{P'}{P''}. \]
For the data processing inequality, say $(X,Y)$ follow distribution $P$ and $(X',Y')$ follow $P'$, and let $Q$ be the maximal coupling of $P(X,Y)$ and $P'(X,Y)$. Then
\[ \tvarg{P(X)}{P'(X)} \le Q(X \neq X') \le Q((X,Y) \ne (X',Y')) = \tvarg{P(X,Y)}{P'(X,Y)}.\]
To prove the tensorization inequality, say $(X_1, \ldots, X_n)$ follow $P$ and $(X'_1, \ldots, X'_n)$ follow $P'$. For each $i$, let $Q_i$ be the maximal coupling of $P(X_i)$ and $P'(X'_i)$, and let $Q$ denote the product distribution of the $Q_i$'s. Note that $Q$ is a valid coupling of $P$ and $P'$, which are each product distributions. Then we have
\begin{align*}
    \tvarg{P}{P'} \le Q(X \neq X') \le \sum_{i=1}^n Q_i(X_i \ne X'_i) = \sum_{i=1}^n \tvarg{P(X_i)}{P'(X'_i)}.\tag*\qedhere
\end{align*}
\end{proof}

\begin{lemma}[Total Variation with Shared Marginal]\label{lem:tv_same_marginal} Let $P$ and $P'$ be joint distributions over random variables $(X,Y)$ such that the \emph{marginals} $P(X)$ and $P'(X)$ coincide. Then, 
\begin{align*}\tvarg{P(X,Y)}{P'(X,Y)} = \E_{X \sim P} \tvarg{P(Y \mid X)}{P'(Y \mid X)}.
\end{align*}
\end{lemma}
\begin{proof}
We first show that $\tvarg{P(X,Y)}{P'(X,Y)} \geq \E_{X \sim P} \tvarg{P(Y \mid X)}{P'(Y \mid X)}$. To see this let $(B_x)_{x \in \Omega}$ be any set of measurable events indexed by $\Omega$. Letting $V(X,Y) = \ind[Y \in B_X]$, the variational form of total variation distance (\Cref{lem:key_tv_props}) implies that
\[ \E_{(X,Y) \sim P}[V(X,Y)] - \E_{(X,Y) \sim P'}[V(X,Y)] \leq \tvarg{P(X,Y)}{P'(X,Y)}. \]
On the other hand, because $P(X) = P'(X)$, we have
\[ \E_{(X,Y) \sim P}[V(X,Y)] - \E_{(X,Y) \sim P'}[V(X,Y)] = \E_{X \sim P}\left[ P(Y \in B_X) - P'(Y \in B_X) \right].\]
Since the choice of $(B_x)_{x \in \Omega}$ was arbitrary, we can conclude that
\[ \tvarg{P(X,Y)}{P'(X,Y)} \geq \E_{X \sim P} \tvarg{P(Y \mid X)}{P'(Y \mid X)}.\]

Now to prove $\tvarg{P(X,Y)}{P'(X,Y)} \leq \E_{X \sim P} \tvarg{P(Y \mid X)}{P'(Y \mid X)}$, we construct a coupling $Q((X,Y), (X',Y'))$ of $P(X,Y)$ and $P'(X,Y)$ as follows. First draw $X \sim P(X)$ and set $X'=X$. Then let $(Y,Y')$ be drawn from the maximal coupling of $P(Y \mid X)$ and $P'(Y \mid X)$ (guaranteed by \Cref{lem:key_tv_props}). By construction, this satisfies that $Q(X,Y) = P(X,Y)$ and $Q(X',Y') = P'(X,Y)$. By the coupling inequality (\Cref{lem:key_tv_props}), we have
\begin{align*}
    \tvarg{P(X,Y)}{P'(X,Y)} \leq Q((X,Y) \neq (X',Y')) = \E_{X \sim P}\left[ \tvarg{P(Y \mid X)}{P'(Y \mid X)}\right].\tag*\qedhere
\end{align*}
% Let $p_X$, $p_{Y \mid X}$, $p_{X,Y}$ denote the densities of $X$ and $Y \mid X$ and $(X,Y)$ under $P$. Similarly for $p'_X$, $p'_{Y \mid X}$ and $p'_{X,Y}$. By Scheffe's Lemma, we have
% \begin{align*}
%     \tvarg{P}{P'} & = \frac{1}{2} \int_{\Omega \times \Omega} |p_{X,Y}(x,y) - p'_{X,Y}(x,y)|  \\
%     &= \frac{1}{2} \int_{\Omega \times \Omega} |p_X(x) p_{Y\mid X}(y \mid x) - p'_X(x) p'_{Y\mid X}(y \mid x)| \\
%     &= \frac{1}{2} \int_{\Omega \times \Omega} p_X(x) |p_{Y\mid X}(y \mid x) - p'_{Y\mid X}(y \mid x)| \\
%     &= \int_{\Omega} p_X(x) \frac{1}{2} \int_{\Omega} |p_{Y\mid X}(y \mid x) - p'_{Y\mid X}(y \mid x)| \\
%     &= \E_{X \sim P} \tvarg{P(Y \mid X)}{P'(Y \mid X)}.
% \end{align*}
\end{proof}

\begin{lemma}[Total Variation with Shared Conditional]\label{lem:tv_same_conditional}  Let $P$ and $P'$ be joint distributions over random variables $(X,Y)$ such that the \emph{conditionals} $P(Y \mid X)$ and $P'(Y \mid X)$ coincide. Then, 
\begin{align*}\tvarg{P(X,Y)}{P(X,Y)} = \tvarg{P(X)}{P'(X)}.
\end{align*}
\end{lemma}
\begin{proof}
By the data processing property of total variation (\Cref{lem:key_tv_props}), we know 
\begin{align*}
 \tvarg{P(X,Y)}{P(X,Y)} \ge \tvarg{P(X)}{P'(X)}. 
 \end{align*}
To prove the lemma, we need to show that the opposite inequality also holds. To do so, let $Q_X$ be the maximal coupling of $P(X)$ and $P'(X)$, and let $Q$ denote the distribution over $((X,Y), (X',Y'))$ induced by first drawing $(X,X')$ from $Q_X$ and then drawing $Y,Y'$ as follows:
\begin{itemize}
\item If $X = X'$, draw $Y \sim P(Y \mid X)$ and set $Y'= Y$.
\item Otherwise, draw $Y$ and $Y'$ independently from $P(Y \mid X)$ and $P'(Y' \mid X')$, respectively
\end{itemize}
It is clear that $Q(X,Y) = P(X,Y)$. To show that $Q$ is a valid coupling, it remains to check that $Q(X',Y') = P'(X,Y)$.  Since $Q'(X) = P'(X)$ by construction, it suffices to check that $Q(Y \mid X' = x') = P'(Y \mid X = x')$ for all $x'$ in the (almost-sure) support of $P(X')$. This follows since 
\begin{align*}
Q(Y \mid X' = x', X) = \begin{cases} P(Y \mid X = x') &\text{ if } X= x'  \\
P'(Y \mid  X= x') &\text{ if } X \ne x'
\end{cases}
\ = P'(Y \mid X = x')
\end{align*}
where we use $P(Y\mid X) = P'(Y \mid X)$. Hence, marginalizing over $X$, $Q(Y \mid X') = P'(Y \mid X = x')$, as needed. Lastly, observe that our construction of $Q$ ensures $Y = Y'$ whenever $X = X'$. Therefore, we conclude

\begin{align*}
\tvarg{P(X,Y)}{P'(X,Y)} \le Q((X,Y) \neq (X',Y')) = Q_X(X \neq X') = \tvarg{P(X)}{P(X')}.\tag*\qedhere
\end{align*}
\end{proof}

\begin{lemma}[Coupled Transport Form]\label{lem:TV_transport} Let $P$ and $P'$ be joint distributions over random variables $(X,Y)$ with coinciding marginals $P(X) = P(X')$ in the first variable. %Let $\calQ$ denote the set of all \emph{coupling} distributions $Q(X,Y,Y')$ whose marginals are $Q(X,Y) = P(X,Y)$ and $Q(X,Y') = P'(X,Y)$. Then $\calQ$ is nonempty and
Then there exists a distribution $Q(X,Y,Y')$ whose marginals satisfy $Q(X,Y) = P(X,Y)$ and $Q(X,Y') = P'(X,Y)$, and for which we have
% \begin{align*}
% \tvarg{P(X,Y)}{P'(X,Y)} = \inf_{Q(X,Y,Y') \in \calQ} Q[Y \ne Y'].
% \end{align*}
\begin{align*}
\tvarg{P(X,Y)}{P'(X,Y)} =  Q[Y \ne Y'].
\end{align*}
\end{lemma}
\begin{proof}
We construct $Q(X,Y,Y')$ as follows. First draw $X \sim P(X)$. Then let $(Y,Y')$ be drawn from the maximal coupling of $P(Y \mid X)$ and $P'(Y \mid X)$ (guaranteed by \Cref{lem:key_tv_props}). By construction, this satisfies that $Q(X,Y) = P(X,Y)$ and $Q(X,Y') = P'(X,Y)$. Moreover, one can see that
\[ Q[Y \neq Y'] = \E_{X \sim P} \tvarg{P(Y \mid X)}{P'(Y \mid X)} = \tvarg{P(X,Y)}{P'(X,Y)},\]
where the first equality follows from the use of the maximal coupling of conditional distributions and the second equality is \Cref{lem:tv_same_marginal}.
\end{proof}

\subsection{General Sensitivity Bounds: Generalizing  \Cref{thm:n_montecarlo_body}}
\label{sec:general_sensitivity}
In general, we address priors over means which are unbounded. We use the following functional to control expectation over their upper tails:
\begin{definition}[Upper Tail Expectation]\label{defn:tail_exp_new}  Let $X$ be a nonnegative random variable on a probability space $(\Omega,\scrF)$ with law $\Pr$ and finite expectation $\E[X] < \infty$. We define its tail expectation, as a function of probabilities $p \in (0,1]$, as
\begin{align*}
\Psi_X(p) &:= \frac{1}{p}\sup_{Y}\E[XY] \quad \\
&\quad \text{s.t. } Y: (\Omega,\scrF) \to [0,1] \text{ and } \E[Y] \le p.
\end{align*}
For $p > 1$, we extend $\Psi_X(p) = \E[X]$. Overloading notation, we let $\cvar_{\theta}(p)$ denote the upper tail function over $\range(\boldmu)$ when drawn from $P_{\theta}$:
\begin{align*}
\cvar_{\theta}(p) :=  \Psi_{\range(\boldmu)}(p) \quad \text{ where } \boldmu \sim P_{\theta}.
\end{align*}
\end{definition}
By taking conditional expectations, one can equivalently verify that $\Psi_X(p) := \frac{1}{p}\sup_{f}\E[Xf(X)]$ is the supremal expected correlation between $X$ and $f(X)$, over functions $f:[0,1] \to \R$ satisfying $\E[f(X)] = p.$ Intuitively, $\cvar_{\theta}(p)$ considers large how conditional expectation of $\frac{1}{p}\E[Xf(X)]$ can be made by concentrating all the mass of $f$ on the upper tail of $X$. We establish key properties, estimates, and a closed form for $\cvar_\theta(p)$ in terms of quantiles of $X$ in \Cref{sec:tail_exp}.

Given this definition, our general sensitivity bound takes the following form:
\begin{theorem}\label{thm:n_montecarlo_general_tail}
Let $\alg(\cdot)$ be an $n$-Monte Carlo family of algorithms on horizon $H \in \N$, and let $\theta,\theta' \in \Theta$. Setting $\veps = \tvarg{P_\theta}{P_{\theta'}}$, we have that 
\begin{align*}
|R(\theta,\alg(\theta)) - R(\theta,\alg(\theta'))| \le 2nH^2 \veps\cdot \Psi_{\theta}(2nH\epsilon),
\end{align*}
where $\Psi_{\theta}(\cdot)$ is the tail expectation defined in \Cref{defn:tail_exp_new}. 
\end{theorem}

We specialize upper bounds on the upper tail expectation for priors satsifying the following tail conditions:
\begin{definition}[Tail Conditions]\label{defn:tail_conditions} We set $\boldmubar_{\theta} := E_{\theta}[\boldmu]$. We say that  $P_{\theta}$ is 
\begin{itemize}
    \item[(a)]  $B$-bounded if $P_{\theta}[\range(\boldmu) \le B]=1$.
    \item[(b)]  Coordinate-wise $\sigma^2$-sub-Gaussian if for all $a \in \calA$, 
    \begin{align*}
    P_\theta(|\mu_a - \mubar_a| \ge t) \le 2 \exp\left( -\frac{t^2}{2\sigma^2} \right).
    \end{align*}
    \item[(c)]  Coordinate-wise $(\sigma^2,\nu)$-sub-Gamma if for all $a \in \calA$, 
    \begin{align*}P_\theta(|\mu_a - \mubar_a| \ge t) \le 2\max\left\{ \exp \left( -\frac{t^2}{2\sigma^2} \right), \exp \left( -\frac{t}{2\nu} \right) \right\}.
    \end{align*}
\end{itemize} 
\end{definition}
For priors satisfying the above tail conditions, \Cref{thm:n_montecarlo_general_tail} specializes as follows:

\begin{theorem}\label{thm:n_montecarlo_tails} Let $\alg(\cdot)$ be an $n$-Monte Carlo family of algorithms on horizon $H \in \N$, and let $\theta,\theta' \in \Theta$. Setting $\veps = \tvarg{P_\theta}{P_{\theta'}}$, we have the following guarantees.
\begin{itemize}
    \item[\emph{(a)}] If $P_\theta$ is $B$-bounded, then $|R(\theta,\alg(\theta)) - R(\theta,\alg(\theta'))| \le 2nH^2 \veps B$.
    \item[\emph{(b)}] If $P_\theta$ is coordinate-wise $\sigma^2$-sub-Gaussian and $\epsilon $, then 
    \begin{align*} |R(\theta,\alg(\theta)) - R(\theta,\alg(\theta'))| \le 2nH^2 \veps \left( \range(E_\theta[\boldmu]) + \sigma \left(8 + 5 \sqrt{\log\left( \tfrac{|\calA|^2}{\min\{1,2nH\epsilon\}}\right)} \right)  \right) 
    \end{align*}
    \item[\emph{(c)}] If $P_\theta$ is coordinate-wise $(\sigma^2, \nu)$-sub-Gamma, then 
    \begin{align*}
    &|R(\theta,\alg(\theta)) - R(\theta,\alg(\theta'))| \\
    &\quad \le 2nH^2 
     \veps \left( \range(E_\theta[\boldmu]) + \sigma \left(8 + 5 \sqrt{\log\left( \tfrac{|\calA|^2}{\min\{1,2nH\epsilon\}}\right)} \right) + \nu \left(11 + 7 \log \left( \tfrac{|\calA|^2}{\min\{1,2nH\epsilon\}} \right) \right)  \right).
     \end{align*}
\end{itemize}
\end{theorem}
The proof of \Cref{thm:n_montecarlo_tails} is a direct consequence of \Cref{thm:n_montecarlo_general_tail} and the estimates from \Cref{lem:tail_condition_bounds} given in \Cref{sec:estimates_upper_bound}. Note that \Cref{thm:n_montecarlo_body} comprises of the first two statements of \Cref{thm:n_montecarlo_tails}.

\subsection{Quantiles, CDFs and Tail Expectations \label{sec:tail_exp}}
Recall classical definitions of quantile and CDF:
\begin{definition}[Quantile and CDF] Given a real-valued random variable $X$ with law $P$, we define its cumulative distribution function, or CDF, by $F_X(t) := P[X \le t]$, and the \emph{quantile function} $q_X(p) := \inf\{t : 1 - F_X(t) \le p\}$. 
\end{definition}
With these definitions in place, we expose the essential properties of $\Psi_X(p)$:
\begin{lemma}[Properties of the Upper Tail expectation]\label{lem:tail_key_props} Then upper tail expectation satisfies the following properties:
\begin{itemize}
\item[\emph{(a)}]\emph{Monotonicity:} $p \mapsto \Psi_X(p)$ is non-increasing in $p$, and $p \mapsto p \cdot \Psi_X(p)$ is non-decreasing.
\item[\emph{(b)}]\emph{Dominance Preservation:} Let $X'$ stochastically dominate $X$, that is, $F_X(t) \ge F_{X'}(t)$ for all $t$. Then, $\Psi_X(p) \le \Psi_{X'}(p)$ for all $p$.
\item[\emph{(c)}]\emph{Translation:} $\Psi_{X}(p) \le c + \Psi_{\max\{X-c,~0\}}(p)$ for any constant $c > 0$. 
\begin{comment}
\item[(d)]\emph{Union Bound:} Let $(X_1,\dots,X_N)$ be jointly distributed nonnegative random variables. Then
\begin{align*}
\Psi_{\max\{X_1,\dots,X_N\}}(p) \le \max_{i}\Psi_{\bar{X}}(p/N),
\end{align*}
where $\bar{X}$ is the dominating random variable with CDF 
\begin{align*}
F_{\bar{X}}(t) = \min_{i \in [N]} F_{X_i}(t) = \min_{i} \Pr[X_i \le t]
\end{align*}
\end{comment}
\item[\emph{(d)}]\emph{Closed Form:} We have that
\begin{align*}
\Psi_X(p) &= \frac{1}{p}\E[X g_p(X)], \quad \text{ where }\\
&\quad g_{p}(u) := \I\{u > q_X(p)\} + \left(p - \Pr[X > q_X(p)]\right)\I\{u = q_X(p)\}.
\end{align*}
In particular, if $F_X(\cdot)$ is continuous, then 
\begin{align*}
\Psi_X(p) = \frac{1}{p}\E[X \I\{X \ge q_X(p)\}] =  \E[X \mid X \ge q_X(p) ] 
\end{align*}
\item[\emph{(e)}]\emph{Useful Estimate:} For any $\alpha > 0$ and $p \in (0,1]$
\begin{align*}
\Psi_X(p) \le \frac{p - \Pr[X > q_p(X)]}{p} + \sum_{i=0}^{\infty} \alpha^{-i} q_X(\alpha^{-i-1}p)
\end{align*}
In particular, if $F_X(\cdot)$ is continuous, then 
\begin{align*}
\Psi_X(p) \le  \sum_{i=0}^{\infty} \alpha^{-i} q_X(\alpha^{-i-1}p).
\end{align*}
\end{itemize}
\end{lemma}
\begin{proof} 
\begin{itemize}
\item[(a)]We can rewrite
\begin{equation}\label{eq:alternate_charac}
\begin{aligned}
\Psi_X(p) &:= \sup_{Y}\E[XY] \quad \\
&\quad \text{s.t. } Y: (\Omega,\scrF) \to [0,\tfrac{1}{p}] \text{ and } \E[Y] \le 1.
\end{aligned}
\end{equation} 
Hence, the constraint on $Y$ becomes strictly less restrictive as $p$ decreases, meaning that $\Psi_X(p)$ is non-increasing. Similarly, $p\cdot \Psi_X(p)$ is the supremum over $\E[XY]$ with $Y \in [0,1]$ and $\E[Y] \le p$, so the constraint becomes more restrictive as $p$ decays, and thus $p \mapsto p\Psi_X(p)$ is non-decreasing. 
\item[(b)]Stochastic domination implies that one can construct a joint distribution $(X,X')$ such that $X' \ge X$ almost
surely (see for example the coupling at the beginning of Section 2.3.1 in \cite{sriboonchita2009stochastic}). This implies that, for any $Y \ge 0$ jointly distributed with $X$ via $(X,Y)$, we can create a joint distribution $(X',Y)$ such that $\E[XY] \le \E[X'Y]$. The bound follows. 
\item[(c)]For any random variable $Y \in [0,1]$ with $\E Y = p$, we have $\frac{1}{p}\E[XY] = c + \frac{1}{p}\E[(X-c)Y] \le c +\frac{1}{p} \E[\max\{X-c,0\}Y] \le c + \Psi_{\max\{X-c,0\}}(p)$. 

\item[(d)] It is clear from the definition that $\E[g_p(X)] = p$ and $0 \le g_p(\cdot) \le 1$. Hence, $\E[Xg_p(X)] \le \Psi_X(p)$. To prove the converse, first observe that for any random variable $Y$, we have
\[ \E[XY] = \E[X \E[Y \mid {X}]]. \]
Since $X$ is non-negative, it suffices to restrict our attention to random variables of the form $Y = f(X)$ where $f: \R \rightarrow [0,1]$ and $\E[f(X)] = p$. We will show that for any such function $f$, if we do not have $f(X) = g_p(X)$ almost surely, then we must have $ \E[X f(X)] < \E[X g_p(X)]$. To see this, it suffices to show that (i) conditioned on the event $X > q_X(p)$, we must have $f(X) = 1$ almost surely, and (ii) if $\Pr[X = q_X(p)] >0$, then conditioned on the event $X = q_X(p)$, we must have $f(X) = p - \Pr[X = q_X(p)]$ almost surely. As the arguments are symmetrical, we will only provide the proof of (i). 

Suppose that (i) does not hold. Then there exist sets $S_+ \subset (q_X(p), \infty)$ and $S_- \subset [0,q_X(p)]$ such that $P(S_+), P(S_-) > 0$, $\E[f(X) \mid X \in S_+] < 1$, and $\E[f(X) \mid S_-] > 0$. Define the function $h: \R \rightarrow [0,1]$ satisfying
\[
h(x) = 
\begin{cases}
\alpha_+ + (1-\alpha_+) f(x) & \text{ if } x \in S_+ \\
(1- \alpha_-) f(x)  & \text{ if } x \in S_- \\
f(x) & \text{ otherwise}.
\end{cases}
\]
where $\alpha_+ = P(S_-) \E[f(X) \mid X \in S_-]$ and $\alpha_- = P(S_+) \E[1 - f(X) \mid X \in S_+]$. By assumption, we have $\alpha_-, \alpha_+ \in [0,1]$, so that $h(x) \in [0,1]$. Further, we can calculate
\begin{align*}
    \E[h(X)] &= \E[f(X) \ind[X \not \in S_+ \cup S_-]] 
    + \E[(\alpha_+ + (1-\alpha_+) f(X)) \ind[X \in S_+ ]] \\
    & \hspace{3em} +\E[(1- \alpha_-) f(X) \ind[X \in S_- ]] \\
    &= \E[f(X)] + \alpha_+ P(S_+) E[1 - f(X) \mid X \in S_+] 
    + \alpha_- P(S_-) \E[f(X) \mid X \in S_-] \\
    &= \E[f(X)] = p.
    \end{align*}
On the other hand, we can see that
\begin{align*}
    &\E[X h(X)] - \E[X f(X)] \\
    &= \E[f(X) X \ind[X \not \in S_+ \cup S_-]] 
    + \E[(\alpha_+ + (1-\alpha_+) f(X)) X \ind[X \in S_+ ]] \\
    & \hspace{3em} +\E[(1- \alpha_-) f(X) X \ind[X \in S_- ]] - \E[X f(X)] \\
    &= \alpha_+ P(S_+) \E[(1-f(X))X \mid X \in S_+] - \alpha_- P(S_-) \E[f(X) X \mid X \in S_-] \\
    &= P(S_+)P(S_-) \E[1-f(X) \mid X \in S_+] \E[f(X) \mid X \in S_-] \\
    &\hspace{2em} \cdot \left( \E\left[ \frac{(1-f(X))X}{\E[1-f(X) \mid X \in S_+]} \mid X \in S_+ \right] -  \E\left[ \frac{f(X)X}{\E[f(X) \mid X \in S_-]} \mid X \in S_- \right] \right) \\
    &>0.
\end{align*}
where the last inequality comes from the fact that
\[ \E\left[ \frac{(1-f(X))X}{\E[1-f(X) \mid X \in S_+]} \mid X \in S_+ \right]\]
is a convex combination of elements from $S_+$ and 
\[ \E\left[ \frac{f(X)X}{\E[f(X) \mid X \in S_-]} \mid X \in S_- \right]\]
is a convex combination of elements from $S_-$, and every element in $S_-$ is strictly smaller than every element of $S_+$.

\item[(e)]  For $\epsilon > 0$, define the sequence of integers $t_{i} = t_i(\epsilon) =  2^i \epsilon + q_p(\alpha^i X)$. Note that $\lim_{i \to \infty} \Pr[X \le t_{i}] = 1$ for each $\epsilon$. Hence, 
\begin{align*}
\Psi_p(X) &= \frac{1}{p}\E[Xg_{p,X}(X)] \\
&\le \frac{1}{p}\E[Xg_{p}(X)\I\{X \le t_0\}] + \frac{1}{p}\sum_{i = 0}^{\infty}  \E[X g_{p}(X) \cdot\I\{t_{i} < X < t_{t+i}\}]\\
&\le 2\epsilon + \frac{1}{p}q_X(p) \E[\I\{X \le t_{0}\}g_p(X)] + \frac{1}{p}\sum_{i = 0}^{\infty}q_p(\alpha^{-(i+1)}X) \Pr[X > t_{i}]
\end{align*}
Since $t_i > q_X(\alpha^{-i} p)$, we have that $\Pr[X > t_{i}] \le \alpha^{-i}p$. Thus, taking $\epsilon \to 0$,
\begin{align*}
\Psi_p(X) \le  \frac{1}{p}q_X(p) \lim_{\epsilon \to 0}\E[\I\{X \le t_{0}(\epsilon)\}g_p(X)] + \sum_{i = 0}^{\infty}q_X(\alpha^{-(i+1)}X) \alpha^{-i}p.
\end{align*}
Finally, we observe that $\E[\I\{X \le t_{0}(\epsilon)\}g_p(X)] \le \Pr[ q_X(p) < X < q_X(p) + \epsilon] + \E[\I\{X \le q_X(p)\}g_p(X)\}$. By continuity of probability measures, the first term tends to $0$ as $\epsilon \to 0$. The second term is precisely $p - \Pr[X > q_X(p)]$. The first bound follows. Note that for continuous CDFs, we necessarily have that $\Pr[X > q_p(X)] = p$, yielding the specialization to continuous CDFs.
\end{itemize}
\end{proof}

\subsection{Upper Tail Expectations under Tail Conditions \label{sec:estimates_upper_bound}}

Each of the conditions in \Cref{defn:tail_conditions} yields a transparent upper bound on $\Psi_{\theta}(p)$:
\begin{lemma}
\label{lem:tail_condition_bounds} Let $\bar{\boldmu}_\theta = E_\theta[\boldmu]$. Then, for $p \in [0,1]$. 
\begin{itemize}
    \item[\emph{(a)}] If $P_{\theta}$ is $B$-bounded, then $\Psi_\theta(p) \le B$ for all $p$. 
    \item[\emph{(b)}] If $P_{\theta}$ is coordinate-wise $\sigma^2$-sub-Gaussian and $\calA$ is finite, then
    \begin{align*}
    \Psi_\theta(p) \le \range(\bar{\boldmu}_\theta) +\sigma(8 + 5\sqrt{\log \tfrac{2|\calA|}{p}})
    \end{align*}
    \item[\emph{(c)}] If $P_{\theta}$ is coordinate-wise $(\sigma^2,\nu)$-sub-Gamma  and $\calA$ is finite, then
    \begin{align*}
    \Psi_\theta(p) \le \range(\bar{\boldmu}_\theta) + \sigma(8 + 5\sqrt{\log \tfrac{2|\calA|}{p}})  + \nu(11 + 7 \log \tfrac{2|\calA|}{p}). 
    \end{align*}
    By \Cref{defn:tail_exp_new}, the above extend to $p \ge 1$ by replacing $p \gets \min\{1,p\}$.
    \end{itemize}
\end{lemma}
The bounds in parts (b) and (c) may be extended to infinite $\calA$ via covering arguments.
\begin{proof}[Proof of \Cref{lem:tail_condition_bounds}] 
We prove each part in sequence
\begin{itemize}
\item[(a)] Suppose $\Psi_{\theta}$ is $B$-bounded. Then, for any random variable $Y \in [0,1]$ with $\E[Y] = p$, we have $\frac{1}{p}\E[\range(\boldmu) Y] \le \frac{1}{p}B\E[Y] = B$. Hence, $\Psi_{\theta}(p) \le B$. 
\item[(b)] Define the random variable $X = \max\{0,\range(\boldmu)-\range(E_{\theta}[\boldmu])\}$. By \Cref{lem:tail_key_props} part (c), we have
\begin{align*}
\Psi_{\theta}(p) \le \range(E_{\theta}[\boldmu]) + \Psi_X(p).
\end{align*}
Further, observe that
\begin{align}
X = \max\{0,\range(\boldmu)-\range(E_{\theta}[\boldmu])\} \le 2\max_{a \in \calA}|\mu_a - E_{\theta}[\mu_a]| \label{eq:two_factor}.
\end{align}

By a union bound over all $a \in \calA$, the sub-Gaussian tail implies
\begin{align*}
F_X(t) = P_{\theta}[X \le t] \le  F_{X'}(t), \quad \text{ where } F_{X'}(t) := 1 -  \begin{cases} 1& t \le 0 \\
\min\left\{1, 2|\calA| e^{- \frac{t^2}{8\sigma^2}}\right\} & t > 0,
\end{cases} 
\end{align*}
where above $e^{- \frac{t^2}{8\sigma^2}} = e^{- \frac{(t/2)^2}{2\sigma^2}}$ accounts for the factor of $2$ in  \Cref{eq:two_factor}, and minimum with $1$ accounts for boundedness of probabilities.

Note that $F_{X'}(t)$ is continuous and is a valid CDF of a random variable, say $X'$, which stochastically dominates $X$ (that is, $F_X(t) \le F_{X'}(t)$ for all $t$). Hence, \Cref{lem:tail_key_props} part (b) implies that $\Psi_X(p) \le \Psi_{X'}(p)$. Moreover, we can compute that the quantile function of $X'$ is 
\begin{align*}
q_{X'}(p) = 2\sqrt{2 \sigma^2 \log \frac{2|\calA|}{p}}.
\end{align*}
Hence, by continuity of $X'$, \Cref{lem:tail_key_props} part (e) implies
\begin{align*}
\Psi_{X'}(p) &\le 2\sum_{i=0}^{\infty} e^{-i} q_{X'}(e^{-i-1}p)\\
&\le 2\sum_{i=0}^{\infty} e^{-i} \sqrt{2 \sigma^2 \log \frac{2|\calA|}{e^{i+1} p}}\\
&\le 2\left(1-\frac{1}{e}\right)^{-1}\sqrt{2 \sigma^2 \log \frac{2|\calA|}{p}} + 2\sqrt{2 \sigma^2}\sum_{i = 0}^{\infty} \underbrace{\sqrt{(i+1)}}_{\le i+1}e^{-i}\\
&\le 2\left(1-\frac{1}{e}\right)^{-1}\sqrt{2 \sigma^2 \log \frac{2|\calA|}{p}} + 2\sqrt{2 \sigma^2}\left(1-\frac{1}{e}\right)^{-2}\\
&\le \sigma\left(8 + 5\sqrt{\log \frac{2|\calA|}{p}}\right). 
\end{align*}
\item[(c)] The proof is analogous, except now we use the bound 
\begin{align*}
q_{X'}(p) \le 2\sqrt{2 \sigma^2 \log \frac{2|\calA|}{p}} + 4 \nu \log \frac{2|\calA|}{p}.
\end{align*}
After some computation, this yields
\begin{align*}
\Psi_{X'}(p)  &\le 2(1-\frac{1}{e})^{-1}\sqrt{2 \sigma^2 \log \frac{2|\calA|}{p}} + 2\sqrt{2 \sigma^2}\left(1-\frac{1}{e}\right)^{-2} \\
&\quad +  4\left(1-\frac{1}{e}\right)^{-1}\nu \log \frac{2|\calA|}{p} + 4\nu\left(1-\frac{1}{e}\right)^{-2} \\
&\le \sigma(8 + 5\sqrt{\log \tfrac{2|\calA|}{p}})  + \nu(11 + 7 \log \tfrac{2|\calA|}{p}). \tag*\qedhere
\end{align*}
\end{itemize}
\end{proof}

\begin{comment}\subsection{}
\begin{definition}[Tail Conditional]\label{defn:tail_exp}  We define the quantile function of the range $q_{\theta}(p) := \inf \{s: P_{\theta}[\range(\boldmu) \ge s] \le p\}$, and the tail conditional
% \begin{align*}
% \cvar_{\theta}(p) := \frac{1}{p} E_{\theta}\left[\range(\boldmu) \cdot \I\{\range(\boldmu) \ge q_{\theta}(p)\}\right]. 
% \end{align*}
\begin{align*}
\cvar_{\theta}(p) :=  E_{\theta}\left[\range(\boldmu) \mid \range(\boldmu) > q_{\theta}(p)\right]. 
\end{align*}
For technical reasons, we define $\cvarbar_\theta(p) := \lim_{p' \downarrow p} \cvar_\theta(p')$.\footnote{This limit is well defined, since $\cvar_\theta(p)$ is nonincreasing in $p$.} \ctcomment{Is $\cvarbar_\theta(p)$ necessary?}
\end{definition}
$\cvar_{\theta}(p)$ measures the conditional expectation of the range of $\boldmu$, conditioned on the event that $\range(\boldmu)$ lies in the uppermost $(1-p)$-quantile of its distribution. It admits the following bounds:
\begin{lemma}[Bounds on Tail Expectations]\label{lem:tail_bounds} If $\range(\boldmu) \le B$ under $P_{\theta}$ almost surely, then $\cvarbar_{\theta}(p) \le B$. If $|\calA| = K < \infty$ and $P_{\theta} = \normal(\nu_{\theta},\Sigma_{\theta})$, then $\cvarbar_{\theta}(p) \le \range(\nu_{\theta}) + 4\lambda_{\max}(\Sigma_{\theta})\sqrt{\log(K/p)}$. 
\end{lemma}
\mscomment{proof}
\end{comment}

\iftoggle{neurips}
{
\subsection{Proof of \Cref{prop:prior_sensitivity_tv_bound} \label{app:prop_tv_sens}}
To prove \Cref{prop:prior_sensitivity_tv_bound}, we first decompose the total variation across steps $h$ via an analogue of the performance difference lemma. 
\input{body/app_prior_sensitivity_tv_bound}
}
{
}

\subsection{Proof of \Cref{lem:perf_diff_bandits} \label{sec:proof:lem:perf_diff_bandits}}

In the interest of generalizing our results to POMDPs (\Cref{app:gen_bayes_dec}), we will prove \Cref{lem:perf_diff_bandits} by establishing a nearly identical lemma which makes explicit the exact properties of the trajectors $\traj_h$ needed for \Cref{lem:perf_diff_bandits} to hold:
\begin{lemma}\label{lem:performance_difference_decomp} Let $P,P'$ be two laws over abstract random variables $\boldmu$, $\traj_{1},\dots,\traj_H$ and $a_{1},\dots,a_H$ such that the following properties hold:
\begin{enumerate}
    \item $\traj_{h-1}$ is a deterministic function of $\traj_h$.
    \item The conditional distributions of $\traj_h$ given $a_h,\traj_{h-1}$ and $\boldmu$ are the same: $P'(\traj_h \mid  a_h, \traj_{h-1},\boldmu) = P(\traj_h \mid a_h,\traj_{h-1},\boldmu)$
    \item Under both $P$ and $P'$, $a_h$ is independent of $\boldmu$ given $\traj_{h-1}$.
\end{enumerate}
Then, the following inequality holds:
\begin{align*}
\tvarg{P(\boldmu,\traj_H)}{P'(\boldmu,\traj_H)} \le \sum_{h=1}^H \Expop_{\traj_{h-1}\sim P}\left[\tvarg{P(a_h \mid \traj_{h-1})}{P'(a_h \mid \traj_{h-1})}\right].
\end{align*}
\end{lemma}
Immediately, \Cref{lem:perf_diff_bandits} is obtained by taking $P = P_{\theta,\alg}$ and $P' = P_{\theta,\alg'}$:
\begin{enumerate}
\item Condition 1 is clear.
\item Condition 2 holds because the only part of $\traj_h$ not determined by $(a_h,\traj_{h-1},\boldmu)$ is the reward $r_h$, and under both $P_{\theta,\alg}$ and $P_{\theta,\alg}'$,  $r_h \sim \calD(\boldmu) \big{|}_{a_h}$ is drawn from the same conditional distribution.% \akcomment{We have not introduce the $\calD_{\mathrm{rew}}$ notation before, it is just $\calD$.}
\item For Condition 3, first suppose there is no random seed $\boldxi$. Then Condition 3 holds because the distribution of $a_h \sim f_h(\cdot \mid \traj_{h-1})$ is just a function of $\traj_{h-1}$. If there is a random seed, then letting $P = P_{\theta, \alg}$, we have
\begin{align*}
    P(a_h, \boldmu \mid \traj_{h-1})
    &= \E\left[ \E\left[ P(a_h, \boldmu \mid \tau_{h-1}, \boldxi) \mid \boldxi \right] \mid \traj_{h-1}\right] \\
    &= \E\left[ \E\left[ f_h(a_h \mid \tau_{h-1}, \boldxi) P(\boldmu \mid \traj_{h-1}, \boldxi) \mid \boldxi \right] \mid \traj_{h-1}\right] \\
    &= \E\left[ \E\left[ f_h(a_h \mid \tau_{h-1}, \boldxi) \mid \boldxi \right] P(\boldmu \mid \traj_{h-1}) \mid \traj_{h-1}\right] \\
    &= P(a_h \mid \tau_{h-1}) P(\boldmu \mid \traj_{h-1})
\end{align*}
where the second equality follows from the fact that $a_h \sim f_h(\cdot \mid \boldxi, \traj_{h-1})$ and the third line follows from the fact that $\boldmu$ is independent of $\boldxi$ conditioned on $\tau_{h-1}$. The same argument holds symmetrically for $P' = P_{\theta, \alg'}$. Thus, Condition 3 holds regardless of whether or not there is a random seed.
\end{enumerate}

\begin{proof}[Proof of \Cref{lem:performance_difference_decomp}]
For brevity, we define augmented trajectories containing the (unknown) mean parameter $\trajbar_h = (\boldmu,\traj_h)$ for $h = 1,2,\dots,H$.
We further define $E$ and $E'$ to be the expectations under $P$ and $P'$, respectively. Fix any event $\calE$ in the $\sigma$-algebra generated by $\trajbar_H$; the total variation 
\begin{align*}
\tvarg{P(\trajbar_H)}{P'(\trajbar_H)} = \sup_{\calE} P[\trajbar_H \in \calE]-P'[\trajbar_H \in \calE]
\end{align*} can be expressed as the supremal difference over such events (\Cref{lem:key_tv_props}).  We can then view this difference as the difference in rewards between two time-inhomogeneous Markov reward processes, with states $\trajbar_h$ at step $h$ (note that the Markov property is trivially satisfied because $\traj_{h-1}$ is assumed to be deterministic function of $\traj_{h}$ by assumption), with identical rewards: at step $H$ the reward is $r_H(\trajbar_H) = \I\{\trajbar_H \in \calE\}$, and steps $h < H$, the reward is zero. Let $V'_h(\cdot)$ denote the value function of step $h$ under the $P'$ reward process, the performance difference lemma~\cite{kakade2003sample} then yields
\begin{align}
P[\trajbar_H \in \calE]-P'[\trajbar_H \in \calE] = \sum_{h=1}^H \Expop_{\trajbar_{h-1}\sim P}\left[ E[V'_{h}(\trajbar_h) \mid \trajbar_{h-1}] - E'[V'_{h}(\trajbar_{h}) \mid \trajbar_{h-1}]\right].  \label{eq:P_diff}
\end{align}
Since the total reward collected is at most $1$ and no less than $0$, $V'_h(\cdot) \in [0,1]$. Hence, by the variational characterization of total variation (\Cref{lem:key_tv_props}), 
\begin{align*}
E[V'_{h}(\trajbar_h) \mid \trajbar_{h-1}] - E'[V'_{h}(\trajbar_{h}) \mid \trajbar_{h-1}] \le \tvarg{P(\trajbar_h \mid \trajbar_{h-1})}{P'(\trajbar_h \mid \trajbar_{h-1})}
\end{align*}

To conclude, it suffices to verify the inequality 
\begin{align}
\tvarg{P(\trajbar_h \mid \trajbar_{h-1})}{P'(\trajbar_h \mid \trajbar_{h-1})} \le \tvarg{P(a_h \mid \traj_{h-1})}{P'(a_h \mid \traj_{h-1})}. \label{eq:perf_diff_equal_TV}
\end{align}
To verify \eqref{eq:perf_diff_equal_TV}, let us fix a step $h$ and realization of $\trajbar_{h-1}$, and set $Q = P(\cdot \mid \trajbar_{h-1})$ and $Q' =  P'(\cdot \mid \trajbar_{h-1})$. Further, applying the data-processing inequality (\Cref{lem:key_tv_props}) followed by \Cref{lem:tv_same_conditional} with $X = a_h$ and $Y = \trajbar_{h}$ gives that $\tvarg{Q(Y)}{Q'(Y)} \le\tvarg{Q(X,Y)}{Q'(X,Y)} = \tvarg{Q(X)}{Q(X)}$. Undoing the notational subsitutions, we have shown
\begin{align*}
\tvarg{P(\trajbar_h \mid \trajbar_{h-1})}{P'(\trajbar_h \mid \trajbar_{h-1})} \le \tvarg{P(a_h \mid \trajbar_{h-1})}{P'(a_h \mid \trajbar_{h-1})}. 
\end{align*}
Finally, we have 
\begin{align*}
\Expop_{\trajbar_{h-1} \sim P}\tvarg{P(a_h \mid \trajbar_{h-1})}{P'(a_h \mid \trajbar_{h-1})} 
&= \Expop_{(\boldmu, \traj_{h-1}) \sim P}\tvarg{P(a_h \mid \boldmu, \traj_{h-1})}{P'(a_h \mid \boldmu, \traj_{h-1})}\\
&= \Expop_{(\boldmu, \traj_{h-1}) \sim P}\tvarg{P(a_h \mid \traj_{h-1})}{P'(a_h \mid \traj_{h-1})}\\
&=\Expop_{\traj_{h-1} \sim P}\tvarg{P(a_h \mid \traj_{h-1})}{P'(a_h \mid \traj_{h-1})},
\end{align*}
where the first equality follows from the definition of $\trajbar_{h-1} = (\boldmu, \traj_{h-1})$, the second equality follows from the assumption that $a_h$ is independent of $\boldmu$ given $\traj_{h-1}$, and the third equality follows from marginalization.

Concluding, we have shown that
\begin{align}
P[\trajbar_H \in \calE]-P'[\trajbar_H \in \calE] \le \sum_{h=1}^H \Expop_{\traj_{h-1} \sim P}\tvarg{P(a_h \mid \traj_{h-1})}{P(a_h \mid \traj_{h-1})}\label{eq:P_diff}
\end{align}
By the definition of total variation, $\tvarg{P(\trajbar_H)}{P'(\trajbar_H)}$ is the supremum of the left-hand side over events $\calE$ (\Cref{defn:total_variation}), and we have defined $\trajbar_H = (\boldmu,\traj_H)$. Thus,
\begin{align*}
\tvarg{P(\boldmu,\traj_H)}{P'(\boldmu,\traj_H)} \le \sum_{h=1}^H \Expop_{\traj_{h-1} \sim P}\tvarg{P(a_h \mid \traj_{h-1})}{P'(a_h \mid \traj_{h-1})}.\tag*\qedhere
\end{align*}
\end{proof}

\subsection{Proof of \Cref{thm:n_montecarlo_general_tail}}
\label{app:proof_n_montecarlo_general_tail}
We now turn to the proof of \Cref{thm:n_montecarlo_general_tail}. Recall the statement of the theorem. 

\textbf{\Cref{thm:n_montecarlo_general_tail}.} Let $\alg(\cdot)$ be an $n$-Monte Carlo family of algorithms on horizon $H \in \N$, and let $\theta,\theta' \in \Theta$. Then, setting $\veps = \tvarg{P_\theta}{P_{\theta'}}$,
\[|R(\theta,\alg(\theta)) - R(\theta,\alg(\theta'))| \le 2nH^2 \veps \cdot \cvar_{\theta}(2nH \veps). \]

\Cref{prop:prior_sensitivity_tv_bound} states that
\begin{align*}
\tvarg{P_H}{P_H'} \le 2nH\cdot \tvarg{P_\theta}{P_{\theta'}},\end{align*}
where $P_H = P_{\theta,\alg(\theta)}(\boldmu,\traj_H)$ and $P_H' = P_{\theta,\alg(\theta')}(\boldmu,\traj_H)$. Thus, \Cref{thm:n_montecarlo_general_tail} is a direct consequence of \Cref{prop:prior_sensitivity_tv_bound}, monotonicty of $p \mapsto p \cdot \Psi_\theta(p)$ (\Cref{lem:tail_key_props}) and the following lemma.

\begin{lemma}\label{lem:tv_to_tail}
Given two algorithms $\alg$ and $\alg'$ and ground truth parameter $\theta$, let $\delta = \tvarg{P_H}{P_H'}$, where again $P_H = P_{\theta,\alg}(\boldmu,\traj_H)$ and $P_H' = P_{\theta,\alg'}(\boldmu,\traj_H)$ denote the marginals over induced trajectories and means. Then, 
\begin{align*}
|R(\theta,\alg) - R(\theta,\alg')| \le H\cdot\delta \cdot \Psi_{\theta}(\delta).
\end{align*}
\end{lemma}
\begin{proof}
Recall that
\begin{align*}
R(\theta,\alg) - R(\theta,\alg') = E_{\theta,\alg}\left[\sum_{h=1}^H \mu_{a_h}\right] - E_{\theta,\alg'}\left[\sum_{h=1}^H \mu_{a_h}\right] 
\end{align*}
To bound this difference, we place all random variables on the same  probability space. Adopt the shorthand $P_H = P_{\theta,\alg}(\boldmu,\traj_H)$ and $P_H' = P_{\theta,\alg'}(\boldmu,\traj_H)$. Since $P_H(\boldmu) = P_H'(\boldmu)$, \Cref{lem:TV_transport} ensures the existence of a coupling $Q(\boldmu,\traj_H,\traj_H')$ such that
\begin{align}
Q(\boldmu,\traj_H) = P_H, \quad Q(\boldmu,\traj_H') = P_H', \quad Q[\traj_H \ne \traj_H'] =  \tvarg{P_H}{P_H'} = \delta. \label{eq:coupling}
\end{align}
Letting $E_Q$ denote expectations under this coupling, and $a_h'$ the actions within $\traj_H'$, we then have
\begin{align*}
R(\theta,\alg) - R(\theta,\alg') &= E_{Q}\left[\sum_{h=1}^H \mu_{a_h} - \mu_{a_h'}\right] \\
&\le E_{Q}\left[\sum_{h=1}^H \range(\boldmu) \I\{a_h \ne a_h'\}\right] \\
&\le H E_{Q}[\range(\boldmu)\I\{\traj_H \ne \traj_H'\}].
\end{align*}
Note that $\I\{\traj_H \ne \traj_H'\} \in [0,1]$ and, by construction of the coupling $Q$, $E_Q[\I\{\traj_H \ne \traj_H'\}] = \delta$. Hence, by the definition of the tail expectation, $E_{Q}[\range(\boldmu)\I\{\traj_H \ne \traj_H'\}] \le \delta \Psi_{\theta}(\delta)$. The bound follows.
\end{proof}

%% file: appendix/app_formal_algorithms.tex
\newpage
\iftoggle{neurips}
{
\section{Specification of Monte Carlo Algorithms \label{app:mc_algs}}
In this appendix, we elaborate upon various examples of algorithms satisfying the $n$-Monte Carlo property. 
\input{body/body_mc_algorithms} 
}
{
\section{Verification of the Monte Carlo Property \label{app:mc_algs}}
}

\subsection{Proof of \Cref{lem:samples_to_n_mc}.}
\label{sec:lem:samples_to_n_mc}
\Cref{lem:key_tv_props} guarantees the existence of a \emph{maximal coupling} $Q(\boldmu, \boldmu')$ between $P := P_{\theta}(\boldmu \mid \tau_{h-1})$ and $P' := P_{\theta'}(\boldmu \mid \tau_{h-1})$; that is, a joint law over $(\boldmu, \boldmu')$ with marginals  $\boldmu \sim P$ and $\boldmu' \sim P'$, and for which $\Pr_{(\boldmu,\boldmu') \sim Q}[\boldmu \ne \boldmu'] = \tvarg{P}{P'}$.  Using $Q$, we construct a coupling $\bar{Q}$ of $P_{\alg(\theta)}(a_{h} \mid \traj_{h-1})$ and $P_{\alg(\theta')}(a_{h} \mid \traj_{h-1})$:
\begin{itemize}
    \item Draw $(\boldmu_1, \boldmu_1'), \ldots, (\boldmu_n, \boldmu_n')$ i.i.d. from $Q$.
    \item If $\boldmu_1 = \boldmu_1', \ldots, \boldmu_n = \boldmu_n'$, then draw $a_h \sim f_h(\cdot \mid \boldmu_1, \ldots, \boldmu_n)$ and let $a_h' = a_h$.
    \item Otherwise, let $a_h \sim f_h(\cdot \mid \boldmu_1, \ldots, \boldmu_n)$ and $a_h' \sim f_h(\cdot \mid \boldmu_1', \ldots, \boldmu_n')$ independently.
\end{itemize}
It is easily verified that this defines a valid coupling of $P_{\alg(\theta)}(a_{h} \mid \traj_{h-1})$ and $P_{\alg(\theta')}(a_{h} \mid \traj_{h-1})$. Hence, \Cref{lem:key_tv_props} ensures that
\begin{align*}
\tvarg{P_{\alg(\theta')}(a_{h} \mid \traj_{h-1})}{P_{\alg(\theta')}(a_{h} \mid \traj_{h-1})} \le \bar{Q}(a_h \neq a'_h).
\end{align*}
Continuing, we conclude
\begin{align*}
    \bar{Q}(a_h \neq a'_h)  
    \leq  \P(\exists i \, : \, \boldmu_i \neq \boldmu_i')
    \leq  \sum_{i=1}^n Q(\boldmu_i \neq \boldmu_i') 
    =  n \cdot \tvarg{P_{\theta}(\boldmu \mid \tau_{h-1})}{P_{\theta'}(\boldmu \mid \tau_{h-1})}.
\end{align*}
Thus, $\alg(\theta)$ is $n$-Monte Carlo.
\qed

\subsection{Monte Carlo Property of $\twompc$ (\Cref{lem:twompc_n_mc}) \label{sec:MC_prop_two_mpc}}
Let us briefly recall the specified of $\twompc$ (\Cref{alg:two_mpc}); we include an explicit dependence on horizon to avoid confusion. At each step $h$, we 
\begin{itemize}
\item Sample means $\boldmutil^{(a,i)}_h\sim P_{\theta}[\cdot \mid \traj_{h-1}]$ for each $a \in \calA$ and $i \in [k_1]$
\item For each such $(a,i)$, we samplea reward vector $\boldrtil^{(a,i)}_h \sim \calD(\boldmutil^{(a,i)})$
\item For each $(a,i)$ and $j \in [k_2]$, we sample ``look-ahead'' means $\boldmuhat_h^{(a,i,j)} \sim P_{\theta}[\in \cdot \mid \traj_{h-1} \text{ and } \{r_{h+1,a} = \rtil_{a}^{(a,i)}\}]$
\item The action selected is a deterministic function of the vector  $(\boldmutil_h^{(a,i)},\boldmuhat_h^{(a,i,j)})_{a \in \calA, i \in [k_1], j \in [k_2]}$. 
\end{itemize}
Continuing, fix a step $h \in \N$, and introduce the shorthand 
\begin{align*}
\mathbf{Z}_{h,a,i} := (\boldmutil^{(a,i)},(\boldmuhat^{(a,i,j)})_{j \in [k_2]}), \quad \text{and} \quad \mathbf{Z}_h = (\mathbf{Z}_{h,a,i})_{a \in \calA, i \in [k_1]}.
\end{align*} 
The $\twompc$ decision rule is then a deterministic function of $\mathbf{Z}_h$, so it suffices to bound
\begin{align*}
\tvarg{ P_{\alg(\theta)}(\mathbf{Z}_h \mid \traj_{h-1})}{P_{\alg(\theta')}(\mathbf{Z}_h \mid \traj_{h-1})}.
\end{align*}
Moreover, given $\traj_{h-1}$, $\mathbf{Z}_{h,a,i}$ are independent across $a \in \calA$ and $i \in [k_1]$. Thus, by the tensorization property of total variation (\Cref{lem:key_tv_props}),
\begin{align}
&\tvarg{ P_{\alg(\theta)}(\mathbf{Z}_h \mid \traj_{h-1})}{P_{\alg(\theta')}(\mathbf{Z} \mid \traj_{h-1})} \nonumber\\
&\quad\le \sum_{a \in \calA}\sum_{i \in [k_1]}\tvarg{ P_{\alg(\theta)}(\mathbf{Z}_{h,a,i} \mid \traj_{h-1})}{P_{\alg(\theta')}(\mathbf{Z}_{h,a,i} \mid \traj_{h-1})}. \label{eq:summing_thin}
\end{align}
We now decouple the summands above by appealing to the following property:
\begin{claim}\label{claim:twompc_props} For the given $h \in [H]$, fix $a \in \calA $ and $i \in [k_1]$, and introduce the short hand 
\begin{align*}
X_0 = \boldmutil^{(a,i)}_h, \quad \ X_j = \boldmuhat^{(a,i,j)}_h, ~j \in [k_1], \quad \text{and} \quad Y = \rtil^{(a,i)}_{h,a},
\end{align*}
where all random variables above are those simulated by $\twompc$ at step $h$.
 Then, 
\begin{itemize}
    \item[(a)] Under $P_{\alg(\theta)}(\cdot \mid \traj_{h-1},Y)$ (and similarly under $\theta'$), $(X_0,X_1,\dots,X_{k_2})$ are independent and identically distributed. 
    \item[(b)] The distributions  $P_{\alg(\theta)}(Y \mid \traj_{h-1},X_0) $ and $P_{\alg(\theta')}(Y \mid \traj_{h-1},X_0)$ are identical.
    \end{itemize}
\end{claim}
\begin{proof} Let us start with point $(b)$. Recall that $\rtil^{(a,i)}_{a}$ denotes the action-$a$ entry of $\tilde{\boldr}^{(a,i)}$, which is drawn from the distribution $\calD(\boldmutil^{(a,i)})$, regardless of the parameter $\theta$. Hence, given $X_0 = \boldmutil^{(a,i)}$, the distribution of $Y$ is identical under $\alg(\theta)$ and $\alg(\theta')$.

Let us turn to part (a). We focus on $P_\theta$, as the argument for $P_{\theta'}$ is identical. We notice that under a given $\theta$, 
\begin{align*}
P_{\theta}(\boldmutil^{(a,i)}_h \mid \traj_{h-1}, \tilde{r}^{(a,i)}_{h,a}= r ) = P_{\theta}(\boldmu \mid  \traj_{h-1}, r_{h,a} = r );
\end{align*} 
in words, the posterior of the simulated mean $\boldmutil^{(a,i)} $ given simulated reward $\rtil^{(a,i)}_{a}$ is equal to the posterior of the \emph{true} mean given that the reward $r_{h,a}$. This is because
\begin{itemize}
\item[{(a)}] $P_{\theta}(\boldmutil^{(a,i)}_h\mid \traj_{h-1}) = P_{\theta}(\boldmu \mid \traj_{h-1})$ (that is, $\boldmutil^{(a,i)}$ is drawn from the true posterior given $\traj_{h-1}$)
\item[{(b)}] The condition distribution of $\rtil^{(a,i)}_{a} \mid \boldmutil^{(a,i)},\traj_{h-1}$ is equal to the condition distribution of $r_{h,a} \mid \boldmu,\traj_{h-1}$. Note that conditioning on the trajectory is immaterial to the draw of these reward, and both are given by the restriction of the reward distribution $\calD$ to entry $a \in \calA$.
\end{itemize} 
Moreover definition of the $\twompc$ procedure, 
\begin{align*}
P_{\theta}(\boldmuhat^{(a,i,j)} \mid \traj_{h-1}, \tilde{r}_{h,a}^{(a,i)} = r ) := P_{\theta}(\boldmu \mid \traj_{h-1}, r_{h,a} = r  )
\end{align*} for all $i \in [k_2]$ as well. Hence, $P_{\theta}(X_i \mid Y)$ are identically distributed. 

To conclude part (a), we must verify independence. This holds since $X_1,\dots, X_{k_2}$ are i.i.d. draws from $P_{\theta}( \cdot \mid \traj_{h-1}, (a_h,r_h) = (a,Y))$, regardless of the realization of $X_0$. Hence, $X_1,\dots,X_{k_2}$ are conditionally independent of each other, and of $X_0$, given $Y,\traj_{h-1}$. 
\end{proof}

The next lemma lets us put the two properties in \Cref{claim:twompc_props} to use:
\begin{lemma}\label{lem:decoupling}
Let $P$ and $P'$ be two probability distributions over random variables $(Y,X_0,X_1,\dots,X_k)$ such that 
\begin{itemize}
    \item[(a)] $X_{0:k} \mid Y$ are independent and identically distributed under both $P$ and $P'$
    \item[(b)] The conditionals $P(Y \mid X_0) = P'(Y \mid X_0)$ are the same. 
\end{itemize}
Then, $\tvarg{P(X_{0:k})}{P'(X_{0:k})} \le (2k+3)\tvarg{P(X_{0})}{P'(X_{0})}$.
\end{lemma}
Before we prove \Cref{lem:decoupling}, we show how it implies \Cref{lem:twompc_n_mc}. Fix indices $a \in \calA$ and $i \in [k_1]$.
Recall the random variables $(X_{0:k_2},Y)$ defined in \Cref{claim:twompc_props}, and note that $\mathbf{Z}_{a,i}$ is precisely given by $X_{0:k_2}$; that is, 
\begin{align*}
\tvarg{ P_{\alg(\theta)}(\mathbf{Z}_{a,i} \mid \traj_{h-1})}{P_{\alg(\theta')}(\mathbf{Z}_{a,i} \mid \traj_{h-1})} = \tvarg{ P_{\alg(\theta)}(X_{0:k_2} \mid \traj_{h-1})}{P_{\alg(\theta')}(X_{0:k_2} \mid \traj_{h-1})}.
\end{align*}
\Cref{claim:twompc_props} further ensures that, under the history-conditioned laws $P \gets P_{\theta}(\cdot \mid \traj_{h-1})$ and $P' \gets P_{\theta}(\cdot \mid \traj_{h-1})$, $X_{0:k_2}$  satisfy the conditions of \Cref{lem:decoupling}. This implies that 
\begin{align*}
&\tvarg{ P_{\alg(\theta)}(X_{0:k_2} \mid \traj_{h-1})}{P_{\alg(\theta')}(X_{0:k_2} \mid \traj_{h-1})} \\
&\quad\le (2k_2+3)\tvarg{P_{\alg(\theta)}(X_{0} \mid \traj_{h-1})}{P_{\alg(\theta')}(X_{0} \mid \traj_{h-1})}. 
\end{align*}
Finally, by construction, $X_0 = \boldmutil^{(a,i)}$ is drawn from $P_{\theta}(\boldmu \mid \traj_{h-1})$. Hence, we conclude
\begin{align*}
\tvarg{ P_{\alg(\theta)}(\mathbf{Z}_{a,i} \mid \traj_{h-1})}{P_{\alg(\theta')}(\mathbf{Z}_{a,i} \mid \traj_{h-1})} \le (2k_2+3)\tvarg{P_{\theta}(\boldmu \mid \traj_{h-1})}{P_{\theta'}(\boldmu \mid \traj_{h-1})}.
\end{align*}
And by \Cref{eq:summing_thin},
\begin{align*}
    \tvarg{ P_{\alg(\theta)}(\mathbf{Z} \mid \traj_{h-1})}{P_{\alg(\theta')}(\mathbf{Z} \mid \traj_{h-1})}
    \le |\calA|k_1(2k_2+3)\tvarg{P_{\theta}(\boldmu \mid \traj_{h-1})}{P_{\theta'}(\boldmu \mid \traj_{h-1})},
\end{align*}
yielding the $n$-Monte Carlo property for $n = |\calA|k_1(2k_2+3)$.
\begin{proof}[Proof of \Cref{lem:decoupling}] Introduce the measure $\tilP$ under which $\tilP(Y) = P(Y)$ and $\tilP(X_{0:k} \mid Y) = P'(X_{0:k} \mid Y)$. By the data processing and triangle inequalities
\begin{align*}
\tvarg{P(X_{0:k})}{P'(X_{0:k})} &\le \tvarg{P(X_{0:k},Y)}{P'(X_{0:k},Y)}\\
&\le \tvarg{P(X_{0:k},Y)}{\tilP(X_{0:k},Y)} + \tvarg{P'(X_{0:k},Y)}{\tilP(X_{0:k},Y)}.
\end{align*} 
Using \Cref{lem:tv_same_conditional} and the fact that $\tilP(X_{0:k} \mid Y) = P'(X_{0:k} \mid Y)$, followed by the fact $\tilP(Y) = P(Y)$ and the data-processing inequality again,  we have
\begin{align}
\tvarg{P'(X_{0:k},Y)}{\tilP(X_{0:k},Y)} &= \tvarg{P'(Y)}{\tilP(Y)} \nonumber\\
&= \tvarg{P'(Y)}{P(Y)} \le \tvarg{P'(X_0,Y)}{P(X_0,Y)}. \label{eq:TV_Prime_Til}
\end{align}
On the other hand, using \Cref{lem:tv_same_marginal} and the fact that $\tilP( Y) = P( Y)$.
\begin{align}
\tvarg{P(X_{0:k},Y)}{\tilP(X_{0:k},Y)} &= \E_{Y \sim P(Y)} \tvarg{P(X_{0:k} \mid Y)}{\tilP(X_{0:k} \mid Y)}. \label{eq:P_Ptil}
\end{align}
Now, observe that $\tilP(X_{0:k} \mid Y) = P'(X_{0:k} \mid Y)$, and under both $P$ and $P'$, $X_{0:k} \mid Y$ are independent and indentically distributed. Hence, from the decoupling property (\Cref{lem:key_tv_props}), \Cref{eq:P_Ptil} yields
\begin{align*}
\tvarg{P(X_{0:k},Y)}{\tilP(X_{0:k},Y)} &= \sum_{i=0}^k\E_{Y \sim P(Y)}\tvarg{P(X_{i} \mid Y)}{\tilP(X_{i} \mid Y)} \\
&= (k+1)\E_{Y \sim P(Y)}\tvarg{P(X_{0} \mid Y)}{\tilP(X_{0} \mid Y)}.
\end{align*}
Reversing the decoupling property, 
\begin{align*}
\tvarg{P(X_{0:k},Y)}{\tilP(X_{0:k},Y)} = (k+1)\tvarg{P(X_{0},Y)}{\tilP(X_{0}, Y)}.
\end{align*}
Now, we invoke the triangle inequality once more to get
\begin{align*}
&\tvarg{P(X_{0:k},Y)}{\tilP(X_{0:k},Y)} \\
&\quad\le (k+1)\tvarg{P(X_{0},Y)}{P'(X_{0}, Y)} + (k+1)\tvarg{\tilP(X_{0},Y)}{P'(X_{0}, Y)}. 
\end{align*}
Invoking \Cref{eq:TV_Prime_Til}, we have $\tvarg{\tilP(X_{0},Y)}{P'(X_{0}, Y)} \le \tvarg{P(X_{0},Y)}{P'(X_{0}, Y)}$, yielding a final bound of
\begin{align*}
\tvarg{P(X_{0:k})}{P'(X_{0:k})} &\le \tvarg{P(X_{0:k},Y)}{P'(X_{0:k},Y)}\\
&\le \tvarg{P(X_{0:k},Y)}{\tilP(X_{0:k},Y)} + \tvarg{P'(X_{0:k},Y)}{\tilP(X_{0:k},Y)}\\
&\le (2k+3)\tvarg{P(X_{0},Y)}{P'(X_{0}, Y)}
\end{align*}
Finally, since $P(Y \mid X_0)$ and $P'(Y \mid X_0)$ coincide, \Cref{lem:tv_same_conditional} yields that \[\tvarg{P(X_{0},Y)}{P'(X_{0}, Y)} = \tvarg{P(X_{0})}{P'(X_{0})}.\] 
The bound follows.
\end{proof}

%% file: body/lower_bound_new.tex
%!TEX root = ../arxiv_main.tex
\newcommand{\ahat}{\hat{a}}
\newcommand{\abar}{\bar{a}}
\newcommand{\bbar}{\bar{b}}
\newcommand{\atil}{\tilde{a}}
\newcommand{\btil}{\tilde{b}}
\newcommand{\bhat}{\hat{b}}
\newcommand{\cnot}{c_0}
\newcommand{\unifsim}{\overset{\mathrm{unif}}{\sim}}
\section{Lower Bounds \label{app:LB}}
\begin{theorem}\label{thm:an_lb} There is a universal constant $\cnot \ge 1$ such that the following holds. Fix any $k \in \N$, and a tolerance $\eta \in (0,1/4)$. Then for all horizons $H \ge \frac{\cnot}{\eta}$ and errors $\epsilon \le \frac{\eta}{\cnot kH}$, there are two priors $\theta,\theta'$ over bandit instances with $|\calA| = H\ceil{\frac{\cnot}{\eta}}$ arms such that (a) $\tvarg{P_{\theta}}{P_{\theta'}} = \epsilon$, (b) $\boldmu \in [0,1]^{\calA}$ with probability one under both $P_{\theta}$ and $P_{\theta'}$, and (c) the difference in rewards collected by $\kshot(\theta)$ and $\kshot(\theta')$ is at least
\begin{align*}
R(\theta,\kshot(\theta)) - R(\theta,\kshot(\theta')) \ge \left(\frac{1}{2} - \eta\right)k \epsilon H^2.
\end{align*}
\end{theorem}
Note that, by rescaling, an $\left(\frac{1}{2} - \eta\right)k \epsilon H^2 B$ bound holds against $B$ -bounded priors for any $B > 0$. The proof of \Cref{thm:an_lb} is given in \Cref{ssec:thm_an_lb}. A proof sketch is given below. The construction is somewhat involved, and relies on a carefully contrived prior and deterministic rewards.

In \Cref{sec:prior_sens_lb_tv}, we also provide a simpler construction that provides a sharp converse to \Cref{prop:prior_sensitivity_tv_bound} and removes the deterministic rewards condition of \Cref{thm:an_lb} to allow for Bernoulli rewards.

\begin{proof}[Proof Sketch of \Cref{thm:an_lb}]
We construct a (rather contrived) prior $\theta$ over means $\boldmu$ with $N+1$ arms; the rewards are deterministically equal to the prior mean. Under $P_\theta$, a single arm $\abar \in [N]$ is chosen uniformly at random to have reward close to $1$, and the rest have zero reward. The $N+1$-st arm has an $\epsilon$ probability of having reward exactly \emph{equal to} 1, and thus an $\epsilon$ probability of being the best. Otherwise, the $N+1$-st arm has low reward, but the value of its reward encodes the location of the optimal arm $\abar \in [N]$.

At each stage $h$, we show that $\kshot(\theta)$ has approximately $k \epsilon$ probability of selecting arm $N+1$ under the hunch that it may be the best, only to find that it reveals the location of the best arm. After this revelation, the algorithm knows the best arm with certainty, and thus collects reward close to $1$ for the remainder of the episode. Hence, at each step $h$, there is a close to $k \epsilon$ chance of accruing reward close to $H-h$ for the remaining steps. For small enough $\epsilon$, we show this yields {cumulative} reward at least about $k \epsilon \binom{H-1}{2} \approx k \epsilon H^2/2$.

We then construct an alternative prior $\theta'$ which places zero probability that arm $N+1$ has the greatest reward, but otherwise coincides with $P_{\theta}$. Thus, $\tvarg{P_\theta}{P_{\theta'}} =\epsilon$, but $\kshot(\theta')$ fails to sample arm $N+1$, and misses out on the additional {information about which arm is optimal}. Without this information, $\kshot(\theta')$ makes random guesses at the best arm $\abar \in [N]$, and accumulates close to zero reward (in expectation) provided $N$ is sufficiently large. Naively, this argument would require $N$ to grow with $1/k\epsilon$. By coupling the behavior of $\kshot(\theta)$ and $\kshot(\theta')$, we only require $N$ to scale with $1/H$.
\end{proof}

\subsection{Proof of \Cref{thm:an_lb} \label{ssec:thm_an_lb}}
By rescaling, we assume without loss of generality that $B = 1$.
Fix a parameter $\delta < 2^{-5}$ to be chosen later, and consider $N+1$ arms. We assume that the rewards are \emph{noiseless}; that is, $r_{h,a} = \mu_a$ with probability 1.

Let $\theta$ denote the prior where an arm 
$\abar \in [N]$ and a binary random variable $\bbar \in \{0,1\}$ are drawn such that
\begin{align*}
\abar \unifsim [N], \quad \bbar \sim \mathrm{Bernoulli}(\epsilon), \quad \ahat \perp b.
\end{align*}
Given $(\abar,\bbar)$, the mean $\boldmu = \boldmu(\abar,\bbar)$, where we define
\begin{align*}
\boldmu(\abar,\bbar)\big{|}_{a} = \begin{cases} 1 - \delta & a = \abar\\
\delta & a \in  [N] \setminus \{\abar\} \\
\delta\frac{\abar}{2N} & a = N+1 \text{ and } b = 0 \\
1  & a= N+1 \text{ and } b= 1.
\end{cases}
\end{align*}

We make the following observations:
\begin{itemize}
	\item $\mu_{N+1}$ uniquely determines the best arm. Indeed, if $\mu_{N+1}\le \delta/2$, then the best arm is $\abar$, which can be recovered by setting $\abar = \frac{2N}{\delta} \cdot \mu_{N+1}$. Otherwise, $\mu_{N+1} = 1$, and it is the best arm. Hence, given any trajectory containing $a_h = N+1$, there is a unique best arm under the posterior. 
	\item Given any trajectory $\traj_h$ which \emph{does not} contain $a_h = N+1$, no information is communicated about the Bernoulli variable $b$. Moreover, if $\traj_h$ also does not include $a_h = \abar$, then $\abar$ is uniform on $[N] \setminus \{a_1,\dots,a_h\}$.
\end{itemize}
Using these facts, we derive an implementation for $\kshot(\theta)$ in \Cref{alg:top_k_lb_theta}. 
\begin{claim}
The pseudocode given by \Cref{alg:top_k_lb_theta} is a valid implementation of $\kshot(\theta)$.
\end{claim}
\begin{proof} Consider any trajectory $\traj_{h-1}$. If $\traj_{h-1}$ contains any action $a_{h'} = N+1$, $h' < h$, then as noted above, the best action is given by the selection in \Cref{line:Nplusone_played}. Otherwise, $\traj_{h-1}$ only contains actions $a \in [N]$. There are now two cases:
\begin{itemize}
	\item\emph{Case 1.} $\traj_{h-1}$ contains a {reward} $r_{h'} = 1-\delta$. We continue to let $h'$ denote this time step. In this case, $a_{h'}$ must be the index $\abar$, and $\abar$ is yields maximal reward over all actions $a \in [N]$, and thus the posterior on rewards for arms $a \in [N]$ satisfies that $\mu_{a_{h'}} = 1-\delta$ and $\mu_a = \delta$ for $a \in [N] \setminus \{a_{h'}\}$. On the other hand, because $\traj_{h-1}$ only has actions $a \in [N]$, the posterior on $b$ is still Bernoulli with parameter $\epsilon$. Hence, a draw $\hat{\boldmu} \sim P_{\theta}[\cdot \mid \traj_{h-1}]$ has distribution $\boldmu(a_{h'},\bhat)$, where $\bhat$ is uniform Bernoulli. If $\bhat = 1$, then $\max_{a} \boldmu(a_{h'},\bhat) \big{|}_a = 1$ for $a = N+1$; otherwise, $\max_{a} \boldmu(a_{h'},\bhat) \big{|}_a = 1 - \delta$, attained by $a = a_{h'}$. Hence, the update rules in \Cref{line:b_high,line:best_in_N} are equivalent to $\kshot(\theta)$.
	\item \emph{Case 2.} $\traj_{h-1}$ contains no action $r_{h'} = 1-\delta$. Hence, it can be ruled out that the top action is not among $a_1,\dots,a_{h-1}$, and thus $(\abar,\bbar) \sim P_{\theta}[\cdot \mid \traj_{h-1}]$ are distributed independently as $\abar \unifsim [N] \setminus \{a_1,\dots,a_{h-1}\}$ and $\bbar \sim \mathrm{Bernoulli}(\epsilon)$. Thus, the update rules in \Cref{line:b_high,line:sample_remainder} correctly execute $\kshot(\theta)$. 
\end{itemize}
\end{proof}
\begin{algorithm}
\caption{$\kshot$ sampling under $\theta$ (lower bound construction)}\label{alg:top_k_lb_theta}
\begin{algorithmic}[1]
    \For{$h=1, \ldots, H$}
    \State{}{Sample} $\bhat_{h,1},\dots,\bhat_{h,k} \sim \mathrm{Bernoulli}(\epsilon)$.
    \If{there exists $h' < h$ with $a_{h'} = N+1$}
    \State{} Select action \algcomment{having played $a = N+1$ in the past gives away the best arm} \label{line:Nplusone_played}
    \begin{align*}
    a_h = \begin{cases} N+1 & r_{h'} = 1\\
    \frac{2N}{\delta} r_{h'} & r_{h'} \ne 1.
    \end{cases}
    \end{align*}
    \ElsIf{$\max_{i \in [k]} b_{h,i} = 1$}
    \State{}Select action $a_h = N+1$ \algcomment{$\argmax_{a} \max_{i}\boldmu(\bar{a},b_{h,i})\big{|}_a = N+1$}, for any reference action $\bar{a} \in [N]$. %\akcomment{The comment seems wrong as $i$ should iterate over posterior samples.}
    \label{line:b_high} 
    \ElsIf{there exists $h' < h$ such that $r_{h'} = 1 - \delta$}
    \State{}Select action $a_h = a_{h'}$, where $h'$ has $r_{h'} = 1 - \delta$.  \label{line:best_in_N}
    \Else
    \State{}{Sample} $\ahat_{h,1},\dots,\ahat_{h,k} \unifsim [N] \setminus \{a_1,\dots,a_{h-1}\}$
    \State{}Select $a_h$ be any element of $\{\ahat_{h,1},\dots,\ahat_{h,k}\}$. \label{line:sample_remainder}
    \EndIf
    \EndFor
\end{algorithmic}
\end{algorithm}

\newcommand{\calEbar}{\bar{\calE}}
\paragraph{Alternative Instance}
We now construct an alternative instance $\theta'$ by agreeing with $\theta$ but always setting $b = 0$:
\begin{align*}
\boldmu = \boldmu(\abar,0), \quad \abar \unifsim [N].
\end{align*}
It is clear that $\tvarg{\theta}{\theta'} = \epsilon$, because the two differ only in the coin-flip of $b$. Under $\theta'$, $a=N+1$ never has the largest reward and is therefore never sampled. Therefore, defining the event $\calEbar = \{\exists h \in [H] : a_h = N+1\}$, and its complement
\begin{align*}
\calEbar^c := \{a_h \ne N+1 \, \forall h \in [H]\}, 
\end{align*} 
we see that
\begin{align*}
E_{\theta,\kshot(\theta)}\left[\sum_{h=1}^H r_h \mid \calEbar^c\right] = R(\theta,\kshot(\theta')).
\end{align*}
That is, the conditional expected reward garnered by well-specified $\kshot(\theta)$ under the event that $a_h \ne N+1$ for all $h$ is equal to the expected reward of misspecified $\kshot(\theta')$. 
\paragraph{Comparing the instances}
To compare the instances, write out  $\calE_h := \{a_{h'} \ne N+1, \quad \forall h' < h \text{ and } a_h = 1\}$. Then, $\calEbar$ is the disjoint union of $\calE_1,\dots,\calE_H$, so that
\begin{align*}
R(\theta,\kshot(\theta))  &= \sum_{h=1}^H P_{\theta,\kshot(\theta)}[\calE_h] \cdot E_{\theta,\kshot(\theta)}\left[ \sum_{h=1}^H r_h \mid \calE_h\right] \\
&\quad+ P_{\theta,\kshot(\theta)}[\calEbar^c] \cdot E_{\theta,\kshot(\theta)}\left[ \sum_{h=1}^H r_h \mid \calEbar^c\right].
\end{align*}
As noted above, $E_{\theta,\kshot(\theta)}[ \sum_{h=1}^H r_h \mid \calEbar^c] = R(\theta,\kshot(\theta'))$. If $\calE_h$ occurs, then the best action $a_h$ is identified, and reward at least $1 - \delta$ is accrued on stages $h' > h$. Introducing the shorthand $P[\calE_h] =  P_{\theta,\kshot(\theta)}[\calE_h]$ and similarly for $\calEbar,\calEbar^c$, we then find 
\begin{align*}
R(\theta,\kshot(\theta))  \ge (1-\delta)\sum_{h=1}^H P[\calE_h](H- h) + P[\calEbar^c] R(\theta,\kshot(\theta')). 
\end{align*}
Therefore, subtracting $R(\theta,\kshot(\theta'))$ from both sides,
\begin{align*}
R(\theta,\kshot(\theta)) - R(\theta,\kshot(\theta')) \ge (1-\delta)\sum_{h=1}^H P[\calE_h](H- h) - P[\calEbar] R(\theta,\kshot(\theta')).
\end{align*}
Let us continue simplifying the above two terms. First, since all draws $b_{h,i}$ are i.i.d. Bernoulli with parameter $\epsilon \le 1/kH$, we have
\begin{align*}
P[\calE_h] &= P[a_{h'} \ne N+1,~\forall h' < h] \cdot P[a_h = N+1 \mid a_{h'} \ne N+1,~\forall h' < h]\\
&= P[\max_i b_{h',i} = 0, \quad \forall h' < h] \cdot P[\max_{i} b_{h,i} = 1]\\
&\ge \left(1- \sum_{h'=1}^{h-1}\sum_{i=1}^kP[ b_{h',i} = 1]\right) \cdot P[\max_{i} b_{h,i} = 1]\\
&=(1 - k(h-1) \epsilon)(1 - (1-\epsilon)^k)\\
&\ge (1 - k(h-1) \epsilon)(1 - e^{- k \epsilon})\\
&\ge (1 - k(h-1) \epsilon)(k \epsilon - \frac{(k\epsilon)^2}{2})\\
&\ge (1 - kH\epsilon )(1 - k \epsilon)(k \epsilon)\\
&\ge (1 - kH \epsilon)^2(k \epsilon).
\end{align*}
Thus, 
\begin{align}
(1-\delta)\sum_{h=1}^H P[\calE_h](H- h) \ge (1 - kH \epsilon)^2 (1-\delta)\cdot k\epsilon\sum_{h=1}^H (H-h) = (1 - kH \epsilon)^2(1-\delta) k \epsilon \binom{H-1}{2}.  \label{eq:lb_first_term}
\end{align}
On the other hand, by a union bound, we have $P[\calEbar]  \le kH \epsilon$. Moreover, we have
\begin{align*}
R(\theta,\kshot(\theta')) \le HP_{\theta,\kshot(\theta')}[\exists h: a_h = \abar] + H\delta,
\end{align*}
since under $\kshot(\theta')$, if $a_h \ne \abar$ for all $h$, then $\kshot(\theta')$ always selects actions $a \in [N] \setminus \abar$, all of which return reward $\delta$. On the other hand, 
\begin{align*}
P_{\theta,\kshot(\theta')}[\exists h: a_h = \abar] &= \sum_{h=1}^H P_{\theta,\kshot(\theta')}[a_h  = \abar \mid a_{1:h-1} \ne \abar]P_{\theta,\kshot(\theta')}[a_{1:h-1} \ne \abar]\\
&\le \sum_{h=1}^H P_{\theta,\kshot(\theta')}[a_h  = \abar \mid a_{1:h-1} \ne \abar]\\
&= \sum_{h=1}^H \frac{1}{N-(h-1)} \le \frac{H}{N-H}.
\end{align*}
Therefore, 
\begin{align*}
R(\theta,\kshot(\theta')) \le \frac{H^2}{N-H}+ H\delta, \quad P[\calEbar] R(\theta,\kshot(\theta')) \le kH^2 \epsilon\left(\frac{H}{N-H}+ \delta\right).
\end{align*}
Hence, combining with \Cref{eq:lb_first_term}, we conclude
\begin{align*}
R(\theta,\kshot(\theta)) - R(\theta,\kshot(\theta')) &\ge (1-\delta)\sum_{h=1}^H P[\calE_h](H- h) - P[\calEbar] R(\theta,\kshot(\theta')).\\
 &\ge k \epsilon H^2 \left((1 - kH \epsilon)^2(1-\delta)  \frac{(H-1)(H-2)}{2H^2} - \frac{H}{N-H} - \delta \right).
\end{align*}
By tuning the above bound, we can see that there is a universal constant $\cnot$ such that, for any $\eta \in (0,1/4)$, taking $\epsilon^{-1} \ge \frac{\eta}{\cnot kH}$, $H\ge \frac{\cnot}{\eta}$ and $N+1 = H\ceil{\frac{\cnot}{\eta}}$ and $\delta = \frac{\eta}{\calE}$ ensures that 
\begin{align*}
R(\theta,\kshot(\theta)) - R(\theta,\kshot(\theta')) \ge \left(\frac{1}{2} - \eta\right)k \epsilon H^2 .
\end{align*}

\subsection{A simple converse to \Cref{prop:prior_sensitivity_tv_bound} \label{sec:prior_sens_lb_tv}}
\begin{proposition}[Lower Bound]
\label{prop:tv_lower_bound}
Let $H, k \geq 1$ be given. Then there exists a pair of priors, $\theta$ and $\theta'$ such that
\begin{align*} \tvarg{P_{\theta,\kshot(\theta)}(\boldmu,\traj_H)}{P_{\theta,\kshot(\theta')}(\boldmu,\traj_H)} \geq \frac{kH}{2} \tvarg{P_\theta}{P_{\theta'}}. \end{align*}
In particular, since Thompson sampling is $\kshot$ for $k=1$, we have
\begin{align*} \tvarg{P_{\theta,\TS(\theta)}(\boldmu,\traj_H)}{P_{\theta,\TS(\theta')}(\boldmu,\traj_H)} \geq \frac{H}{2} \tvarg{P_\theta}{P_{\theta'}}.  \end{align*}
Moreover, the rewards are Bernoulli.
\end{proposition}
\begin{proof}
Recall the $\kshot(\theta)$ selection rule at time $h$:
\begin{enumerate}
    \item Sample means $\boldmu^{(1)}, \ldots, \boldmu^{(k)}$ from the posterior $P_\theta[\cdot \mid \tau_{h-1}]$.
    \item Select action
    \[ a_h \in \argmax_{a \in \calA} \max \{ \mu_a^{(1)}, \ldots, \mu_a^{(k)} \} .\]
\end{enumerate}
Now we will show the lower bound for the following two prior distributions. 
\begin{itemize}
    \item $P_\theta$ places all its probability mass on the mean vector $(1/2, 0)$.
    \item $P_{\theta'}$ places $1-\epsilon$ of its probability mass on the mean vector $(1/2, 0)$ and $\epsilon$ of its probability mass on the mean vector $(1/2, 1)$.
\end{itemize}
%\[ \tvarg{P_{\theta,\kshot(\theta)}(\boldmu,\traj_H)}{P_{\theta,\kshot(\theta')}(\boldmu,\traj_H)} \geq \frac{kH}{2} \tvarg{P_\theta}{P_{\theta'}}. \]
Clearly we have $\TVof{P_\theta, P_{\theta'}} = \epsilon$. Moreover, we have the following three observations.
\begin{itemize}
    \item[(a)] $\kshot(\theta)$ will always pull arm 1. Thus, to give a lower bound on the total variation distance between the behavior of $\kshot(\theta)$ and $\kshot(\theta')$, it suffices to lower bound the probability that $\kshot(\theta')$ pulls arm $2$ in the course of $H$ interactions.
    \item[(b)] The posterior distribution under $\theta'$ remains unchanged when arm 1 is pulled. Thus, the probability that $\kshot(\theta')$ never pulls arm $2$ in the course of $H$ interactions is the $H$-fold product of the probability that $\kshot(\theta')$ does not pull arm 2 in one round of interaction.
    \item[(c)] The probability that $\kshot(\theta')$ does not pull arm 2 in one round of interaction is exactly the probability that $k$ i.i.d. draws from $\theta'$ does not yield an instance of $(1/2, 1)$, i.e. $(1-\epsilon)^k$. 
\end{itemize}
Combining (a), (b), and (c), and assuming $\epsilon \leq \frac{1}{2Hk}$,
\begin{align*}
\tvarg{P_{\theta,\kshot(\theta)}(\boldmu,\traj_H)}{P_{\theta,\kshot(\theta')}(\boldmu,\traj_H)} 
&\geq 1 - (1-\epsilon)^{Hk} \\
&\geq 1 - e^{-\epsilon H k} \\
&\geq \epsilon H k - (\epsilon H k)^2 \\
&\geq \frac{1}{2} \epsilon H k.
\end{align*}
In the above we have used the inequalities $1 + x \leq e^{x}$ for all $x\in \R$ and $e^{-x} \leq 1 - x+ x^2$ for $x \in [0,1]$, and the assumption that $\epsilon \leq \frac{1}{2Hk}$. 
\end{proof}

%% file: appendix/app_bayes_cdps.tex
%!TEX root = ../arxiv_main.tex
\section{General Bayesian Decision-Making \label{app:gen_bayes_dec}}

\subsection{POMDP Formalism and Special Cases}
We begin by presenting a general formalism for Bayesian POMDPs and listing some illustrative examples. For a more thorough introduction to Bayesian reinforcement learning, we direct the reader to \cite{ghavamzadeh2016bayesian}.

\paragraph{Bayesian POMDP} In a Bayesian POMDP, the priors $\{P_{\theta}:\theta \in \Theta\}$ are distributions over POMDP \emph{environments} $\boldphi \in \Phi$ with (possibly unobserved) \emph{states} $s_h \in \mathcal{S}$,  \emph{observations} $y_h \in \mathcal{Y}$, \emph{actions} $a_h \in \mathcal{A}$, and \emph{rewards} $r_h \in \R$. Common to all POMDP environments are (possibly time-varying) transition functions 
\begin{align*}
\sfP_h: \Phi \times  \mathcal{A} \times \mathcal{S}\times \mathcal{Y} \to \Delta(\mathcal{S}\times \mathcal{Y} \times \R)
\end{align*} 
which induce distributions 
\begin{align*}
\sfP_h(s_{h+1}, y_{h+1}, r_h \mid a_h, s_h, y_h,\boldphi).
\end{align*} There is also an initial distribution $\sfP_0 : \Phi \to \Delta(\mathcal{S} \times \mathcal{Y})$ giving an initial distribution of $(s_1,y_1) \sim \sfP_0(\cdot \mid \boldphi)$. The relevant definition of trajectories revealed to the learner are:
\begin{align} 
\traj_h := (a_1,y_1,r_1,\dots,a_h,y_h,r_h,y_{h+1}) \label{eq:gen_trajectory}
\end{align} 

\paragraph{POMDP Algorithms} A randomized POMDP algorithm is formally identical to a bandit one. A family of randomized POMDP algorithms $\alg(\theta)$ is a specified by a distribution $\mathcal{D}_{\mathrm{seed}}$ (independent of $\theta$), a domain $\Xi$ over random seeds $\boldxi$, and step-wise mappings $f_{1},\dots,f_H$ from trajectories, the random seed, and parameters $\theta$ to distributions over actions:
\begin{align*}
f_h(\tau_{h-1},\boldxi \mid \theta): \{h\text{-trajectories}\} \times \Xi \times \Theta \to \Delta(\calA).
\end{align*}
Each $\alg(\theta)$ operates as follows:
\begin{itemize}
    \item $\boldxi$ is drawn from $\Dseed$ at the start of the episode before interaction.
    \item At each step $h$, $a_h$ is chosen independently according to $a_h \sim f_h(\tau_{h-1},\boldxi \mid \theta)$. 
\end{itemize}
Note again the \emph{two} sources of randomness: the draw of $a_h$ and the initial random seed $\boldxi$; see \Cref{rem:randomness} for details. In short: the types of algorithms we allow for not only include those that are implemented only with the randomness in the choice of $a_h$ but also those that use initial random seeds $\boldxi$ to induce correlations across steps $h$. 

\paragraph{Interaction Protocol}
\begin{itemize}
    \item $\boldxi$ is drawn from $\Dseed(\cdot)$ at the start of the episode before interaction. 
    \item Then, an environment $\boldphi$ is drawn from $P_{\theta}$ (independent of $\boldxi$)
    \item An initial state and observation are drawn as $(s_1,y_1) \sim \sfP_0(s,y \mid \boldphi)$, and $\traj_0 = (y_1)$ is revealed to the learner. 
	\item Subsequently, for all $h \in \{1,2,\dots,H\}$,
	\begin{enumerate}
	\item The learner selects action  $a_h$ with independent randomness via $a_h \sim f_{h-1}(\tau_{h-1},\boldxi \mid \theta)$. 
	\item The environment draws a state, observation, and reward 
	\begin{align*}
	(s_{h+1}, y_{h+1}, r_h) \sim \sfP_h( s,y,r \mid a_h, s_h, y_h, \boldphi)
	\end{align*}
	\item The agent observes reward $r_h$ and observation $y_{h+1}$. The triple $(r_h,a_h,y_{h+1})$ is then appended to $\traj_{h-1}$ to form $\traj_h$.
\end{enumerate}
\end{itemize}
As in the bandit case, the reward accrued by $\alg$ is
\begin{align*}
R(\theta,\alg) := E_{\theta,\alg}\left[\sum_{h=1}^H r_h\right],
\end{align*} 
where $E_{\theta,\alg}$ denotes expectations under $\boldphi \sim P_{\theta}$, the transitions $\sfP_h(\cdot \mid \cdot,\boldphi)$, and the choice of actions $a_h$ as above.

\subsection{Special Cases of Bayesian POMDPs}
The Bayesian POMDP set up encompasses a number of special cases:

\paragraph{Mean-Parametrized Bayesian Bandits: } The Bayesian-Bandit setting considered in the main body can be viewed as a POMDP with no state, no observation, and where the instance $\boldphi$ is summarized by the mean parameter $\boldmu$. The only randomness after $\boldmu$ is selected is the generation of rewards, that is $r_h \sim \sfP(r \mid a_h,\boldmu)$, which is equivalent to the distribution $\calD(\boldmu)$ described in the main text. For example, $\sfP(r \mid a_h = a, \boldmu) = \mathcal{N}(\mu_a, \sigma^2)$ for some fixed $\sigma^2 > 0$.

Note that the distribution over mean vectors $\boldmu$ may arise to form means with, for example, linear structure (e.g.\cite{abeille2017linear}).  For example, consider an instance where each action $a$ corrsponds to a vector $\mathbf{v}_a$, and each $\boldmu$ to a vector $\mathbf{w}_{\boldmu}$ drawn from a prior, say, $\mathcal{N}(0, \Sigma_{\theta})$, for which
\begin{align*}
\boldmu_a = \langle \mathbf{v}_a, \mathbf{w}_{\boldmu} \rangle.
\end{align*}

\paragraph{General-Reward Bayesian Bandits:} More generally, we could consider Bayesian instances where the prior $P_{\theta}$ over models $\boldphi$ governs not only the reward means $\boldmu$ but can encode general conditional distributions of rewards. For example, we may have priors over mean-variance vector pairs $\boldphi = (\boldmu,\boldsymbol{\sigma}) \in \R^{2|\calA|}$, and conditional rewards $\sfP(r\mid a_h = a, \boldphi) = \mathcal{N}(\mu_a, \sigma_a^2)$. 

\paragraph{Mean-Parametrized and General Contextual Bandits:} Bayesian POMDPs also capture the contextual bandits formalism. Here, contexts $x_h \in \calX$, are directly revealed to the learner and correspond to both states and observations (i.e. $x_h := y_h = s_h$ ), and are drawn i.i.d. from a law $\sfP_{\mathrm{context}}(x \in  \cdot \mid \boldphi)$. Then, the distribution of rewards is selected depending on the context $\sfP_{\mathrm{reward}}(r_h  \mid x_h, a, \boldphi)$. In other words, the transition distribution $\sfP(x_{h+1}, r_h \mid s_h,a,\boldphi)$ is the product distribution of $\sfP_{\mathrm{context}}(x  \mid \boldphi)$ and $\sfP_{\mathrm{reward}}( r_h \mid x_h,a,\boldphi)$. {Note that the next context $x_{h+1}$ is independent of all other randomness given the instance $\boldphi$, so the dynamics are trivial.}

For example, we might have that contexts are vectors $\mathbf{x} \in \mathcal{R}^d$ (bolded to denote that they are vectors), actions are identified with vectors $\mathbf{v}_a \in \R^{n}$, $\boldphi = (\boldsymbol{\Sigma}_{\boldphi},\mathbf{L}_{\boldphi}) \in \R^{d\times d} \times \R^{n \times d}$, and that contexts and rewards are drawn 
\begin{align*}
\mathbf{x}_h \iidsim \mathcal{N}(0,\boldsymbol{\Sigma}_{\boldphi}), \quad \sfP(r_h \mid \mathbf{s}_h , a_h = a) = \mathcal{N}\left(\mathbf{v}_a^\top \mathbf{L}_{\boldphi} \mathbf{x}_h,I\right).
\end{align*}
The above is an example of Bayesian linear contextual bandits. 

The special case of contextual bandits studied in the main text are where the model is parameterized by the mean ($\boldphi = \boldmu$) and the distribution over contexts does not depend on the model: formally, contexts are drawn $x_h \sim \calD_x$ which does not depend on the realized model $\boldphi$, and where (as in reward-parametrized bandits) rewards are drawn as $r_h \sim \calD(\boldmu,x_h)$ for some $\calD: \R^{\calA} \times \calX \to \Delta(\R)$.

\paragraph{Bayesian MDPs} One final case is that of the Bayesian MDP, where the agent observes the state: $s_h = y_h$. Bayesian MDPs, and in particular, Bayesian Tabular MDPs have recieved extensive study \cite{osband2017posterior}. More general results were studied in \cite{gopalan2014thompson}.

\paragraph{Bayesian Control Problems} In addition, many online control settings -- notably, the online Linear Quadratic Regulator -- satisfy the Markov property, and hence are examples of Bayesian MDPs when formulated in the Bayesian setting \cite{abbasi2015bayesian,abeille2018improved}. Bayesian control Kalman Filtering and Linear Quadratic Gaussian control may also be formulated as a PODMP. 

\subsection{Formal Guarantees}
We now state the formal guarantees for the Bayesian POMDP setting, which straightforwardly specialize to the bandit decision-making setting described in the main text. 

\paragraph{Monte Carlo Property} Mirroring the bandit case, we let $P_{\theta}(\boldphi \mid \traj_h)$ denote the conditional distribution of the POMDP environment $\boldphi$ given the trajector $\traj_h$ (using the generalization of trajectories stated in \Cref{eq:gen_trajectory}), and $P_{\alg}(a_h \mid \traj_h)$ the conditional distribution of actions $a_h$ under algorithm $\alg$ given the trajectory. 

\begin{definition}[Generalized $n$-Monte Carlo] We say that a family of Bayesian POMDP algorithms $\{\alg(\theta):\theta \in \Theta\}$ satisfy the generalized $n$-Monte Carlo property if, for all possible trajectories $\traj_h$, 
\begin{align*}
\tvarg{P_{\alg(\theta)}(a_h \mid \traj_{h-1})}{P_{\alg(\theta')}(a_h \mid \traj_{h-1})} \le \tvarg{P_{\theta}(\boldphi \mid \traj_{h-1})}{P_{\theta'}(\boldphi \mid \traj_{h-1})}. 
\end{align*}
\end{definition}

\paragraph{Tail Expectations} Second, we require the relevant notion of \emph{tail expectation}. We propose a slightly different definition than the one given for Bayesian bandits, due to the fact that different algorithms may visit different states under the same POMDP environment. We now introduce the average conditional reward (ACR):
\begin{align}
\bar{r}_{H,\theta,\alg} := \frac{1}{H}\sum_{h=1}^H \mathsf{E}[ r_h  \mid a_h, s_h, y_h, \boldphi], \quad \text{where } (a_{1:H},s_{1:H},y_{1:H}) \sim P_{\theta,\alg}. \label{eq:bar_r}
\end{align}
Above, we use $\mathsf{E}[ r_h  \mid a_h ,s_h , y_h , \boldphi]$ to denote expectation over the law $\mathsf{P}(r_h \mid a_h, s_h, y_h, \boldphi)$, and note that the conditional \emph{does not} depend on $\theta$ or the specification of $\alg$.
In the special case of bandits, notice that \Cref{eq:bar_r} simplies to 
\begin{align*}
\bar{r}_{H,\theta,\alg} := \frac{1}{H}\sum_{h=1}^H \mu_{a_h}.
\end{align*}
\newcommand{\Psibar}{\bar{\Psi}}
\newcommand{\Psibarthet}{\Psibar_{\theta}}
\newcommand{\Psibarthetpr}{\Psibar_{\theta/\theta'}}

\begin{definition}[POMDP Tail expectation]\label{defn:pomdp} Given an algorithm $\alg(\cdot)$ parametrized by $\theta \in \Theta$, we define
\begin{align*}
\Psibarthet(p) = \Psi_{|\bar{r}_{H,\theta,\alg(\theta)}|}(p), \quad \Psibarthetpr(p) = \Psi_{|\bar{r}_{H,\theta,\alg(\theta')}|}(p),
\end{align*}
where $\Psi_X(p)$ is the tail expectation of a nonnegative random variable $X$ as in \Cref{defn:tail_exp_new}.
\end{definition}
Unlike the bandit tail expectations $\Psi_{\theta}$ in \Cref{defn:tail_exp_new} which depend only on the mean parameter $\boldmu$, the tail expectations above depend on both the ACR under $\theta$, as in $\Psibarthet$, and the ACR under $\theta,\alg(\theta')$ as in $\Psibarthetpr$. In particular, both terms depend on the family of algorithms $\alg(\cdot)$. This is important in POMDP environments with unbounded state spaces (e.g. linear control), where sensitivity can be quite poor if the misspecified policy $\alg(\theta')$ visits much lower-reward states under $\theta$ than the well-specified $\alg(\theta)$. Nevertheless, for bandits, $\Psibarthet$ and $\Psi_\theta$ are qualitatively similar because
\begin{align*}
|\bar{r}_{H,\theta,\alg(\theta)}| \le \sup_{a}|\mu_a|,
\end{align*}
and similarly for $\bar{r}_{H,\theta,\alg(\theta')}$. 

\paragraph{Strongly $B$-Bounded distributions. }
To interpret the tail conditions, we consider the special case of \emph{strongly} $B$ bounded distributions. 
\begin{definition}\label{defn:strongly_B_bounded}
We say that $P_{\theta}$ is \emph{strongly $B$-bounded} if, with probability $1$, $\E_{\theta}[r_h \mid a_h, \traj_{h-1},\boldphi] \in [-\frac{B}{2},\frac{B}{2}]$ conditioned on any action $a_h$, trajectory $\traj_{h-1}$ and $\boldphi$. 
\end{definition}
In the special case of bandits, strong $B$-boundedness implies that $\mu_a \in [-\frac{B}{2},\frac{B}{2}]$ with probability one, and is therefore slightly stronger than $B$-boundedness, which states that $\range(\boldmu) \le B$.  Observe that if $P_{\theta}$ is strongly $B$-bounded (\Cref{defn:strongly_B_bounded}), then 
\begin{align*}
|\bar{r}_{H,\theta,\alg(\theta)}|  \le \frac{B}{2}, \quad \text{and} \quad |-\bar{r}_{H,\theta,\alg(\theta')}| \le \frac{B}{2}.
\end{align*}
Therefore $\Psibarthet(p) + \Psibarthetpr(p) \le B$. 

\paragraph{}
We are now ready to state our general theorem, consisting of (a) a total variation bound, (b) a reward bound for strongly $B$-bounded rewards, and (c) a reward bound for general tail expectations:
\begin{theorem}\label{thm:gen_pomdp} Let $\alg(\cdot)$ satisfy the $n$-Monte Carlo property on horizon $H$, and consider two priors $\theta,\theta' \in \Theta$ with $\varepsilon = \tvarg{P_\theta}{P_{\theta'}}$
\begin{itemize}
	\item[(a)] Let $P_H = P_{\theta,\alg(\theta)}(\boldphi,\traj_H,s_{1:H+1})$ and $P_H' = P_{\theta,\alg(\theta')}(\boldphi,\traj_H,s_{1:H+1})$ denote the joint law of the environment $\boldphi$, trajectory $\traj_H$, and sequence of states $s_{1:H}$ under prior $P_{\theta,\alg(\theta)}$ and $P_{\theta,\alg(\theta')}$, respectively. Then, 
	\begin{align*}
	\tvarg{P_H}{P_H'} \le 2H\varepsilon. 
	\end{align*}
	\item[(b)] If $P_{\theta}$ is strongly $B$-bounded, then $|R(\theta,\alg(\theta)) - R(\theta,\alg(\theta'))| \le 2nBH^2 \varepsilon $.
	\item[(c)] For general tail expectations, the following bound holds:
	\begin{align*}
	|R(\theta,\alg(\theta)) - R(\theta,\alg(\theta'))| \le 2nH^2 \varepsilon \cdot \left(\Psibarthet(2H \varepsilon) + \Psibarthetpr(2H\varepsilon)\right).
	\end{align*}
\end{itemize}
\end{theorem}
Note that part (a) of the above theorem generalizes \Cref{prop:prior_sensitivity_tv_bound}, part (b) generalizes \Cref{thm:n_montecarlo_body}(a). 
\newcommand{\trajtil}{\tilde{\traj}}
\begin{proof}[Proof of \Cref{thm:gen_pomdp}] 
We begin by establishing part (a), from which parts (b) and (c) follow.
\paragraph{Part a.} Let us start by developing an analogoue of \Cref{lem:perf_diff_bandits}. To do so, we invoke \Cref{lem:performance_difference_decomp}, replacing $\boldmu$ with the POMDP environment denoted by $\boldphi$, and with $P_H = P_{\theta,\alg(\theta)}$ and $P_H' =P_{\theta,\alg(\theta')}$, and with augmented trajectories $\trajtil_{h-1} = (\traj_{h-1},s_{1:h})$ To invoke the latter lemma, we need to check three conditions.
\begin{enumerate}
	\item \emph{Condition 1: $\trajtil_{h-1}$ is a deterministic function of $\trajtil_h$.} This is definitionally true, even in the PODMP setting.
	\item \emph{Condition 2: The conditional distributions of $\traj_h$ given $a_h,\traj_{h-1}$ and $\boldphi$ are the same: $P_H'(\traj_h \mid  a_h, \trajtil_{h-1},\boldphi) = P_H(\trajtil_h \mid a_h,\traj_{h-1},\boldphi)$.} This follows because $P(\traj_h \mid  a_h, \trajtil_{h-1},\boldphi) = \sfP(s_{h+1},y_{h+1},r_h \mid a_h, s_h, y_h, \boldphi)$ for $P \in \{P_H,P_H'\}$, were $\sfP$ is the transition function.
	\item \emph{Condition 3: Under both $P_H$ and $P_H'$, $a_h$ is independent of $\boldphi$ given $\trajtil_{h-1}$.} Following the same logic as in \Cref{lem:perf_diff_bandits} and letting $P = P_{H}$, we have
\begin{align*}
    P(a_h, \boldphi \mid \traj_{h-1})
    &= \E\left[ \E\left[ P(a_h, \boldphi \mid \trajtil_{h-1}, \boldxi) \mid \boldxi \right] \mid \trajtil_{h-1}\right] \\
    &= \E\left[ \E\left[ f_h(a_h \mid \traj_{h-1}, \boldxi) P(\boldphi \mid \trajtil_{h-1}, \boldxi) \mid \boldxi \right] \mid \trajtil_{h-1}\right] \\
    &= \E\left[ \E\left[ f_h(a_h \mid \traj_{h-1}, \boldxi) \mid \boldxi \right] P(\boldphi \mid \trajtil_{h-1}) \mid \trajtil_{h-1}\right] \\
    &= P(a_h \mid \trajtil_{h-1}) P(\boldphi \mid \trajtil_{h-1})
\end{align*}
where the second equality follows from the fact that $a_h \sim f_h(\cdot \mid \boldxi, \traj_{h-1})$ and the third line follows from the fact that $\boldphi$ is independent of $\boldxi$ conditioned on $\tau_{h-1}$. The same argument holds symmetrically for $P' = P'_{H}$. 
\end{enumerate}
As a consequence of these three conditions, it holds that for $P_H = P_{\theta,\alg(\theta)}(\boldmu,\traj_H,s_{1:H+1})$ and  $P_H' = P_{\theta,\alg(\theta')}(\boldmu,\traj_H,s_{1:H+1})$,
\begin{align*}
\tvarg{P_H}{P_H'} &\le \sum_{h=1}^H \E_{\trajtil_{h-1}\sim P_H}\tvarg{P_{H}(a_h \mid \trajtil_{h-1})}{P_{H}'(a_h \mid \trajtil_{h-1})}\\
&= \sum_{h=1}^H \E_{\traj_{h-1}\sim P_H}\tvarg{P_{H}(a_h \mid \traj_{h-1})}{P_{H}'(a_h \mid \traj_{h-1})}.
\end{align*}
where we use the fact that $P_{H}(a_h \mid \trajtil_{h-1})=P_{H}(a_h \mid \traj_{h-1})$, and similarly under $P_H'$. Next, using the $n$-Monte Carlo Property, and that $P_{H}(a_h \mid \traj_{h-1}) = P_{\alg(\theta)}(a_h \mid \traj_{h-1})$ (and similarly for $P_H'$ and $P_{\alg(\theta')}$), 
\begin{align*}
\tvarg{P_H}{P_H'}  \le n\sum_{h=1}^H \tvarg{P_{\theta}(\boldphi \mid \traj_{h-1})}{P_{\theta'}(\boldphi \mid \traj_{h-1})}.
\end{align*}
By invoking the de-conditioning lemma, \Cref{lem:fundamental_deconditioning}, with $Q = P_{\theta,\alg(\theta)}$, $Q' = P_{\theta',\alg(\theta')}$, $X = \traj_{h-1}$ and $Y = \boldphi$, we have that
\begin{align*}
\tvarg{P_H}{P_H'}  \le n\sum_{h=1}^H 2\tvarg{Q(Y)}{Q'(Y)} = 2Hn\tvarg{P_{\theta}}{P_{\theta'}}.
\end{align*}
This establishes part (a). 

\paragraph{Parts b and c.} Part (b) is a consequence of part (c) and the fact that, for strongly $B$-bounded $P_\theta$, $\Psibarthet(\cdot) + \Psibarthetpr(\cdot) \le B$. We conclude by proving part (c), and keep the notation $\trajtil_{h} = (\traj_h,s_{1:h+1})$. Let $P,E$ and $P',E'$ denote probability and expectation operators under $P_{\theta,\alg(\theta)}$ and $P_{\theta,\alg(\theta')}$, respectively. Introduce the conditional reward function 
\begin{align*}
\bar{\mu}_h(a,s,y,\boldphi) =\E[ r_h  \mid a_h = a, s_h = s, y_h = y, \boldphi].
\end{align*}
By the tower rule, 
\begin{align*}
R(\theta,\alg(\theta)) - R(\theta,\alg(\theta')) = E\left[\sum_{h=1}^H \bar{\mu}_h(a_h,s_h,y_h,\boldphi)\right] - E'\left[\sum_{h=1}^H \bar{\mu}_h(a_h,s_h,y_h,\boldphi)\right].
\end{align*}
By \Cref{lem:TV_transport} and the fact that $P(\boldphi) = P'(\boldphi)$, there exists a coupling $Q$ such over random variables $(\boldphi,\trajtil_H,\trajtil_H')$ such that both
\begin{align*}
Q(\boldphi,\trajtil_H) = P(\boldphi,\trajtil_H), \quad  Q(\boldphi,\trajtil_H') = P'(\boldphi,\trajtil_H),
\end{align*}
and
\begin{align*}
Q[\trajtil_H \ne \trajtil_H'] = \tvarg{P(\boldphi,\trajtil_H)}{P'(\boldphi,\trajtil_H)}.% \le \delta := 2H \varepsilon,
\end{align*}
Letting $(a_h,s_h,y_h)$ and $(a_h',s_h',y_h')$ denote states and actions corresponding to $\trajtil_H$ and $\trajtil_H'$, and let $E_Q$ denote expectations under the coupling $Q$, we then have
\begin{align*}
&R(\theta,\alg(\theta)) - R(\theta,\alg(\theta')) \\
&\quad= E_Q\left[\sum_{h=1}^H \bar{\mu}_h(a_h,s_h,y_h,\boldphi) - \bar{\mu}_h(a_h',s_h',y_h',\boldphi)\right]\\
&\quad\overset{(i)}{=} E_Q\left[\I\{\trajtil_H \ne \trajtil_H'\}\left(\sum_{h=1}^H \bar{\mu}_h(a_h,s_h,y_h,\boldphi) - \bar{\mu}_h(a_h',s_h',y_h',\boldphi)\right)\right]\\
&\quad= H \left(E_Q\left[\underbrace{\I\{\trajtil_H \ne \trajtil_H'\}}_{:=Y}\cdot \underbrace{\frac{1}{H}\sum_{h=1}^H \bar{\mu}_h(a_h,s_h,y_h,\boldphi)}_{:=X}\right] - E_Q\left[\underbrace{\I\{\trajtil_H \ne \trajtil_H'\}}_{:=Y}\cdot\underbrace{\frac{1}{H}\sum_{h=1}^H \bar{\mu}_h(a_h',s_h',y_h',\boldphi)}_{:=X'}\right]\right)%\\
%&\quad\overset{(ii)}{=} H \left( E_Q\left[ \I\{\trajtil_H \ne \trajtil_H'\} X\right] - H E_Q\left[\I\{\trajtil_H \ne \trajtil_H'\}\bar{r}_{H,\theta,\alg(\theta')}\right]\right)
\end{align*}
where $(i)$ uses that $\sum_{h=1}^H \bar{\mu}_h(a_h,s_h,y_h,\boldphi) - \bar{\mu}_h(a_h',s_h',y_h',\boldphi) = 0$ whenever $\trajtil_H = \trajtil'_H$. By the triangle inequality, 
\begin{align*}
|R(\theta,\alg(\theta)) - R(\theta,\alg(\theta'))| \le H\left(\E\left[|X|\cdot Y\right] + E_Q\left[|X'|\cdot Y\right]\right).
\end{align*}
Since $|X|$ and $|X'|$ are nonnegative, and that $Y \in [0,1]$ satisfies
\begin{align*}
\E[Y] = Q[\trajtil_H \ne \trajtil_H'] = \tvarg{P(\boldphi,\trajtil_H)}{P'(\boldphi,\trajtil_H)} \le \delta := 2nH \varepsilon,
\end{align*}
where the first equality is from our choice of coupling $Q$ and the inequality follows from part (a) of the theorem. Hence, by definition of the tail expectation functional which maximizes the correlation with $[0,1]$-bounded random variables $Y$ satisfying the above constraints (\Cref{defn:tail_exp_new}),
\begin{align*}
|R(\theta,\alg(\theta)) - R(\theta,\alg(\theta'))| \le H\delta \left(\Psi_{|X|}(\delta) + \Psi_{|X'|}(\delta)\right).
\end{align*}
Finally, we observe that $X$ has the same distribution of $\bar{r}_{H,\theta,\alg(\theta)}$ and $X'$ the distribution of $\bar{r}_{H,\theta,\alg(\theta')}$. Hence, $\Psi_{|X|}(\delta) = \Psibarthet(\delta)$ and $\Psi_{|X'|}(\delta) = \Psibarthetpr(\delta)$. We therefore conclude
\begin{align*}
|R(\theta,\alg(\theta)) - R(\theta,\alg(\theta'))| &\le H\delta \left(\Psibarthet(\delta) + \Psibarthetpr(\delta)\right)\\
&= 2nH^2 \varepsilon \cdot \left(\Psibarthet(2H \varepsilon) + \Psibarthetpr(2H\varepsilon)\right).\tag*\qedhere
\end{align*}
\end{proof}
\subsection{The Monte-Carlo property in POMDPs}

We conclude this section by mentioning that, as in the bandit setting, any algorithm whose actions depend only on independent samples of environments drawn from the posterior distribution is $n$-Monte Carlo. Thus, \Cref{alg:pomdp_post_sample}, which is the POMDP generalization of \Cref{alg:general_post_sample}, is $k$-Monte Carlo, where $k$ is the number of environments sampled from the posterior. The proof of this fact is identical to the proof of \Cref{lem:samples_to_n_mc}. Similarly, along the lines of \Cref{lem:twompc_n_mc}, one can establish the Monte-Carlo property for a suitable generalizations of the $\twompc$ algorithm (\Cref{alg:two_mpc}); for brevity we omit details.

\begin{algorithm}
\caption{$(k,f_{1:H})$-Posterior Sampling ($\kfpost(\theta)$)}\label{alg:pomdp_post_sample}
\begin{algorithmic}[1]
\State{}\textbf{Input:} Prior $\theta$, sample size $k \in \N$, functions $f_1, \ldots, f_H: \R^{\Phi \times k} \rightarrow \Delta^{\calA}$. 
    \For{$h=1, \ldots, H$}
    \Statex{}\algcomment{action selection at step $h$}
    \State{}{Sample} $\boldphi^{(1)}, \ldots, \boldphi^{(k)}$ independently from the posterior $P_{\theta}[\cdot \mid \traj_{h-1}]$
    \State{}Select action $ a_h \sim f_h(\cdot \mid \boldphi^{(1)}, \ldots, \boldphi^{(k)})$.
    \EndFor
\end{algorithmic}
\end{algorithm}

%% file: appendix/app_estimation.tex
\section{Accuracy of moment estimators\label{app:estimation}}

\subsection{Beta priors and Bernoulli rewards}
\label{app:beta-binomial}

We first show how to translate sufficiently good error bounds in parameter estimation for regular exponential families into bounds on total variation error.

\begin{lemma}\label{lem:exponential_families}
  Let $\{ P_{\boldtheta} : \boldtheta \in \Theta \}$ be a standard exponential family with natural parameter space $\Theta \subset \R^p$.
  For any $\boldtheta \in \Theta$, there exist $C, c>0$ depending only on $\boldtheta$ such that
$\tvarg{P_\boldtheta}{P_{\boldtheta'}}\leq C\cdot\nrm{\boldtheta'-\boldtheta}_2$ for all $\boldtheta' \in \Theta$ satisfying $\nrm{\boldtheta'-\boldtheta}_2 \leq c$.
\end{lemma}

\begin{proof}%[Proof of Lemma~\ref{lem:exponential_families}]
  Let $A$ be the log-partition function for the exponential family, which is infinitely-differentiable on $\Theta$~\cite{brown1986fundamentals}, and let $\Bregarg{\boldtheta'}{\boldtheta}=A(\boldtheta')-A(\boldtheta)-\langle\nabla A(\boldtheta), \boldtheta'-\boldtheta\rangle$ be its corresponding Bregman divergence.
  By Pinsker's inequality and properties of Bregman divergences~\cite[Appendix A]{banerjee2005clustering}, we have
\begin{equation*}
  2 \, \tvarg{P_\boldtheta}{P_{\boldtheta'}}^2
  \leq \klarg{P_{\boldtheta}}{P_{\boldtheta'}}
  = \Bregarg{\boldtheta'}{\boldtheta}
  = A(\boldtheta')-A(\boldtheta)-\langle\nabla A(\boldtheta), \boldtheta'-\boldtheta\rangle . 
\end{equation*}
Let $g(t)=A\big(t\cdot\boldtheta'+(1-t)\cdot \boldtheta\big)$.
By Taylor's theorem, there exists $\xi\in[0,1]$ such that $A(\boldtheta')=A(\boldtheta)+\langle\nabla A(\boldtheta), \boldtheta'-\boldtheta\rangle+\frac{1}{2}(\boldtheta'-\boldtheta)^\T \nabla^2 A\big(\xi\boldtheta'+(1-\xi)\boldtheta\big)(\boldtheta'-\boldtheta)$.
We can therefore take any $C$ and $c$ such that the Hessian has eigenvalues bounded by $4C^2$ in a Euclidean ball of radius $c$ around $\boldtheta$, upon which we have $\Bregarg{\boldtheta'}{\boldtheta} \leq 2C^2\nrm{\boldtheta'-\boldtheta}_2^2$.
\end{proof}

Now we argue that the method-of-moments estimator of \cite{tripathi1994estimation} for the Beta-Binomial distribution gives accurate parameter estimates of the Beta component parameters (i.e., $\alpha$ and $\beta$) provided a large enough sample size.
The bound is given for the parameters corresponding to a single arm; applying the result for all arms $a \in \calA$ with a union bound delivers the final sample complexity claim.

\begin{lemma} \label{lemma:beta-binomial}
  Let $\hat m_1$ and $\hat m_2$ be empirical moments based on $N$ i.i.d.~draws from a Beta-Binomial distribution with parameters $(\alpha,\beta,n)$ where $n\geq 2$.
  Let $(\hat\alpha,\hat\beta)$ be the method-of-moments estimate of $(\alpha,\beta)$ obtained using
  \begin{equation*}
    \hat\alpha := \frac{n\hat m_1 - \hat m_2}{n(\frac{\hat m_2}{\hat m_1}-\hat m_1 - 1) + \hat m_1} \quad\text{and}\quad
    \hat\beta := \frac{(n-\hat m_1)(n - \frac{\hat m_2}{\hat m_1})}{n(\frac{\hat m_2}{\hat m_1}-\hat m_1 - 1) + \hat m_1} .
  \end{equation*}
  There exists a positive constant $C>0$ depending only on $(\alpha,\beta,n)$ such that for any $\epsilon >0$ and $\delta \in (0,1)$, if $N \geq C \log(1/\delta) / \epsilon^2$, then $\Pr(\max\{ |\hat\alpha - \alpha|, |\hat\beta - \beta| \} \leq \epsilon) \geq 1 - \delta$.
\end{lemma}

\begin{proof}%[Proof of Proposition~\ref{lemma:beta-binomial}]
  Let $X_1,\dotsc,X_N$ denote an i.i.d.~sample from the Beta-Binomial distribution with parameters $(\alpha,\beta,n)$, and let $m_1 := \E[X_1]$ and $m_2 := \E[X_1^2]$.
  First, by Hoeffding's inequality and union bounds, with probability at least $1-\delta$, we have
  \begin{equation*}
    |\hat m_1 - m_1| \leq n \sqrt{\frac{2\ln(4/\delta)}{N}} \quad\text{and}\quad
    |\hat m_2 - m_2| \leq n^2 \sqrt{\frac{2\ln(4/\delta)}{N}} ,
  \end{equation*}
  where $\hat m_1 := \frac1N \sum_{i=1}^N X_i$ and $\hat m_2 := \frac1N \sum_{i=1}^N X_i^2$.
  Let us henceforth condition on this $1-\delta$ probability event.
  Now, treating $\hat\alpha(\hat m_1, \hat m_2)$ and $\hat\beta(\hat m_1, \hat m_2)$ as functions of $(\hat m_1, \hat m_2)$, we have by Taylor's theorem that
  \begin{align*}
    \hat\alpha(\hat m_1, \hat m_2)
    & = \hat\alpha(m_1, m_2)
    + \frac{\partial\hat\alpha}{\partial\hat m_1}(\tilde m_1) \cdot (\hat m_1 - m_1)
    + \frac{\partial\hat\alpha}{\partial\hat m_2}(\tilde m_2) \cdot (\hat m_2 - m_2)
    \\
    \hat\beta(\hat m_1, \hat m_2)
    & = \hat\beta(m_1, m_2)
    + \frac{\partial\hat\beta}{\partial\hat m_1}(\tilde m_1) \cdot (\hat m_1 - m_1)
    + \frac{\partial\hat\beta}{\partial\hat m_2}(\tilde m_2) \cdot (\hat m_2 - m_2)
  \end{align*}
  where $(\tilde m_1, \tilde m_2) = (1-\xi) (\hat m_1, \hat m_2) + \xi (m_1, m_2)$ for some $\xi \in [0,1]$.
  It can be verified using properties of the Beta-Binomial distribution that $\hat\alpha(m_1, m_2) = \alpha$ and $\hat\beta(m_1, m_2) = \beta$.
  Moreover, since the functions $\hat\alpha(\hat m_1, \hat m_2)$ and $\hat\beta(\hat m_1, \hat m_2)$ are analytic, it follows that there is a Euclidean ball of radius (say) $c'>0$ around $(m_1,m_2)$ on which the gradients of $\hat\alpha$ and $\hat\beta$ are uniformly bounded by (say) $C'>0$ in Euclidean norm.
  Here, both $C'$ and $c'$ depend only on $m_1$ and $m_2$.
  So, as long as $\sqrt{(\hat m_1 - m_1)^2 + (\hat m_2 - m_2)^2} \leq c'$, we have
  \begin{align*}
    |\hat\alpha(\hat m_1, \hat m_2) - \alpha|
    & \leq C \sqrt{(\hat m_1 - m_1)^2 + (\hat m_2 - m_2)^2}
    \\
    |\hat\beta(\hat m_1, \hat m_2) - \beta|
    & \leq C \sqrt{(\hat m_1 - m_1)^2 + (\hat m_2 - m_2)^2}
  \end{align*}
  by Cauchy-Schwarz.
  The claim now follows by choosing $N \geq C \log(1/\delta) / \epsilon^2$ for some $C$ depending only on $C'$, $c'$, and $n$.
\end{proof}

\subsection{Gaussian priors and Gaussian rewards}
\label{app:gaussian}

We directly bound the KL-divergence between two multivariate Gaussian distributions in terms of distances between their corresponding parameters.

%Recall that a random variable $X$ is \emph{$(v,c$)-subexponential} if $\ln\E\exp(\lambda (X - \E(X))) \leq v\lambda^2/2$ for all $\lambda \in \R$ such that $|\lambda|\leq1/c$.
\begin{lemma}[Gaussian KL-divergence]
  \label{lemma:gaussian-kl}
  Let $P := \Normal(\normmean, \normcov)$ and $\widehat{P} := \Normal(\normmeanest, \normcovest)$ be multivariate Gaussian distributions in $\R^{\calA}$.
  Then
  \begin{equation*}
    \klarg{\widehat{P}}{P}
    = \frac12 \left\{
      \tr(\normcov^{-1/2} \normcovest \normcov^{-1/2} - I)
      - \ln\det(\normcov^{-1/2} \normcovest \normcov^{-1/2})
      + \|\normcov^{-1/2}(\normmeanest - \normmean)\|_2^2
    \right\}
    .
  \end{equation*}
  Moreover, if
  \begin{equation*}
    \|\normcov^{-1/2}\normcovest\normcov^{-1/2} - I\|_2 \leq \frac23 ,
  \end{equation*}
  then
  \begin{equation*}
    \klarg{\widehat{P}}{P}
    \leq \frac12 \left\{ |\calA| \cdot \|\normcov^{-1/2}\normcovest\normcov^{-1/2} - I\|_2^2 + \|\normcov^{-1/2}(\normmeanest - \normmean)\|_2^2 \right\} .
  \end{equation*}
\end{lemma}
\begin{proof}%[Proof of Lemma~\ref{lemma:gaussian-kl}]
  The formula for the KL-divergence is standard.
  Now suppose that $\|\normcov^{-1/2}\normcovest\normcov^{-1/2} - I\|_2 \leq 2/3$.
  This means that all of the eigenvalues $\lambda_1,\dotsc,\lambda_K$ of $\normcov^{-1/2}\normcovest\normcov^{-1/2}$ are contained in the interval $[1/3,5/3]$.
  In this case, we have
  \begin{align*}
      \tr(\normcov^{-1/2} \normcovest \normcov^{-1/2} - I)
      - \ln\det(\normcov^{-1/2} \normcovest \normcov^{-1/2})
      & = \sum_{i=1}^K \left\{ \lambda_i - 1 \right\} - \ln \prod_{i=1}^K \lambda_i \\
      & = \sum_{i=1}^K \left\{ \lambda_i - 1 - \ln \lambda_i \right\} \\
      & \leq \sum_{i=1}^K (\lambda_i - 1)^2 \\
      & \leq K \cdot \|\normcov^{-1/2}\normcovest\normcov^{-1/2} - I\|_2^2
      ,
  \end{align*}
  where the first inequality uses the fact $\ln(1+x) \geq x - x^2$ for all $x \geq -2/3$.
  Plugging this inequality into the KL-divergence formula gives the claimed inequality.
\end{proof}

Lemma~\ref{lemma:gaussian-kl} and Pinsker's inequality imply that, to obtain an estimate of $\Normal(\normmean, \normcov)$ that is $\varepsilon$-close in total variation distance, it suffices to obtain estimates $\normmeanest$ and $\normcovest$ such that
\begin{equation*}
  \|\normcov^{-1/2}(\normmeanest - \normmean)\|_2 \leq \varepsilon , \quad
  \|\normcov^{-1/2}(\normcovest - \normcov)\normcov^{-1/2}\|_2 \leq \frac{\varepsilon}{\sqrt{|\calA|}} .
\end{equation*}
Below, we give estimators $\normmeanest$ and $\normcovest$ that satisfy these inequalities with probability at least $1-\delta$ provided that
\begin{align*}
  T & \geq C' \cdot \frac{d \cdot (|\calA|^4 + |\calA|^3 \log(1/\delta))}{\veps^2} ,
\end{align*}
where $d$ is defined in Lemma~\ref{lemma:gaussian-cov2}, and $C'$ is an absolute constant.
We note that if $\normcov$ is known and does not need to be estimated, then the requirement improves to
\begin{align*}
  T & \geq C'' \cdot \frac{d_2 \cdot (|\calA|^2 + |\calA| \log(1/\delta))}{\veps^2} ,
\end{align*}
where $d_2$ is defined in Lemma~\ref{lemma:gaussian-mean}, and $C''$ is another absolute constant.

\paragraph{Mean estimation.}

We first consider the estimate of $\normmean$.
To do so, we assume the first round in each of $T$ episodes is chosen uniformly at random from $\calA$.
In episode $t$:
\begin{enumerate}
  \item let $\boldmu_t \sim P = \Normal(\normmean,\normcov)$ denote the mean reward vector;
  \item let $a_t \sim \Uniform(\calA)$ be the action taken in the first round (independent of $\boldmu_t$);
  \item let $\boldr_t$ be the reward vector for the first round, so
  \begin{equation*}
    \boldr_t \mid (\boldmu_t,a_t) \sim \Normal(\boldmu_t, \sigma^2 \boldI) .
  \end{equation*}
\end{enumerate}
The reward observed (and accrued) in the first round of episode $t$ is $\boldr_t^\T \bolde_{a_t}$.
Our estimate of prior mean $\normmean$ is
\begin{equation}
  \normmeanest := \frac{|\calA|}{T} \sum_{t=1}^T (\boldr_t^\T \bolde_{a_t}) \bolde_{a_t} .
  \label{eq:normmeanest}
\end{equation}

\begin{lemma}[Gaussian mean estimation]
  \label{lemma:gaussian-mean}
  There exists a universal constant $C>0$ such that the following holds.
  Consider any multivariate Gaussian distribution $P := \Normal(\normmean, \normcov)$ in $\R^{\calA}$.
  Let
  \begin{equation*}
    (\boldmu_1,a_1,\boldr_1), (\boldmu_2,a_2,\boldr_2), \dotsc, (\boldmu_T,a_T,\boldr_T)
  \end{equation*}
  be $T$ iid random variables, with
  \begin{align*}
    (\boldmu_t, a_t) & \sim P \otimes \Uniform(\calA), \\
    \boldr_t \mid (\boldmu_t,a_t) & \sim \Normal(\boldmu_t, \sigma^2 \boldI) ;
  \end{align*}
  and define $\normmeanest$ as in \eqref{eq:normmeanest}.
%  \begin{equation*}
%    \normmeanest := \frac{K}{T} \sum_{t=1}^T (\boldr_t^\T \bolde_{a_t}) \bolde_{a_t} .
%  \end{equation*}
  For any $\delta \in (0,1)$, with probability at least $1-\delta$,
  \begin{equation*}
    \|\normcov^{-1/2} (\normmeanest - \normmean)\|_2
    \leq C
    \left(
      \sqrt{\frac{d_2 (|\calA|^2 + |\calA| \log(1/\delta))}{T}}
      + \frac{d_\infty (|\calA|^2 + |\calA| \log(1/\delta))}{T}
    \right)
    ,
  \end{equation*}
  where
  \begin{align*}
    d_2 & := \lambda_{\max}\left( \normcov^{-1/2} \left( \diag(\normcov) + \sigma^2 \boldI + \diag(\normmean)^2 \right) \normcov^{-1/2} \right) , \\
    d_\infty & := \max_{a \in \calA} \sqrt{(\normcov^{-1})_{a,a} ((\normcov)_{a,a} + \sigma^2 + (\normmean)_a^2)}
    .
  \end{align*}
\end{lemma}
\begin{proof}
  First, since
  \begin{equation*}
    \E\left[ (\boldr_t^\T \bolde_{a_t}) \bolde_{a_t} \right]
    = \frac1{|\calA|} \sum_{a \in \calA} \E\left[ (\boldr_t^\T \bolde_{a_t}) \bolde_{a_t} \mid a_t = a \right]
    = \frac1{|\calA|} \E\left[ \boldr_t \right]
    = \frac1{|\calA|} \E\left[ \boldmu_t \right]
    = \frac1{|\calA|} \normmean ,
  \end{equation*}
  it follows by linearity that $\E[\normmeanest] = \normmean$.
  Next, we show that for any unit vector $\boldu \in S^{|\calA|-1}$, the random variable
  \begin{align*}
    X_{\boldu,t}
    & :=
    (\normcov^{-1/2} \boldu)^\T \left( (\boldr_t^\T \bolde_{a_t}) \bolde_{a_t} - \frac1{|\calA|} \normmean \right)
    \\
    & \hphantom:=
    (\normcov^{-1/2} \boldu)^\T \bolde_{a_t} \bolde_{a_t}^\T (\boldr_t - \normmean) + (\normcov^{-1/2} \boldu)^\T \bolde_{a_t} \bolde_{a_t}^\T \normmean - \frac{(\normcov^{-1/2} \boldu)^\T \normmean}{|\calA|}
  \end{align*}
  is $(4d_2/|\calA|,2d_\infty)$-subexponential.
  Consider $\lambda \in \R$ such that $|\lambda| \leq 1/(2d_\infty)$, and let $\boldv := \normcov^{-1/2} \boldu$.
  Then
  \begin{align*}
    \lefteqn{
      \frac{\lambda^2 v_{a_t}^2}{2} \left( \bolde_{a_t}^\T \normcov \bolde_{a_t} + \sigma^2 \right)
      + \lambda v_{a_t} \bolde_{a_t}^\T \normmean
    } \\
    & = \frac{\lambda^2 (\boldu^\T \normcov^{-1/2} \bolde_{a_t})^2 (\bolde_{a_t}^\T \normcov \bolde_{a_t} + \sigma^2)}{2}
    + \lambda (\boldu^\T \normcov^{-1/2} \bolde_{a_t}) \bolde_{a_t}^\T \normmean \\
    & \leq \frac{\lambda^2 \bolde_{a_t}^\T \normcov^{-1} \bolde_{a_t} (\bolde_{a_t}^\T \normcov \bolde_{a_t} + \sigma^2)}{2}
    + |\lambda| \sqrt{\bolde_{a_t}^\T \normcov^{-1} \bolde_{a_t}} |\bolde_{a_t}^\T \normmean| \\
    & \leq \frac{\lambda^2 d_\infty^2}{2} + |\lambda| d_\infty \leq 1
  \end{align*}
  where the first inequality follows by Cauchy-Schwarz, the second inequality follows by definition of $d_\infty$, and the third inequality follows by assumption on $\lambda$.
  Further, observe that $\boldr_t$ has the same distribution as $\normmean + \normcov^{1/2} \boldx + \sigma \boldy$, where $(\boldx,\boldy) \iidsim \Normal(\bm0,\boldI)^{\otimes 2}$, independent of $a_t$.
  Since $a_t$ and $(\boldx,\boldy)$ are independent,
  \begin{align*}
    \E\exp(\lambda X_{\boldu,t})
    & = \E\left[ \E\left[ \exp\left( \lambda v_{a_t} \bolde_{a_t}^\T (\normcov^{1/2} \boldx + \sigma \boldy) + \lambda \boldv^\T \bolde_{a_t} \bolde_{a_t}^\T \normmean - \frac{\lambda \boldv^\T \normmean}{|\calA|} \right) \mid a_t \right] \right] \\
    & = \E\left[ \exp\left( \frac{\lambda^2 v_{a_t}^2 (\bolde_{a_t}^\T \normcov \bolde_{a_t} + \sigma^2)}{2} + \lambda \boldv^\T \bolde_{a_t} \bolde_{a_t}^\T \normmean - \frac{\lambda \boldv^\T \normmean}{|\calA|} \right) \right] \\
    & \leq \exp\left( \sum_{a \in \calA} \frac{2\lambda^2 v_a^2 \left( \bolde_a^\T \normcov \bolde_a + \sigma^2 + (\bolde_a^\T \normmean)^2 \right)}{|\calA|} \right) \\
    & = \exp\left( \frac{2\lambda^2 \boldu^\T \normcov^{-1/2} \left( \diag(\normcov) + \sigma^2 \boldI + \diag(\normmean)^2 \right) \normcov^{-1/2} \boldu}{|\calA|} \right) \\
    & \leq \exp\left( \frac{2d_2 \lambda^2}{|\calA|} \right) ,
  \end{align*}
  where we have used the moment generating function of standard Gaussian random variables, Lemma~\ref{lem:cat-mgf} with the inequality from the previous display, and the definition of $d_2$.
  Thus $X_{\boldu,t}$ is $(4d_2/|\calA|,2d_\infty)$-subexponential.
  By independence, $\sum_{t=1}^T X_{\boldu,t}$ is $(4Td_2/|\calA|,2d_\infty)$-subexponential.
  For any $\delta' \in (0,1)$, a Bernstein inequality for subexponential random variables~\cite[Theorem 2.8.1]{vershynin2018high} gives, with probability at least $1-\delta'$,
  \begin{equation*}
    \sum_{t=1}^T X_{\boldu,t}
    \leq \frac{C}2 \left(
      \sqrt{
        \frac{T d_2 \log(1/\delta')}{|\calA|}
      }
      + d_\infty \log(1/\delta')
    \right)
    .
  \end{equation*}
  Combining with a union bound over all choices of $\boldu$ from a $(1/2)$-net $N$ of $S^{|\calA|-1}$ shows that with probability at least $1 - |N|\delta'$, the inequality in the previous display holds for all $\boldu \in N$.
  A standard volume argument shows that we can take $|N| \leq 5^{|\calA|}$~\cite[Corollary 4.2.13]{vershynin2018high}.
  Therefore, the claim follows by choosing $\delta' := \delta/5^{|\calA|}$ and observing that~\cite[Exercise 4.4.2]{vershynin2018high}
  \begin{align*}
    \|\normcov^{-1/2} (\normmeanest - \normmean)\|_2
    & = \frac{|\calA|}{T} \sup_{\boldu \in S^{{|\calA|}-1}}
    \sum_{t=1}^T X_{\boldu,t}
%    \\
%    &
    \leq \frac{2|\calA|}{T} \sup_{\boldu \in N}
    \sum_{t=1}^T X_{\boldu,t}.\tag*\qedhere
  \end{align*}
\end{proof}

\paragraph{Covariance estimation.}

We now consider the estimate of $\normcov$.
To do so, we first consider the case where $\normmean$ is already known, so only $\normcov$ needs to be estimated.
We assume the first two rounds in each of $T$ episodes are chosen independently and uniformly at random from $\calA$.
In episode $t$:
\begin{enumerate}
  \item let $\boldmu_t \sim P = \Normal(\normmean,\normcov)$ denote the mean reward vector;
  \item let $a_t, b_t \iidsim \Uniform(\calA)$ be the actions taken in the first two rounds (independent of $\boldmu_t$);
  \item let $\boldr_t$ and $\bolds_t$ be the reward vectors for the first two rounds, so
  \begin{equation*}
    \boldr_t, \bolds_t \mid (\boldmu_t,a_t,b_t) \iidsim \Normal(\boldmu_t, \sigma^2 \boldI) .
  \end{equation*}
\end{enumerate}
The rewards observed (and accrued) in the first two rounds of episode $t$ are $\boldr_t^\T \bolde_{a_t}$ and $\bolds_t^\T \bolde_{b_t}$.
Our estimate of prior covariance $\normcov$ is
\begin{equation} \label{eq:normcovest}
  \normcovest := \frac{|\calA|^2}{2T} \sum_{t=1}^T (\boldr_t - \normmean)^\T \bolde_{a_t} (\bolds_t - \normmean)^\T \bolde_{b_t} \left( \bolde_{a_t} \bolde_{b_t}^\T + \bolde_{b_t} \bolde_{a_t}^\T \right).
\end{equation}

\begin{lemma}[Gaussian covariance estimation with known mean]
  \label{lemma:gaussian-cov}
  There exists a universal constant $C>0$ such that the following holds.
  Consider any multivariate Gaussian distribution $P := \Normal(\normmean, \normcov)$ in $\R^{\calA}$.
  Let
  \begin{equation*}
    (\boldmu_1,a_1,b_1,\boldr_1,\bolds_1),
    (\boldmu_2,a_2,b_2,\boldr_2,\bolds_2), \dotsc,
    (\boldmu_T,a_T,b_T,\boldr_T,\bolds_T)
  \end{equation*}
  be $T$ iid random variables, with
  \begin{align*}
    (\boldmu_t, a_t, b_t) & \sim P \otimes \Uniform(\calA) \otimes \Uniform(\calA), \\
    (\boldr_t,\bolds_t) \mid (\boldmu_t,a_t,b_t) & \sim \Normal(\boldmu_t, \sigma^2 \boldI) \otimes \Normal(\boldmu_t, \sigma^2 \boldI);
  \end{align*}
  and define $\normcovest$ as in \eqref{eq:normcovest}.
%  \begin{equation*}
%    \normcovest := \frac{K^2}{2T} \sum_{t=1}^T (\boldr_t - \normmean)^\T \bolde_{a_t} (\bolds_t - \normmean)^\T \bolde_{b_t} \left( \bolde_{a_t} \bolde_{b_t}^\T + \bolde_{b_t} \bolde_{a_t}^\T \right).
%  \end{equation*}
  For any $\delta \in (0,1)$, with probability at least $1-\delta$,
  \begin{equation*}
    \|\normcov^{-1/2}(\normcovest - \normcov)\normcov^{-1/2}\|_2
    \leq C
    \sqrt{ d }
    \left(
      \sqrt{\frac{|\calA|^3 + |\calA|^2 \log(1/\delta)}{T}}
      + \frac{|\calA|^3 + |\calA|^2 \log(1/\delta)}{T}
    \right)
    ,
  \end{equation*}
  where
  \begin{align*}
    d & := \frac{\sigma^4 + \max_{a \in \calA} (\normcov)_{a,a}^2}{\lambda_{\min}(\normcov)^2} .
  \end{align*}
\end{lemma}
\begin{proof}%[Proof of Lemma~\ref{lemma:gaussian-cov}]
  The proof is very similar to that of Lemma~\ref{lemma:gaussian-mean}.
  We first observe that
  \begin{equation*}
    \E
    \left[
      \frac12
      (\boldr_t - \normmean)^\T \bolde_{a_t}  
      (\bolds_t - \normmean)^\T \bolde_{b_t}
      \left( \bolde_{a_t} \bolde_{b_t}^\T + \bolde_{b_t} \bolde_{a_t}^\T \right)
    \right]
    = \frac1{|\calA|^2} \normcov .
  \end{equation*}
  We claim that for any unit vector $\boldu \in S^{|\calA|-1}$,
  \begin{equation*}
    X_{\boldu,t} :=
    (\normcov^{-1/2} \boldu)^\T
    \left(
      \frac12
      (\boldr_t - \normmean)^\T \bolde_{a_t}  
      (\bolds_t - \normmean)^\T \bolde_{b_t}
      \left( \bolde_{a_t} \bolde_{b_t}^\T + \bolde_{b_t} \bolde_{a_t}^\T \right)
      - \frac1{|\calA|^2} \normcov
    \right)
    (\normcov^{-1/2} \boldu)
  \end{equation*}
  is $(v,c)$-subexponential with $v = O(d/|\calA|^2)$ and $c = O(\sqrt{d})$.
  We defer this argument until the end.
  By independence, $\sum_{t=1}^T X_{\boldu,t}$
%  \begin{equation*}
%    \sum_{t=1}^T
%    (\normcov^{-1/2} \boldu)^\T
%    \left(
%      (\boldr_t - \normmean)^\T \bolde_{a_t}
%      (\bolds_t - \normmean)^\T \bolde_{b_t}
%      \left( \bolde_{a_t} \bolde_{b_t}^\T + \bolde_{b_t} \bolde_{a_t}^\T \right)
%      - \frac1{|\calA|^2} \normcov
%    \right)
%    (\normcov^{-1/2} \boldu)
%  \end{equation*}
  is $(Tv,c)$-subexponential.
  For any $\delta' \in (0,1)$, a Bernstein inequality for subexponential random variables~\cite[Theorem 2.8.1]{vershynin2018high} gives, with probability at least $1-\delta'$,
  \begin{equation*}
    \biggl|
    \sum_{t=1}^T
    X_{\boldu,t}
%    (\normcov^{-1/2} \boldu)^\T
%    \left(
%    (\boldr_t - \normmean)^\T \bolde_{a_t} (\bolds_t - \normmean)^\T \bolde_{b_t} \left( \bolde_{a_t} \bolde_{b_t}^\T + \bolde_{b_t} \bolde_{a_t}^\T \right)
%    - \frac1{|\calA|^2} \normcov
%    \right)
%    (\normcov^{-1/2} \boldu)
    \biggr|
%    \\
    \leq \frac{C}2 \left(
      \sqrt{
        \frac{T d \log(1/\delta')}{|\calA|^2}
      }
      + \sqrt{d} \log(1/\delta')
    \right)
    .
  \end{equation*}
  Combining with a union bound over all choices of $\boldu$ from a $(1/4)$-net $N$ of $S^{|\calA|-1}$ shows that with probability at least $1 - |N|\delta'$, the inequality in the previous display holds for all $\boldu \in N$.
  A standard volume argument shows that we can take $|N| \leq 9^{|\calA|}$~\cite[Corollary 4.2.13]{vershynin2018high}.
  Therefore, the claim follows by choosing $\delta' := \delta/9^{|\calA|}$ and observing that~\cite[Exercise 4.4.3(b)]{vershynin2018high}

  \begin{align*}
%    \lefteqn{
      \|\normcov^{-1/2} (\normcovest - \normcov) \normcov^{-1/2}\|_2
%    } \\
    & = \frac{|\calA|^2}{T} \sup_{\boldu \in S^{|\calA|-1}}
%    \\
%    & \qquad
    \biggl|
    \sum_{t=1}^T 
    X_{\boldu,t}
%    (\normcov^{-1/2} \boldu)^\T
%    \left(
%    (\boldr_t - \normmean)^\T \bolde_{a_t} (\bolds_t - \normmean)^\T \bolde_{b_t} \left( \bolde_{a_t} \bolde_{b_t}^\T + \bolde_{b_t} \bolde_{a_t}^\T \right)
%    - \frac1{|\calA|^2} \normcov \right)
%    (\normcov^{-1/2} \boldu) 
    \biggr|
%    \\
%    &
    \leq \frac{2|\calA|^2}{T} \sup_{\boldu \in N}
%    \\
%    & \qquad
    \biggl|
    \sum_{t=1}^T 
    X_{\boldu,t}
%    (\normcov^{-1/2} \boldu)^\T
%    \left(
%    (\boldr_t - \normmean)^\T \bolde_{a_t} (\bolds_t - \normmean)^\T \bolde_{b_t} \left( \bolde_{a_t} \bolde_{b_t}^\T + \bolde_{b_t} \bolde_{a_t}^\T \right)
%    - \frac1{|\calA|^2} \normcov \right)
%    (\normcov^{-1/2} \boldu)
    \biggr|
  \end{align*}

  It remains to show that $X_{\boldu,t}$ is $(v,c)$-subexponential with $v = O(d/|\calA|^2)$ and $c = O(\sqrt{d})$.
  Observe that $(\boldr_t,\bolds_t,a_t,b_t)$ has the same joint distribution as $(\normmean + \normcov^{1/2} \boldx + \sigma \boldy, \normmean + \normcov^{1/2} \boldx + \sigma \boldz,a_t,b_t)$, where $\boldx, \boldy, \boldz$ are i.i.d.~$\Normal(\bm0,\boldI)$ random vectors in $\R^{\calA}$, independent of $(a_t,b_t)$.
  Let $\boldv := \normcov^{-1/2} \boldu$ and $\boldw := \normcov^{1/2} \boldx$, so
  \begin{align*}
    X_{\boldu,t}
    & =
    (\normcov^{-1/2} \boldu)^\T
    \left(
      \frac12
      (\boldr_t - \normmean)^\T \bolde_{a_t}  
      (\bolds_t - \normmean)^\T \bolde_{b_t}
      \left( \bolde_{a_t} \bolde_{b_t}^\T + \bolde_{b_t} \bolde_{a_t}^\T \right)
      - \frac1{|\calA|^2} \normcov
    \right)
    (\normcov^{-1/2} \boldu)
    \\
    & \stackrel{\text{dist}}{=}
    (\normcov^{-1/2} \boldu)^\T
    \left(
      \frac12
      (\boldw + \sigma \boldy)^\T \bolde_{a_t}  
      (\boldw + \sigma \boldz)^\T \bolde_{b_t}  
      \left( \bolde_{a_t} \bolde_{b_t}^\T + \bolde_{b_t} \bolde_{a_t}^\T \right)
      - \frac1{|\calA|^2} \normcov
    \right)
    (\normcov^{-1/2} \boldu)
    \\
    & =
    \left(
      v_{a_t} v_{b_t}
      (w_{a_t} + \sigma y_{a_t})
      (w_{b_t} + \sigma z_{b_t})
      - \frac1{|\calA|^2}
    \right)
    .
  \end{align*}
  Now we fix $\lambda \in \R$ such that $|\lambda| \leq 1/(C \sqrt{d})$ for some sufficiently large constant $C>0$, and bound the moment generating function of $X_{\boldu,t}$ at $\lambda$.
  To do so, we use the above characterization of the distribution of $X_{\boldu,t}$ in terms of the independent Gaussian random vectors.
  First, taking expectation only with respect to $(\boldy,\boldz)$ (i.e., conditional on $a_t,b_t,\boldx$):
  \begin{align*}
    \lefteqn{ \E\left[ \exp(\lambda X_{\boldu,t}) \right] } \\
    & =
    \E\left[
      \E\left[
        \exp\left(
          \lambda
          \left(
            v_{a_t} v_{b_t}
            (w_{a_t} + \sigma y_{a_t})
            (w_{b_t} + \sigma z_{b_t})
            - \frac1{|\calA|^2}
          \right)
        \right)
        \mid a_t,b_t,\boldx
      \right]
    \right]
    \\
    & = \E\left[
      \exp\left(
        \eta w_{a_t} w_{b_t}
        + \frac{\eta^2 \sigma^2 w_{a_t}^2 + \eta^2 \sigma^2 w_{b_t}^2 + \eta^3 \sigma^4 w_{a_t} w_{b_t}}{2(1 - \eta^2 \sigma^4)}
        + \frac12 \ln \frac{1}{1-\eta^2\sigma^4}
        - \frac{\lambda}{|\calA|^2}
      \right)
    \right]
    \\
    & \leq \E\left[
      \exp\left(
        \eta w_{a_t} w_{b_t}
        + \frac{\eta^2 \sigma^2 w_{a_t}^2 + \eta^2 \sigma^2 w_{b_t}^2 + \eta^3 \sigma^4 w_{a_t} w_{b_t}}{2(1 - \eta^2 \sigma^4)}
        + \eta^2\sigma^4
        - \frac{\lambda}{|\calA|^2}
      \right)
    \right]
  \end{align*}
  where $\eta := \lambda v_{a_t} v_{b_t}$ satisfies $\eta^2 \sigma^4 \leq 1/2$ (due to the assumption on $\lambda$).
  Next, we note that
  \begin{align*}
    \eta w_{a_t} w_{b_t}
    + \frac{\eta^2 \sigma^2 w_{a_t}^2 + \eta^2 \sigma^2 w_{b_t}^2 + \eta^3 \sigma^4 w_{a_t} w_{b_t}}{2(1 - \eta^2 \sigma^4)}
    & = \boldx^\T \boldA \boldx
  \end{align*}
  where $\boldA$ is the random symmetric matrix defined by
  \begin{align*}
    \boldA
    & :=
    \frac{\eta}{2(1-\eta^2\sigma^4)}
    \normcov^{1/2}
    \left(
      (1 - \eta^2\sigma^4/2)
      (\bolde_{a_t} \bolde_{b_t}^\T + \bolde_{b_t} \bolde_{a_t}^\T)
      + \eta\sigma^2
      (\bolde_{a_t} \bolde_{a_t}^\T + \bolde_{b_t} \bolde_{b_t}^\T)
    \right)
    \normcov^{1/2}
    \\
    & \hphantom:=
    \eta\left( 1 + \frac{\eta^2\sigma^4}{2(1-\eta^2\sigma^4)} \right)
    \normcov^{1/2}
    \left(
      \frac12(\bolde_{a_t} \bolde_{b_t}^\T + \bolde_{b_t} \bolde_{a_t}^\T)
    \right)
    \normcov^{1/2}
    \\
    & \qquad
    + \frac{\eta^2\sigma^2}{1-\eta^2\sigma^4}
    \normcov^{1/2}
    \left(
      \frac12(\bolde_{a_t} \bolde_{a_t}^\T + \bolde_{b_t} \bolde_{b_t}^\T)
    \right)
  \end{align*}
  (where the randomness comes from $a_t,b_t$).
  Since $\boldA$ is symmetric, it has real eigenvalues $\lambda_1, \lambda_2, \dotsc, \lambda_{|\calA|}$.
  We shall ensure via the assumption on $\lambda$ that $\|\boldA\|_2 \leq 1/3$, which implies that $|\lambda_i| \leq 1/3$ for all $i$.
  The rotational invariance of $\Normal(\bm0,\boldI)$ implies that the distribution of $\boldx^\T \boldA \boldx$ (conditional on $a_t,b_t$) is the same as that of $\sum_i \lambda_i t_i$, where $t_1, t_2, \dotsc, t_{|\calA|}$ are i.i.d.~$\chi^2(1)$ random variables.
  This implies that
  \begin{align*}
    \E\left[
      \exp\left( \boldx^\T \boldA \boldx \right) \mid a_t,b_t
    \right]
    & = \exp\left( \frac12 \sum_i \ln \frac{1}{1-2\lambda_i} \right) \\
    & \leq \exp\left( \sum_i \lambda_i + 2\lambda_i^2 \right)
    = \exp\left( \trace(\boldA) + 2\|\boldA\|_{\F}^2 \right)
  \end{align*}
  where the inequality uses the bound $|\lambda_i| \leq 1/3$.
  We expand $\trace(\boldA)$ to reveal its dependence on $a_t,b_t$:
  \begin{equation*}
    \trace(\boldA)
    = \eta (\normcov)_{a_t,b_t} + \frac{\eta^3 \sigma^4 (\normcov)_{a_t,b_t} + \eta^2 \sigma^2 ((\normcov)_{a_t,a_t} + (\normcov)_{b_t,b_t})}{2(1-\eta^2\sigma^4)} .
  \end{equation*}
  And we bound $\|\boldA\|_{\F}^2$ as follows:
  \begin{align*}
    \|\boldA\|_{\F}^2
    & \leq 2\eta^2 \left( 1 + \frac{\eta^2\sigma^4}{2(1-\eta^2\sigma^4)} \right)^2
    \left\|
      \normcov^{1/2}
      \left(
        \frac12(\bolde_{a_t} \bolde_{b_t}^\T + \bolde_{b_t} \bolde_{a_t}^\T)
      \right)
      \normcov^{1/2}
    \right\|_{\F}^2
    \\
    & \qquad
    + 2\eta^2\left( \frac{\eta\sigma^2}{1-\eta^2\sigma^4} \right)^2
    \left\|
      \normcov^{1/2}
      \left(
        \frac12(\bolde_{a_t} \bolde_{a_t}^\T + \bolde_{b_t} \bolde_{b_t}^\T)
      \right)
      \normcov^{1/2}
    \right\|_{\F}^2
    \\
    & = \eta^2 \left( 1 + \frac{\eta^2\sigma^4}{2(1-\eta^2\sigma^4)} \right)^2
    \left( (\normcov)_{a_t,b_t}^2 + (\normcov)_{a_t,a_t} (\normcov)_{b_t,b_t} \right)
    \\
    & \qquad
    + \frac{\eta^2}{2} \left( \frac{\eta\sigma^2}{1-\eta^2\sigma^4} \right)^2
    \left( (\normcov)_{a_t,a_t}^2 + 2(\normcov)_{a_t,b_t}^2 + (\normcov)_{b_t,b_t}^2 \right)
    \\
    & \leq 4\eta^2 \left( 1 + \frac{\eta^2\sigma^4}{2(1-\eta^2\sigma^4)} \right)^2
    \left( (\normcov)_{a_t,a_t}^2 + (\normcov)_{b_t,b_t}^2 \right)
    \\
    & \qquad
    + \eta^2 \left( \frac{\eta\sigma^2}{1-\eta^2\sigma^4} \right)^2
    \left( (\normcov)_{a_t,a_t}^2 + (\normcov)_{b_t,b_t}^2 \right)
    \\
    & \leq 11\eta^2 \left( (\normcov)_{a_t,a_t}^2 + (\normcov)_{b_t,b_t}^2 \right)
    .
  \end{align*}
  Above the first inequality follows by the triangle inequality and the fact $(a+b)^2 \leq 2(a^2+b^2)$ for $a,b\geq0$; the second inequality uses Cauchy-Schwarz and the AM/GM inequality; the third inequality uses the assumption $\eta^2 \sigma^4 \leq 1/2$.
  Thus, we have shown that
  \begin{align*}
    \E\left[ \exp(\lambda X_{\boldu,t}) \right]
    & \leq \E\left[
      \exp\left(
        \trace(\boldA) + 2\|\boldA\|_{\F}^2
        + \eta^2\sigma^4
        - \frac{\lambda}{|\calA|^2}
      \right)
    \right]
    \\
    & \leq \E\biggl[
      \exp\biggl(
        \eta (\normcov)_{a_t,b_t} + \frac{\eta^3 \sigma^4 (\normcov)_{a_t,b_t} + \eta^2 \sigma^2 ((\normcov)_{a_t,a_t} + (\normcov)_{b_t,b_t})}{2(1-\eta^2\sigma^4)}
        \\
    & \qquad\qquad\qquad
        + 11\eta^2 \left( (\normcov)_{a_t,a_t}^2 + (\normcov)_{b_t,b_t}^2 \right)
        - \frac{\lambda}{|\calA|^2}
      \biggr)
    \biggr]
    .
  \end{align*}
  Define
  \begin{align*}
    \alpha_{a,b} & := \beta_{a,b} + \gamma_{a,b} \\
    \beta_{a,b} & := \lambda v_a v_b (\normcov)_{a,b} \\
    \gamma_{a,b} & := \frac{\lambda^3 v_a^3 v_b^3 \sigma^4 (\normcov)_{a,b} + \lambda^2 v_a^2 v_b^2 \sigma^2 ((\normcov)_{a,a} + (\normcov)_{b,b})}{2(1-\lambda^2 v_a^2 v_b^2 \sigma^4)} \\
                 & \qquad + \lambda^2 v_a^2 v_b^2 (11((\normcov)_{a,a}^2 + (\normcov)_{b,b}^2) + \sigma^4) .
  \end{align*}
  Observe that
  \begin{equation*}
    \frac1{|\calA|^2} \sum_{a,b \in \calA} \beta_{a,b}
    = 1
    .
  \end{equation*}
  The assumptions $\lambda$ ensure that $|\beta_{a,b}| + |\gamma_{a,b}| \leq 1$, so we can apply Lemma~\ref{lem:cat-mgf} to bound the final expression in the previous display to obtain the inequality
  \begin{align*}
    \E\left[ \exp(\lambda X_{\boldu,t}) \right]
    & \leq
    \exp\left( \frac1{|\calA|^2} \sum_{a,b \in \calA} \left( \alpha_{a,b} + \alpha_{a,b}^2 \right) - \frac{\lambda}{|\calA|^2} \right)
    \\
    & \leq
    \exp\left( \frac1{|\calA|^2} \sum_{a,b \in \calA} (4\gamma_{a,b} + \beta_{a,b}^2) \right)
    \\
    & \leq
    \exp\left( \frac{8\|\boldv\|_2^4(\sigma^4 + \|\diag(\normcov)\|_2^2)}{|\calA|^2} \left[ \frac{|\lambda| \|\boldv\|_\infty^2 \sigma^2 + 2}{4(1-\lambda^2\|\boldv\|_\infty^4 \sigma^4)} + 23 \right] \frac{\lambda^2}{2} \right)
    \\
    & \leq
    \exp\left( \frac{200\|\boldv\|_2^4(\sigma^4 + \|\diag(\normcov)\|_2^2)}{|\calA|^2} \cdot \frac{\lambda^2}{2} \right)
    ,
  \end{align*}
  where the inequalities use the bounds on $|\beta_{a,b}|$ and $|\gamma_{a,b}|$, and the additional bound $|\lambda| \|\boldv\|_\infty^2 ( \sigma^2 + \|\diag(\normcov)\|_2 ) \leq 1/10$ which is implied by the assumption on $\lambda$.
  The final bound is $\exp(C (d/|\calA|^2) \cdot \lambda^2/2)$ for a sufficiently large absolute constant $C>0$.
\end{proof}

\paragraph{Covariance estimation, redux.}

Now we consider the case where both $\normmean$ and $\normcov$ are unknown and need to be estimated.
A standard approach to estimating $\normcov$ is to simply estimate the second moment of $\boldmu_t$ (instead of its covariance), and then to subtract $\normmeanest \normmeanest^\T$ using some estimate $\normmeanest$ of $\normmean$.
However, the quality of our estimate of $\normmean$ (described above) depends on properties of $\normmean$ itself, which should not be necessary.
Below, we instead analyze an estimator of $\normcov$ based on differences, essentially leveraging the fact that the variance of a random variable $X$ is half the expected squared difference between $X$ and an independent copy of itself.

We assume the first two rounds in each of $2T$ episodes are chosen independently and uniformly at random from $\calA$.
However, we use the same two chosen actions in two consecutive episodes.
That is, in episodes $2t-1$ and $2t$:
\begin{enumerate}
  \item let $\boldmu_t, \tilde\boldmu_t \iidsim P = \Normal(\normmean,\normcov)$ denote the mean reward vectors for episodes $2t-1$ and $2t$;
  \item let $a_t, b_t \iidsim \Uniform(\calA)$ be the actions taken in the first two rounds (independent of $\boldmu_t$, $\tilde\boldmu_t$);
  \item
    let $\boldr_t$ and $\bolds_t$ be the reward vectors for the first two rounds of episode $2t-1$, and let
    $\tilde\boldr_t$ and $\tilde\bolds_t$ be the reward vectors for the first two rounds of episode $2t$, so
    \begin{equation*}
      (\boldr_t, \bolds_t, \tilde\boldr_t, \tilde\bolds_t) \mid (\boldmu_t,\tilde\boldmu_t,a_t,b_t) \sim \Normal(\boldmu_t, \sigma^2 \boldI) \otimes \Normal(\boldmu_t, \sigma^2 \boldI) \otimes \Normal(\tilde\boldmu_t, \sigma^2 \boldI) \otimes \Normal(\tilde\boldmu_t, \sigma^2 \boldI) .
    \end{equation*}
\end{enumerate}
The rewards observed (and accrued) in the first two rounds of episode $2t-1$ are $\boldr_t^\T \bolde_{a_t}$ and $\bolds_t^\T \bolde_{b_t}$, and the rewards observed (and accrued) in the first two rounds of episode $2t$ are $\tilde\boldr_t^\T \bolde_{a_t}$ and $\tilde\bolds_t^\T \bolde_{b_t}$.
Our estimate of prior covariance $\normcov$ is
\begin{equation} \label{eq:normcovest2}
  \normcovest := \frac{|\calA|^2}{4T} \sum_{t=1}^T (\boldr_t - \tilde\boldr_t)^\T \bolde_{a_t} (\bolds_t - \tilde\bolds_t)^\T \bolde_{b_t} \left( \bolde_{a_t} \bolde_{b_t}^\T + \bolde_{b_t} \bolde_{a_t}^\T \right).
\end{equation}

\begin{lemma}[Gaussian covariance estimation with unknown mean]
  \label{lemma:gaussian-cov2}
  There exists a universal constant $C>0$ such that the following holds.
  Consider any multivariate Gaussian distribution $P := \Normal(\normmean, \normcov)$ in $\R^{\calA}$.
  Let
  \begin{equation*}
    (\boldmu_1,\tilde\boldmu_1,a_1,b_1,\boldr_1,\bolds_1,\tilde\boldr_1,\tilde\bolds_1),
    (\boldmu_2,\tilde\boldmu_2,a_2,b_2,\boldr_2,\bolds_2,\tilde\boldr_2,\tilde\bolds_2),
    \dotsc,
    (\boldmu_T,\tilde\boldmu_T,a_T,b_T,\boldr_T,\bolds_T,\tilde\boldr_T,\tilde\bolds_T),
  \end{equation*}
  be $T$ iid random variables, with
  \begin{align*}
    (\boldmu_t, \tilde\boldmu_t, a_t, b_t) & \sim P \otimes P \otimes \Uniform(\calA) \otimes \Uniform(\calA), \\
    (\boldr_t,\bolds_t,\tilde\boldr_t,\tilde\bolds_t) \mid (\boldmu_t,\tilde\boldmu_t,a_t,b_t) & \sim \Normal(\boldmu_t, \sigma^2 \boldI) \otimes \Normal(\boldmu_t, \sigma^2 \boldI) \otimes \Normal(\tilde\boldmu_t, \sigma^2 \boldI) \otimes \Normal(\tilde\boldmu_t, \sigma^2 \boldI);
  \end{align*}
  and define $\normcovest$ as in \eqref{eq:normcovest2}.
%  \begin{equation*}
%    \normcovest := \frac{K^2}{4T} \sum_{t=1}^T (\boldr_t - \tilde\boldr_t)^\T \bolde_{a_t} (\bolds_t - \tilde\bolds_t)^\T \bolde_{b_t} \left( \bolde_{a_t} \bolde_{b_t}^\T + \bolde_{b_t} \bolde_{a_t}^\T \right).
%  \end{equation*}
  For any $\delta \in (0,1)$, with probability at least $1-\delta$,
  \begin{equation*}
    \|\normcov^{-1/2}(\normcovest - \normcov)\normcov^{-1/2}\|_2
    \leq C
    \sqrt{ d }
    \left(
      \sqrt{\frac{|\calA|^3 + |\calA|^2 \log(1/\delta)}{T}}
      + \frac{|\calA|^3 + |\calA|^2 \log(1/\delta)}{T}
    \right)
    ,
  \end{equation*}
  where
  \begin{align*}
    d & := \frac{\sigma^4 + \max_{a \in \calA} (\normcov)_{a,a}^2}{\lambda_{\min}(\normcov)^2} .
  \end{align*}
\end{lemma}

We omit the proof of Lemma~\ref{lemma:gaussian-cov2} since it is completely analogous to that of Lemma~\ref{lemma:gaussian-cov}.

The following lemma is used to bound the exponential moment of a discrete real-valued random variable.
\begin{lemma} \label{lem:cat-mgf}
  Let $Y$ be a random variable supported on $\{ \alpha_1,\dotsc,\alpha_K \} \subset \R$ with $\alpha_i \leq 1$ and $p_i := \Pr(Y = \alpha_i)$ for all $i$.
  Then
  \begin{equation*}
    \E[ \exp(Y) ]
    \leq \exp\left( \sum_{i=1}^K p_i \alpha_i + p_i \alpha_i^2 \right)
    .
  \end{equation*}
\end{lemma}
\begin{proof}
  Since $e^t \leq 1 + t + t^2$ for all $t \leq 1$, we have
  \begin{equation*}
    \E[ \exp(Y) ]
    \leq \E[ 1 + Y + Y^2 ]
    = 1 + \sum_{i=1}^K q_i \alpha_i + \sum_{i=1}^K q_i \alpha_i^2 .
  \end{equation*}
  The claim now follows since $1+t\leq e^t$ for all $t \in \R$.
\end{proof}

%% file: arxiv_main.bbl
\newcommand{\etalchar}[1]{$^{#1}$}
\begin{thebibliography}{WKNT{\etalchar{+}}17}

\bibitem[AG12]{agrawal2012analysis}
Shipra Agrawal and Navin Goyal.
\newblock Analysis of thompson sampling for the multi-armed bandit problem.
\newblock In {\em Conference on Learning Theory}, 2012.

\bibitem[AL17]{abeille2017linear}
Marc Abeille and Alessandro Lazaric.
\newblock Linear thompson sampling revisited.
\newblock In {\em International Conference on Artificial Intelligence and
  Statistics}, 2017.

\bibitem[AL18]{abeille2018improved}
Marc Abeille and Alessandro Lazaric.
\newblock Improved regret bounds for thompson sampling in linear quadratic
  control problems.
\newblock In {\em International Conference on Machine Learning}, 2018.

\bibitem[ALB13]{azar2013sequential}
Mohammad~Gheshlaghi Azar, Alessandro Lazaric, and Emma Brunskill.
\newblock Sequential transfer in multi-armed bandit with finite set of models.
\newblock In {\em Advances in Neural Information Processing Systems}, 2013.

\bibitem[AYS15]{abbasi2015bayesian}
Yasin Abbasi-Yadkori and Csaba Szepesv{\'a}ri.
\newblock Bayesian optimal control of smoothly parameterized systems.
\newblock In {\em Conference on Uncertainty in Artificial Intelligence}, 2015.

\bibitem[Bax98]{baxter1998theoretical}
Jonathan Baxter.
\newblock Theoretical models of learning to learn.
\newblock In {\em Learning to learn}. Springer, 1998.

\bibitem[Bax00]{baxter2000model}
Jonathan Baxter.
\newblock A model of inductive bias learning.
\newblock {\em Journal of Artificial Intelligence Research}, 2000.

\bibitem[BMD{\etalchar{+}}05]{banerjee2005clustering}
Arindam Banerjee, Srujana Merugu, Inderjit~S Dhillon, Joydeep Ghosh, and John
  Lafferty.
\newblock Clustering with bregman divergences.
\newblock {\em Journal of Machine Learning Research}, 2005.

\bibitem[BMP{\etalchar{+}}94]{berger1994overview}
James~O Berger, El{\'\i}as Moreno, Luis~Raul Pericchi, M~Jes{\'u}s Bayarri,
  Jos{\'e}~M Bernardo, Juan~A Cano, Juli{\'a}n De~la Horra, Jacinto
  Mart{\'\i}n, David R{\'\i}os-Ins{\'u}a, Bruno Betr{\`o}, A.~Dasgupta, Paul
  Gustafson, Larry Wasserman, Joseph~B. Kadane, Cid Srinivasan, Michael Lavine,
  Anthony O'Hagan, Wolfgang Polasek, Christian~P. Robert, Constantinos Goutis,
  Fabrizio Ruggeri, Gabriella Salinetti, and Siva Sivaganesan.
\newblock An overview of robust bayesian analysis.
\newblock {\em Test}, 1994.

\bibitem[Bro86]{brown1986fundamentals}
Lawrence~D Brown.
\newblock {\em Fundamentals of Statistical Exponential Families with
  Applications in Statistical Decision Theory}.
\newblock Institute of Mathematical Statistics, 1986.

\bibitem[CLP20]{cella2020meta}
Leonardo Cella, Alessandro Lazaric, and Massimiliano Pontil.
\newblock Meta-learning with stochastic linear bandits.
\newblock In {\em International Conference on Machine Learning}, 2020.

\bibitem[DSC{\etalchar{+}}16]{duan2016rl}
Yan Duan, John Schulman, Xi~Chen, Peter~L Bartlett, Ilya Sutskever, and Pieter
  Abbeel.
\newblock {RL}$^{2}$: Fast reinforcement learning via slow reinforcement
  learning.
\newblock {\em arXiv:1611.02779}, 2016.

\bibitem[GMM14]{gopalan2014thompson}
Aditya Gopalan, Shie Mannor, and Yishay Mansour.
\newblock Thompson sampling for complex online problems.
\newblock In {\em International Conference on Machine Learning}, 2014.

\bibitem[GMPT15]{ghavamzadeh2016bayesian}
Mohammad Ghavamzadeh, Shie Mannor, Joelle Pineau, and Aviv Tamar.
\newblock Bayesian reinforcement learning: A survey.
\newblock {\em Foundations and Trends\textregistered\ in Machine Learning},
  2015.

\bibitem[HCJ{\etalchar{+}}21]{hu2021near}
Jiachen Hu, Xiaoyu Chen, Chi Jin, Lihong Li, and Liwei Wang.
\newblock Near-optimal representation learning for linear bandits and linear
  {RL}.
\newblock In {\em International Conference on Machine Learning}, 2021.

\bibitem[HGH{\etalchar{+}}20]{humplik2019meta}
Jan Humplik, Alexandre Galashov, Leonard Hasenclever, Pedro~A Ortega, Yee~Whye
  Teh, and Nicolas Heess.
\newblock Meta reinforcement learning as task inference.
\newblock In {\em International Conference on Learning Representations}, 2020.

\bibitem[HYC01]{hochreiter2001learning}
Sepp Hochreiter, A~Steven Younger, and Peter~R Conwell.
\newblock Learning to learn using gradient descent.
\newblock In {\em International Conference on Artificial Neural Networks},
  2001.

\bibitem[Kak03]{kakade2003sample}
Sham~Machandranath Kakade.
\newblock {\em On the sample complexity of reinforcement learning}.
\newblock PhD thesis, University College London, 2003.

\bibitem[KKM12]{kaufmann2012thompson}
Emilie Kaufmann, Nathaniel Korda, and R{\'e}mi Munos.
\newblock Thompson sampling: An asymptotically optimal finite-time analysis.
\newblock In {\em International Conference on Algorithmic Learning Theory},
  2012.

\bibitem[KKZ{\etalchar{+}}21]{kveton2021meta}
Branislav Kveton, Mikhail Konobeev, Manzil Zaheer, Chih-wei Hsu, Martin
  Mladenov, Craig Boutilier, and Csaba Szepesvari.
\newblock Meta-thompson sampling.
\newblock {\em arXiv:2102.06129}, 2021.

\bibitem[Lin02]{lindvall2002lectures}
Torgny Lindvall.
\newblock {\em Lectures on the Coupling Method}.
\newblock Courier Corporation, 2002.

\bibitem[LL16]{liu2016prior}
Che-Yu Liu and Lihong Li.
\newblock On the prior sensitivity of thompson sampling.
\newblock In {\em International Conference on Algorithmic Learning Theory},
  2016.

\bibitem[OVR17]{osband2017posterior}
Ian Osband and Benjamin Van~Roy.
\newblock Why is posterior sampling better than optimism for reinforcement
  learning?
\newblock In {\em International Conference on Machine Learning}, 2017.

\bibitem[PAYD19]{phan2019thompson}
My~Phan, Yasin Abbasi~Yadkori, and Justin Domke.
\newblock Thompson sampling and approximate inference.
\newblock {\em Advances in Neural Information Processing Systems}, 2019.

\bibitem[RPF12]{ryzhov2012knowledge}
Ilya~O Ryzhov, Warren~B Powell, and Peter~I Frazier.
\newblock The knowledge gradient algorithm for a general class of online
  learning problems.
\newblock {\em Operations Research}, 2012.

\bibitem[Rus16]{russo2016simple}
Daniel Russo.
\newblock Simple bayesian algorithms for best arm identification.
\newblock In {\em Conference on Learning Theory}, 2016.

\bibitem[RVR16]{russo2016information}
Daniel Russo and Benjamin Van~Roy.
\newblock An information-theoretic analysis of thompson sampling.
\newblock {\em The Journal of Machine Learning Research}, 2016.

\bibitem[Str00]{strens2000bayesian}
Malcolm Strens.
\newblock A bayesian framework for reinforcement learning.
\newblock In {\em International Conference on Machine Learning}, 2000.

\bibitem[SWDN09]{sriboonchita2009stochastic}
Songsak Sriboonchita, Wing-Keung Wong, Sompong Dhompongsa, and Hung~T Nguyen.
\newblock {\em Stochastic dominance and applications to finance, risk and
  economics}.
\newblock CRC Press, 2009.

\bibitem[TGG94]{tripathi1994estimation}
Ram~C. Tripathi, Ramesh~C. Gupta, and John Gurland.
\newblock Estimation of parameters in the beta binomial model.
\newblock {\em Annals of the Institute of Statistical Mathematics}, 1994.

\bibitem[Tho33]{thompson1933likelihood}
William~R Thompson.
\newblock On the likelihood that one unknown probability exceeds another in
  view of the evidence of two samples.
\newblock {\em Biometrika}, 1933.

\bibitem[Thr96]{thrun1996explanation}
Sebastian Thrun.
\newblock {\em Explanation-based neural network learning: A lifelong learning
  approach}.
\newblock Springer, 1996.

\bibitem[Thr98]{thrun1998lifelong}
Sebastian Thrun.
\newblock Lifelong learning algorithms.
\newblock In {\em Learning to learn}. Springer, 1998.

\bibitem[Ver18]{vershynin2018high}
Roman Vershynin.
\newblock {\em High-dimensional probability: An introduction with applications
  in data science}.
\newblock Cambridge University Press, 2018.

\bibitem[WKNT{\etalchar{+}}17]{wang2016learning}
Jane~X Wang, Zeb Kurth-Nelson, Dhruva Tirumala, Hubert Soyer, Joel~Z Leibo,
  Remi Munos, Charles Blundell, Dharshan Kumaran, and Matt Botvinick.
\newblock Learning to reinforcement learn.
\newblock In {\em Annual Meeting of the Cognitive Science Society}, 2017.

\bibitem[YHLD21]{yang2021impact}
Jiaqi Yang, Wei Hu, Jason~D Lee, and Simon~S Du.
\newblock Impact of representation learning in linear bandits.
\newblock In {\em International Conference on Learning Representations}, 2021.

\end{thebibliography}
